\documentclass[11pt]{article}
\usepackage[utf8]{inputenc}
\usepackage[T1]{fontenc}

\usepackage{graphicx, graphics, epsfig, color}
\usepackage{booktabs}
\usepackage{subcaption}
\usepackage{amsmath,amssymb,amsthm}
\usepackage{tikz, pgfplots}
\usetikzlibrary{plotmarks,spy}
\pgfplotsset{compat=newest}
\usepackage[margin=2cm]{geometry}
\usepackage{hyperref}
\hypersetup{colorlinks=false}
\usepackage{makecell}

\title{A Random Matrix Analysis of Random Fourier Features: \\Beyond the Gaussian Kernel, a Precise Phase Transition, and the Corresponding Double Descent}

\author{
  Zhenyu Liao\\
  ICSI and Department of Statistics\\
  University of California, Berkeley, USA\\
  \texttt{zhenyu.liao@berkeley.edu}\\
  \and
  Romain Couillet \\
  G-STATS Data Science Chair, GIPSA-lab\\
  University Grenobles-Alpes, France\\
  \texttt{romain.couillet@gipsa-lab.grenoble-inp.fr}
  \and
   Michael W. Mahoney\\
  ICSI and Department of Statistics\\
  University of California, Berkeley, USA\\
  \texttt{mmahoney@stat.berkeley.edu}
}
\date{\today}

\DeclareMathOperator{\tr}{tr}

\DeclareMathOperator{\train}{train}
\DeclareMathOperator{\test}{test}
\newcommand{\T}{{\sf T}}

\newcommand{\RR}{{\mathbb{R}}}

\newcommand{\EE}{{\mathbb{E}}}
\newcommand{\NN}{{\mathcal{N}}}

\newcommand{\A}{\mathbf{A}}
\newcommand{\B}{\mathbf{B}}
\newcommand{\C}{\mathbf{C}}
\newcommand{\D}{\mathbf{D}}

\newcommand{\K}{\mathbf{K}}

\newcommand{\U}{\mathbf{U}}
\newcommand{\W}{\mathbf{W}}
\newcommand{\X}{\mathbf{X}}
\newcommand{\Z}{\mathbf{Z}}
\newcommand{\Q}{\mathbf{Q}}

\newcommand{\x}{\mathbf{x}}
\newcommand{\y}{\mathbf{y}}
\newcommand{\z}{\mathbf{z}}

\newcommand{\w}{\mathbf{w}}

\newcommand{\zo}{\mathbf{0}}

\newcommand{\I}{\mathbf{I}}

\newcommand{\bSigma}{\boldsymbol{\Sigma}}
\newcommand{\bPhi}{\boldsymbol{\Phi}}
\newcommand{\bXi}{\boldsymbol{\Xi}}

\newcommand{\bOmega}{\boldsymbol{\Omega}}

\newcommand{\bbeta}{\boldsymbol{\beta}}

\newcommand{\asto}{{ \xrightarrow{\rm a.s.} }}

\newcommand*{\QED}{\hfill\ensuremath{\blacksquare}}%

\definecolor{RED}{rgb}{0.7,0,0}
\definecolor{BLUE}{rgb}{0,0,0.69}
\definecolor{GREEN}{rgb}{0,0.6,0}
\definecolor{PURPLE}{rgb}{0.69,0,0.8}

\newcommand{\RED}{\color[rgb]{0.70,0,0}}
\newcommand{\BLUE}{\color[rgb]{0,0,0.69}}

\newtheorem{Assumption}{Assumption}
\newtheorem{Theorem}{Theorem}

\newtheorem{Lemma}{Lemma}
\newtheorem{Remark}{Remark}

\begin{document}
\maketitle

\begin{abstract}
This article characterizes the exact asymptotics of random Fourier feature (RFF) regression, in the realistic setting where the number of data samples $n$, their dimension $p$, and the dimension of feature space $N$ are all large and comparable. 
In this regime, the random RFF Gram matrix no longer converges to the well-known limiting Gaussian kernel matrix (as it does when $N \to \infty$ alone), but it still has a tractable behavior that is captured by our analysis. 
This analysis also provides accurate estimates of training and test regression errors for large $n,p,N$. 
Based on these estimates, a precise characterization of two qualitatively different phases of learning, including the phase transition between them, is provided; and the corresponding double descent test error curve is derived from this phase transition behavior. 
These results do not depend on strong assumptions on the data distribution, and they perfectly match empirical results on real-world data sets. 
\end{abstract}

\tableofcontents

\section{Introduction}
\label{sec:introduction}

For a machine learning system having $N$ parameters, trained on a data set of size $n$, asymptotic analysis as used in classical statistical learning theory typically either focuses on the (statistical) population $n \to \infty$ limit, for $N$ fixed, or the over-parameterized $N \to \infty$ limit, for a given $n$. 
These two settings are technically more convenient to work with, yet less practical, as they essentially assume that one of the two dimensions is negligibly small compared to the other, and this is rarely the case in practice.
Indeed, with a factor of $2$ or $10$ more data, one typically works with a more complex model.
This has been highlighted perhaps most prominently in recent work on neural network models, in which the model complexity and data size increase together.
For this reason, the \emph{double asymptotic} regime where $n,N \rightarrow \infty$, with $N/n\rightarrow c$, a constant, is a particularly interesting (and likely more realistic) limit, despite being technically more challenging~\cite{SST92,WRB93,DKST96,EB01_BOOK,MezardMontanari09,MM17_TR,BKPx20}.
In particular, working in this regime allows for a finer quantitative assessment of machine learning systems, as a function of their \emph{relative} complexity $N/n$, as well as for a precise description of the under- to over-parameterized ``phase transition'' (that does not appear in the $N\to \infty$ alone analysis). 
This transition is largely hidden in the usual style of statistical learning theory~\cite{Vapnik98}, but it is well-known in the statistical mechanics approach to learning theory~\cite{SST92,WRB93,DKST96,EB01_BOOK}, and empirical signatures of it have received attention recently under the name ``double descent'' phenomena~\cite{advani2020high,belkin2019reconciling}.

This article considers the asymptotics of random Fourier features~\cite{rahimi2008random}, and more generally random feature maps, which may be viewed also as a single-hidden-layer neural network model, in this limit.
More precisely, let $\X = [\x_1, \ldots, \x_n] \in \RR^{p \times n}$ denote the data matrix of size $n$ with data vectors $\x_i \in \RR^p$ as column vectors. 
The random feature matrix $\bSigma_\X$ of $\X$ is generated by pre-multiplying some random matrix $\W \in \RR^{N \times p}$ having i.i.d.\@ entries and then passing through some \emph{entry-wise} nonlinear function $\sigma(\cdot)$, i.e., $\bSigma_\X \equiv \sigma(\W \X) \in \RR^{N \times n}$.
Commonly used random feature techniques such as random Fourier features (RFFs) \cite{rahimi2008random} and homogeneous kernel maps \cite{vedaldi2012efficient}, however, rarely involve a single nonlinearity. 
The popular RFF maps are built with cosine and sine nonlinearities, so that $\bSigma_\X \in \RR^{2N \times n}$ is obtained by cascading the random features of both, i.e., $\bSigma_\X^\T \equiv [\cos(\W \X)^\T,~\sin(\W\X)^\T]$. 
Note that, by combining both nonlinearities, RFFs generated from $\W \in \RR^{N \times p}$ are of dimension $2N$. 


The large $N$ asymptotics of random feature maps is closely related to their limiting kernel matrices $\K_\X$. 
In the case of RFF, it was shown in \cite{rahimi2008random} that \emph{entry-wise} the Gram matrix $\bSigma_\X^\T \bSigma_\X/N$ converges to the Gaussian kernel matrix $\K_\X \equiv \{\exp(- \| \x_i - \x_j \|^2/2) \}_{i,j=1}^n$, as $N \to \infty$. 
This follows from $\frac1N [\bSigma_\X^\T \bSigma_\X]_{ij} = \frac1N \sum_{t=1}^N \cos(\x_i^\T \w_t) \cos(\w_t^\T \x_j) + \sin(\x_i^\T \w_t) \sin(\w_t^\T \x_j) $, for $\w_t$ independent Gaussian random vectors, so that by the strong law of large numbers, for fixed $n,p$, $[\bSigma_\X^\T \bSigma_\X/N]_{ij}$ goes to its expectation (with respect to $\w \sim \NN(\zo, \I_p)$) almost surely as $N \to \infty$, i.e.,
\begin{equation}
  [\bSigma_\X^\T \bSigma_\X/N]_{ij}~\asto~\EE_\w\left[\cos(\x_i^\T \w) \cos(\w^\T \x_j) + \sin(\x_i^\T \w) \sin(\w^\T \x_j) \right] \equiv \K_{\cos} + \K_{\sin},
\end{equation}
with 
\begin{equation}
    \K_{\cos} + \K_{\sin}\equiv e^{-\frac12 (\| \x_i \|^2 + \| \x_j \|^2) } \left( \cosh(\x_i^\T \x_j) + \sinh(\x_i^\T \x_j) \right) = e^{-\frac12 (\| \x_i - \x_j \|^2) } \equiv [\K_\X]_{ij}. \label{eq:Gram-large-N}
\end{equation}

While this result holds in the $N \to \infty$ limit, recent advances in random matrix theory \cite{louart2018random,liao2018spectrum} suggest that, in the more practical setting where $N$ is not much larger than $n,p$ and $n,p,N \to \infty$ at the same pace,
the situation is more subtle.
In particular, the above entry-wise convergence remains valid, but the convergence $\| \bSigma_\X^\T \bSigma_\X/N - \K_\X \| \to 0$ no longer holds in spectral norm, due to the factor $n$, now large, in the norm inequality $\| \A \|_\infty \le \| \A \| \le n \| \A \|_\infty$ for $\A \in \RR^{n \times n}$ and $\| \A \|_\infty \equiv \max_{ij} |\A_{ij}|$. 
This implies that, in the large $n,p,N$ regime, the assessment of the behavior of $\bSigma_\X^\T \bSigma_\X/N$ via $\K_\X$ may result in a spectral norm error that blows up. 
As a consequence, for various machine learning algorithms \cite{cortes2010impact}, the performance guarantee offered by the limiting Gaussian kernel is less likely to agree with empirical observations in real-world large-scale problems, when $n,p$ are~large.

\subsection{Warm-up: Sample Covariance Matrix and the Mar{\u c}enko-Pastur Equation}
\label{subsec:SCM-and-MP}

As a warm-up example for the large $n,p,N$ mismatch issue that we shall address, consider the sample covariance matrix $\hat \C = \frac1n \X \X^\T$ from some data $\X \in \RR^{p \times n}$ composed of $n$ i.i.d.~$\x_i \sim \NN(\mathbf{0}, \C)$ with positive definite $\C \in \RR^{p \times p}$. In this zero-mean Gaussian setting, the sample covariance $\hat \C$, despite being the maximum likelihood estimator of the \emph{population covariance} $\C$ and providing \emph{entry-wise} consistent estimate for it, is an extremely poor estimator of $\C$ in a \emph{spectral norm} sense, for $n,p$ large. More precisely, $\| \hat \C - \C \| \not\to 0$ as $n,p \to \infty$ with $p/n \to c \in (0,\infty)$.
Indeed, one has $\| \hat \C - \C \|/\| \C \| \approx 20\%$, even with $n=100p$, in the simple $\C = \I_p$ setting. Figure~\ref{fig:MP-law} compares the eigenvalue histogram of $\hat \C$ with the population eigenvalue of $\C$, in the setting of $\C = \I_p$ and $n = 100p$. In the $\C = \I_p$ case, the limiting eigenvalue distribution of $\hat \C$ as $n,p \to \infty$ is known to be the popular Mar{\u c}enko-Pastur law \cite{marvcenko1967distribution} given by
\begin{equation}\label{eq:MP-law}
  \mu(dx) = (1 - c^{-1}) \cdot \delta_0(x) + \frac1{2\pi c x} \sqrt{ \left( x - (1-\sqrt c)^2 \right)^+ \left( (1 + \sqrt c)^2 -x \right)^+ } dx
\end{equation}
with $\delta_0(x)$ the Dirac mass at zero, $c = \lim p/n$ and $(x)^+ = \max(x,0)$, so that the support of $\mu$ has length $(1+\sqrt c)^2 - (1-\sqrt c)^2 = 4 \sqrt c = 0.4$ for $n = 100 p$.

\begin{figure}[htb]
\centering
\begin{tikzpicture}[font=\footnotesize]
    \renewcommand{\axisdefaulttryminticks}{4} 
    \pgfplotsset{every major grid/.append style={densely dashed}}       
    \pgfplotsset{every axis legend/.append style={cells={anchor=west},fill=white, at={(0.98,0.98)}, anchor=north east, font=\scriptsize }}
    \begin{axis}[
    width = .6\linewidth,
      height = .45\linewidth,
      xmin=0.7,
      ymin=0,
      xmax=1.3,
      ymax=4.5,
      bar width=2pt,
      grid=major,
      ymajorgrids=false,
      scaled ticks=true,
      xlabel={Eigenvalues of $\hat \C$},
      ylabel={}
      ]
      \addplot+[ybar,mark=none,draw=white,fill=blue!60!white,area legend] coordinates{
      (0.803814, 0.000000)(0.811441, 0.256076)(0.819068, 1.536458)(0.826695, 1.280382)(0.834322, 1.536458)(0.841949, 2.560764)(0.849576, 2.048611)(0.857203, 1.792535)(0.864831, 2.816840)(0.872458, 2.816840)(0.880085, 2.816840)(0.887712, 2.816840)(0.895339, 2.560764)(0.902966, 3.328993)(0.910593, 3.072917)(0.918220, 3.328993)(0.925847, 3.072917)(0.933475, 3.072917)(0.941102, 3.328993)(0.948729, 2.816840)(0.956356, 3.072917)(0.963983, 3.328993)(0.971610, 3.072917)(0.979237, 3.328993)(0.986864, 3.328993)(0.994492, 3.072917)(1.002119, 3.072917)(1.009746, 3.328993)(1.017373, 3.328993)(1.025000, 2.816840)(1.032627, 3.072917)(1.040254, 3.328993)(1.047881, 2.560764)(1.055508, 2.816840)(1.063136, 3.328993)(1.070763, 2.560764)(1.078390, 3.072917)(1.086017, 2.304687)(1.093644, 2.816840)(1.101271, 2.560764)(1.108898, 2.816840)(1.116525, 2.560764)(1.124153, 2.048611)(1.131780, 2.048611)(1.139407, 2.560764)(1.147034, 1.792535)(1.154661, 1.792535)(1.162288, 1.792535)(1.169915, 2.048611)(1.177542, 1.536458)(1.185169, 1.024306)(1.192797, 1.024306)(1.200424, 0.768229)(1.208051, 0.256076)(1.215678, 0.000000)(1.223305, 0.000000)(1.230932, 0.000000)(1.238559, 0.000000)(1.246186, 0.000000)(1.253814, 0.000000)
      };
      \addlegendentry{{Empirical eigenvalues}}
      \def\c{0.01}
      \addplot[samples=200,domain=0.7:1.3,RED,line width=1pt] {1/(2*pi*\c*x)*sqrt(max(((1+sqrt(\c))^2-x)*(x-(1-sqrt(\c))^2),0))};
      \addlegendentry{{Mar\u{c}enko-Pastur law} }
      \addplot+[ybar,mark=none,draw=white,fill=black,area legend] coordinates{(1, 5)};
      \addlegendentry{{ Population eigenvalue} }
    \end{axis}
  \end{tikzpicture}
  \caption{ Eigenvalue histogram of $\hat \C$ versus the {Mar\u{c}enko-Pastur law, for $p=512$ and $ n= 100p$. }}
  \label{fig:MP-law}
  \end{figure}
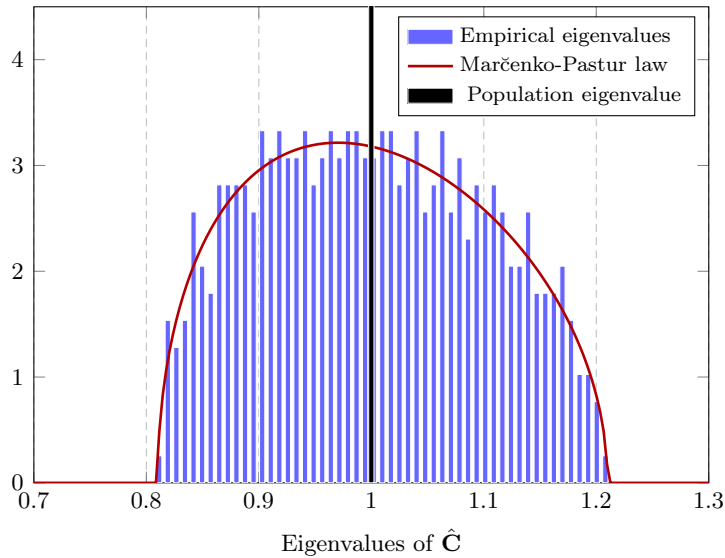

In the regression analysis (such as ridge regression) based on $\X$, of more immediate interest is the \emph{resolvent} $\Q_{\hat \C}(\lambda) \equiv ( \hat \C + \lambda \I_p )^{-1}, \lambda > 0$ of the sample covariance $\hat \C$, and more concretely, the bilinear forms of the type $\mathbf{a}^\T \Q_{\hat \C}(\lambda) \mathbf{b}$ for $\mathbf{a}, \mathbf{b} \in \RR^p$. As a result of the spectral norm inconsistency $\| \hat \C - \C \| \not \to 0$ in the large $n,p$ regime, it is unlikely that for most $\mathbf{a}, \mathbf{b}$, the convergence $\mathbf{a}^\T \Q_{\hat \C}(\lambda) \mathbf{b} - \mathbf{a}^\T (\C + \lambda \I_p)^{-1} \mathbf{b} \to 0$ would still hold.


While the \emph{random} variable $\mathbf{a}^\T \Q_{\hat \C}(\lambda) \mathbf{b}$ is not getting close to $\mathbf{a}^\T (\C + \lambda \I_p)^{-1} \mathbf{b}$ as $n,p \to \infty$, it does exhibit a tractable asymptotically \emph{deterministic} behavior, described by the Mar{\u c}enko-Pastur equation \cite{marvcenko1967distribution} for $\C = \I_p$. Notably, for $\mathbf{a}, \mathbf{b} \in \RR^p$ deterministic vectors of bounded Euclidean norms, we have, as $n,p \to \infty$ and $p/n \to c \in (0,\infty)$,
\[
  \mathbf{a}^\T \Q_{\hat \C}(\lambda) \mathbf{b} - m(\lambda) \cdot \mathbf{a}^\T \mathbf{b}~\asto~0,
\]
with $m(\lambda)$ the unique positive solution to the following Mar{\u c}enko-Pastur equation \cite{marvcenko1967distribution}
\begin{equation}\label{eq:MP}
  c\lambda m^2(\lambda) + (1+\lambda -c ) m(\lambda) - 1 = 0.
\end{equation}
In a sense, $\bar \Q(\lambda) \equiv m(\lambda) \I_p$ can be seen as a \emph{deterministic equivalent} \cite{hachem2007deterministic,couillet2011random} for the \emph{random} $\Q_{\hat \C}(\lambda)$ that asymptotically characterizes the behavior of the latter, when bilinear forms are considered. In Figure~\ref{fig:MP-DE} we compare the quadratic forms $\mathbf{a}^\T \Q_{\hat \C}(\lambda) \mathbf{a}$ as a function of $\lambda$, for $n = 10p$ and $n = 2p$. We observe that, in both cases, the RMT prediction in \eqref{eq:MP} provides a much closer match than the large-$n$ alone asymptotic given by $\mathbf{a}^\T (\C + \lambda \I_p)^{-1} \mathbf{a}$. This, together with Figure~\ref{fig:compare-kernel-RMT} on RFF ridge regression model, conveys a strong practical motivation of this work.

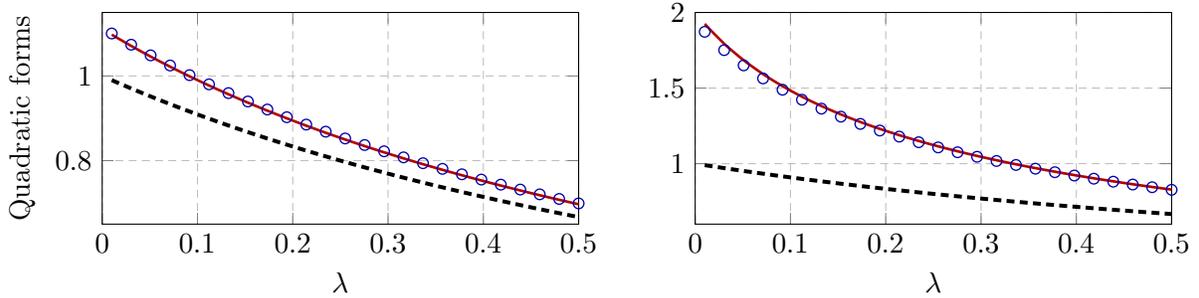
\begin{figure}[htb]
\centering
\begin{tabular}{cc}
\begin{tikzpicture}
    \pgfplotsset{every major grid/.style={style=densely dashed}}
    \begin{axis}[
      width = .45\linewidth,
      height = .25\linewidth,
      xmin=0,
      xmax=0.5,
      ymin=0.65,
      ymax=1.15,
      grid=major,
      scaled ticks=true,
      xlabel={ $\lambda$},
      ylabel= {Quadratic forms},
      legend style = {at={(0.02,0.98)}, anchor=north west, font=\footnotesize}
      ]
      \addplot[BLUE,only marks,mark=o,line width=.5pt] coordinates{
          (0.010000, 1.100330)(0.030417, 1.073815)(0.050833, 1.048648)(0.071250, 1.024724)(0.091667, 1.001947)(0.112083, 0.980231)(0.132500, 0.959501)(0.152917, 0.939686)(0.173333, 0.920724)(0.193750, 0.902559)(0.214167, 0.885138)(0.234583, 0.868415)(0.255000, 0.852347)(0.275417, 0.836894)(0.295833, 0.822020)(0.316250, 0.807692)(0.336667, 0.793879)(0.357083, 0.780552)(0.377500, 0.767686)(0.397917, 0.755255)(0.418333, 0.743238)(0.438750, 0.731614)(0.459167, 0.720362)(0.479583, 0.709464)(0.500000, 0.698904)
      };
      \addplot[black,densely dashed,smooth,line width=1.5pt] coordinates{
          (0.010000, 0.990099)(0.030417, 0.970481)(0.050833, 0.951626)(0.071250, 0.933489)(0.091667, 0.916031)(0.112083, 0.899213)(0.132500, 0.883002)(0.152917, 0.867365)(0.173333, 0.852273)(0.193750, 0.837696)(0.214167, 0.823610)(0.234583, 0.809990)(0.255000, 0.796813)(0.275417, 0.784057)(0.295833, 0.771704)(0.316250, 0.759734)(0.336667, 0.748130)(0.357083, 0.736874)(0.377500, 0.725953)(0.397917, 0.715350)(0.418333, 0.705053)(0.438750, 0.695048)(0.459167, 0.685323)(0.479583, 0.675866)(0.500000, 0.666667)
      };
      \addplot[RED,smooth,line width=1pt] coordinates{
          (0.010000, 1.097577)(0.030417, 1.071037)(0.050833, 1.045861)(0.071250, 1.021940)(0.091667, 0.999175)(0.112083, 0.977480)(0.132500, 0.956775)(0.152917, 0.936992)(0.173333, 0.918066)(0.193750, 0.899939)(0.214167, 0.882559)(0.234583, 0.865879)(0.255000, 0.849855)(0.275417, 0.834447)(0.295833, 0.819618)(0.316250, 0.805335)(0.336667, 0.791568)(0.357083, 0.778286)(0.377500, 0.765465)(0.397917, 0.753078)(0.418333, 0.741105)(0.438750, 0.729523)(0.459167, 0.718314)(0.479583, 0.707458)(0.500000, 0.696938)
      };
    \end{axis}
  \end{tikzpicture}
  &
  \begin{tikzpicture}
    \pgfplotsset{every major grid/.style={style=densely dashed}}
    \begin{axis}[
      width = .45\linewidth,
      height = .25\linewidth,
      xmin=0,
      xmax=0.5,
      ymin=0.6,
      ymax=2,
      grid=major,
      scaled ticks=true,
      xlabel={ $\lambda$},
      ylabel= {},
      legend style = {at={(0.02,0.98)}, anchor=north west, font=\footnotesize}
      ]
      \addplot[BLUE,only marks,mark=o,line width=.5pt] coordinates{
      (0.010000, 1.872020)(0.030417, 1.751010)(0.050833, 1.650070)(0.071250, 1.563859)(0.091667, 1.488910)(0.112083, 1.422840)(0.132500, 1.363945)(0.152917, 1.310960)(0.173333, 1.262923)(0.193750, 1.219085)(0.214167, 1.178850)(0.234583, 1.141740)(0.255000, 1.107362)(0.275417, 1.075391)(0.295833, 1.045555)(0.316250, 1.017625)(0.336667, 0.991404)(0.357083, 0.966725)(0.377500, 0.943441)(0.397917, 0.921427)(0.418333, 0.900572)(0.438750, 0.880777)(0.459167, 0.861958)(0.479583, 0.844036)(0.500000, 0.826944)
      };
      \addplot[black,densely dashed,smooth,line width=1.5pt] coordinates{
      (0.010000, 0.990099)(0.030417, 0.970481)(0.050833, 0.951626)(0.071250, 0.933489)(0.091667, 0.916031)(0.112083, 0.899213)(0.132500, 0.883002)(0.152917, 0.867365)(0.173333, 0.852273)(0.193750, 0.837696)(0.214167, 0.823610)(0.234583, 0.809990)(0.255000, 0.796813)(0.275417, 0.784057)(0.295833, 0.771704)(0.316250, 0.759734)(0.336667, 0.748130)(0.357083, 0.736874)(0.377500, 0.725953)(0.397917, 0.715350)(0.418333, 0.705053)(0.438750, 0.695048)(0.459167, 0.685323)(0.479583, 0.675866)(0.500000, 0.666667)
      };
      \addplot[RED,smooth,line width=1pt] coordinates{
      (0.010000, 1.924474)(0.030417, 1.793120)(0.050833, 1.684501)(0.071250, 1.592409)(0.091667, 1.512847)(0.112083, 1.443092)(0.132500, 1.381206)(0.152917, 1.325764)(0.173333, 1.275685)(0.193750, 1.230134)(0.214167, 1.188453)(0.234583, 1.150110)(0.255000, 1.114677)(0.275417, 1.081796)(0.295833, 1.051172)(0.316250, 1.022556)(0.336667, 0.995736)(0.357083, 0.970531)(0.377500, 0.946785)(0.397917, 0.924363)(0.418333, 0.903146)(0.438750, 0.883030)(0.459167, 0.863924)(0.479583, 0.845747)(0.500000, 0.828427)
      };
    \end{axis}
  \end{tikzpicture}
\end{tabular}
\caption{ Quadratic forms $\mathbf{a}^\T \Q_{\hat \C}(\lambda) \mathbf{a}$ as a function of $\lambda$, for $p=512$, $n=10p$ \textbf{(left)} and $n = 2p$ \textbf{(right)}. Empirical results displayed in {\BLUE \bf blue} circles; population predictions $\mathbf{a}^\T (\C + \lambda \I_p)^{-1} \mathbf{a}$ (assuming $n \to \infty$ alone with $p$ fixed) in {\bf black} dashed lines; and RMT prediction from \eqref{eq:MP} in {\RED \bf red} solid lines. Results obtained by averaging over $50$ runs.}
\label{fig:MP-DE}
\end{figure}

\subsection{Our Main Contributions}

We consider the RFF model in the more realistic large $n,p,N$ limit.
While, in this setting, the RFF empirical Gram matrix does \emph{not} converge to the Gaussian kernel matrix, we can characterize the Gram matrix behavior as $n,p,N\rightarrow\infty$ and provide \emph{asymptotic performance guarantees} for RFF on large-scale problems.
We also identify a phase transition as a function of the ratio $N/n$, including the corresponding double descent phenomenon.
In more detail, our contributions are the~following.
\begin{enumerate}
  \item
  We provide a \emph{precise} characterization of the asymptotics of the RFF empirical Gram matrix, in the large $n,p,N$ limit (Theorem~\ref{theo:asy-behavior-E[Q]}).
  This is accomplished by constructing a deterministic equivalent for the resolvent of the RFF Gram matrix.
  Based on this, the asymptotic behavior of the RFF model is accessible through a fixed-point equation, that can be interpreted in terms of an angle-like correction induced by the non-trivial large $n,p,N$ limit (relative to the $N\rightarrow\infty$ alone limit). 
  \item
  We derive the asymptotic training and test mean squared errors (MSEs) of RFF ridge regression, as a function of the ratio $N/n$, the regularization penalty $\lambda$, and the training as well as test sets (Theorem~\ref{theo:asy-training-MSE}~and~\ref{theo:asy-test-MSE}, respectively).
  We identify precisely the under- to over-parameterization phase transition, as a function of the relative model complexity $N/n$; and we prove the existence of a ``singular'' peak of test error at the $N/n = 1/2$ boundary that characterizes the \emph{double descent} behavior. 
  Importantly, our result is valid \emph{with almost no specific assumption} on the data distribution. 
  This is a significant improvement over existing double descent analyses, which fundamentally rely on the knowledge of the data distribution (often assumed to be Gaussian for simplicity)~\cite{hastie2019surprises,mei2019generalization}.
  \item
  We provide a detailed empirical evaluation of our theoretical results, demonstrating that the theory closely matches empirical results on a range of real-world data sets (Section~\ref{sec:empirical_main} and~\ref{sec:empirical_additional}).
  This includes demonstrating the correction due to the large $n,p,N$ limit, sharp transitions (as a function of $N/n$) in angle-like quantities that disappear as the regularization increases, and the corresponding double descent. 
  This also includes an evaluation of the impact of training-test similarity and the effect of different data sets, thus confirming that (unlike in prior work) the phase transition and double descent hold generally without specific assumption on the data distribution.
\end{enumerate}

\subsection{Related Work}

Here, we provide a brief review of related previous efforts. 

\paragraph{Random features and limiting kernels.} 
In most RFF work~\cite{rahimi2009weighted,bach2017equivalence,avron2017random,rudi2017generalization}, non-asymptotic bounds are given, on the number of random features $N$ needed for a predefined approximation error, for a given kernel matrix with fixed $n,p$. 
A more recent line of work  \cite{allen2019convergence,du2019gradient,jacot2018neural,chizat2019lazy} has focused on the over-parameterized $N \to \infty$ limit of large neural networks by studying the corresponding \emph{neural tangent kernels}.
Here, we position ourselves in the more practical regime where $n,p,N $ are all large and comparable, and we provide \emph{asymptotic performance guarantees} that better fit large-scale problems compared to the large-$N$ analysis.

\paragraph{Random matrix theory.} 
From a random matrix theory perspective, nonlinear Gram matrices of the type $\bSigma_\X^\T \bSigma_\X$ have recently received an unprecedented research interests, due to their close connection to neural networks \cite{pennington2017nonlinear,pennington2017resurrecting,benigni2019eigenvalue,pastur2020random}, with a particular focus on the associated eigenvalue distribution. Here we propose a deterministic equivalent \cite{couillet2011random,hachem2007deterministic} analysis for the resolvent matrix that provides access, not only to the eigenvalue distribution, but also to the regression error of central interest in this article. While most existing deterministic equivalent analyses are performed on linear models, here we focus on the \emph{nonlinear} RFF model. From a technical perspective, the most relevant (random matrix theory) work is \cite{louart2018random,mei2019generalization}. 
We improve their results by considering \emph{generic} data model on the nonlinear RFF model.

\paragraph{Statistical mechanics of learning.} 
There exits a long history of connections between statistical mechanics and machine learning models (such as neural networks), including a range of techniques to establish generalization bounds~\cite{SST92,WRB93,DKST96,EB01_BOOK}, and recently there has been renewed interest~\cite{MM17_TR,MM19_HTSR_ICML,MM19_KDD,MM20_SDM,BKPx20}.
The relevance of this work to our results lies in the use of the thermodynamic limit (akin to the large $n,p,N$ limit), rather than the classical limits more commonly used in statistical learning theory, where uniform convergence bounds and related techniques can be applied. 

\paragraph{Double descent in large-scale learning systems.} 
The large $n,N$ asymptotics of statistical models has received considerable research interests in machine learning \cite{pennington2018emergence,hastie2019surprises}, resulting in a counterintuitive phenomenon referred to as the ``double descent.''
Instead of focusing on different ``phases of learning''~\cite{SST92,WRB93,DKST96,EB01_BOOK,MM17_TR},
the ``double descent'' phenomenon focuses on an empirical manifestation of the phase transition, and it refers to the empirical observations about the form of the test error curve as a function of the model complexity, which differs from the usual textbook description of the bias-variance tradeoff~\cite{advani2020high,belkin2019reconciling,friedman2001elements}.
Theoretical investigation into this phenomenon mainly focuses on the generalization property of various regression models \cite{dobriban2018high,bartlett2020benign,deng2019model,liang2020just,hastie2019surprises,mei2019generalization}.
In most cases, quite specific (and rather strong) assumptions are imposed on the input data distribution. In this respect, our work extends the analysis in \cite{mei2019generalization} to handle the RFF model and its phase structure \emph{on real-world data sets}.

\subsection{Organization of the Paper}
Throughout this article, we follow the convention of denoting scalars by lowercase, vectors by lowercase boldface, and matrices by uppercase boldface letters.  In addition, the notation $(\cdot)^\T$ denotes the transpose operator; the norm $\| \cdot \| $ is the Euclidean norm for vectors and the spectral or operator norm for matrices; and $\asto$ stands for almost sure convergence of random variables.

Our main results on the asymptotic behavior of the RFF resolvent matrix, as well as of the training MSE and testing MSE of RFF ridge regression are presented in Section~\ref{sec:main}, with detailed proofs deferred to the Appendix. 
In Section~\ref{sec:empirical_main}, we provide a detailed empirical evaluation of our main results; and
in Section~\ref{sec:empirical_additional}, we provide additional empirical evaluation on real-world data, illustrating the practical effectiveness of the proposed analysis. 
Concluding remarks are placed in Section~\ref{sec:conclusion}.

\section{Main Technical Results}
\label{sec:main}

In this section, we present our main theoretical results.
To investigate the large $n,p,N$ asymptotics of the RFF model, we shall technically position ourselves under the following assumption.
\begin{Assumption}
\label{ass:high-dim}
As $n \to \infty$, we have
\begin{enumerate}
  \item $ 0 < \liminf_n \min\{\frac{p}n, \frac{N}n\} \le \limsup_n \max\{ \frac{p}n, \frac{N}n \} < \infty$; or, practically speaking, the ratios $p/n$ and $N/n$ are only moderately large or moderately small. 
  \item $ \limsup_n \| \X \| < \infty$ and $\limsup_n \| \y \|_\infty < \infty$, i.e., they are both normalized with respect to $n$. 
\end{enumerate}
\end{Assumption}

\noindent
Under Assumption~\ref{ass:high-dim}, we consider the RFF regression model as in Figure~\ref{fig:RFF-regression}.

\begin{figure}[!hbt]
\centering
\begin{minipage}{0.7\columnwidth}
\centering
\begin{tikzpicture}[node distance = 0.04\linewidth, auto]
\tikzstyle{neuron} = [circle, draw=white, fill=blue!20!white, minimum height=0.03\linewidth, inner sep=0pt]
    \draw [decorate,decoration={brace,amplitude=10pt}] (-0.015\linewidth,-0.26\linewidth) -- (-0.015\linewidth,0.02\linewidth) node [black,midway] {};
    \node [neuron] (neuron 11) {};
    \node [neuron, below of=neuron 11] (neuron 12) {};
    \node [neuron, below of=neuron 12] (neuron 13) {};
    \node [neuron, below of=neuron 13] (neuron 14) {};
    \node [neuron, below of=neuron 14] (neuron 15) {};
    \node [neuron, below of=neuron 15] (neuron 16) {};
    \node [neuron, below of=neuron 16] (neuron 17) {};
    \node [below of=neuron 14, yshift=-0.15\linewidth]{\makecell{$\X \in \RR^{p \times n}$\\ $\hat \X \in \RR^{p \times \hat n}$}};
    \draw [decorate,decoration={brace,mirror,amplitude=10pt}] (0.015\linewidth,-0.26\linewidth) -- (0.015\linewidth,0.02\linewidth) node [black,midway] {};
    \node [neuron, left of=neuron 11, xshift=0.42 \linewidth, yshift=-0.02\linewidth] (neuron 21) {$\sin$};
    \node [neuron, below of=neuron 21,yshift=-0.02\linewidth] (neuron 22) {$\sin$};
    \node [below of=neuron 22] (neuron 23) {};
    \node [neuron, below of=neuron 23] (neuron 24) {$\cos$};
    \node [neuron, below of=neuron 24,yshift=-0.02\linewidth] (neuron 25) {$\cos$};
    \node [below of=neuron 23, yshift=-0.15\linewidth]{ \makecell{$\bSigma_\X^\T =  [\cos(\W \X)^\T,~ \sin(\W\X)^\T]$\\ $\bSigma_{\hat \X}^\T = [\cos(\W \hat \X)^\T,~\sin(\W\hat \X)^\T]$} };
    \node [above of=neuron 21,yshift=0.04\linewidth] {random Fourier features};
    \draw [decorate,decoration={brace,amplitude=10pt}] (0.36\linewidth,-0.24\linewidth) -- (0.36\linewidth,0) node [black,midway] {};
    \draw [decorate,decoration={brace,mirror,amplitude=10pt}] (0.4\linewidth,-0.24\linewidth) -- (0.4\linewidth,0) node [black,midway] {};
    \node [left of=neuron 23, xshift=0.35\linewidth, yshift=0.04\linewidth] (neuron 41) {};
    \node [neuron, below of=neuron 41] (neuron 42) {};
    \draw [->] (-0.08\linewidth,-0.12\linewidth) -- node[left] {$\X$ or $\hat \X$} (-0.05\linewidth, -0.12\linewidth);
    \draw [->] (0.06\linewidth,-0.12\linewidth) -- node[above] {$\W \in \RR^{N \times p}$} (0.32\linewidth,-0.12\linewidth);
    \draw [->] (0.44\linewidth,-0.12\linewidth) -- node[above] {$\bbeta \in \RR^{2N}$ in \eqref{eq:def-beta} } (neuron 42);
\end{tikzpicture}
\caption{Illustration of random Fourier features regression model.
}  
\label{fig:RFF-regression}
\end{minipage}
\end{figure}
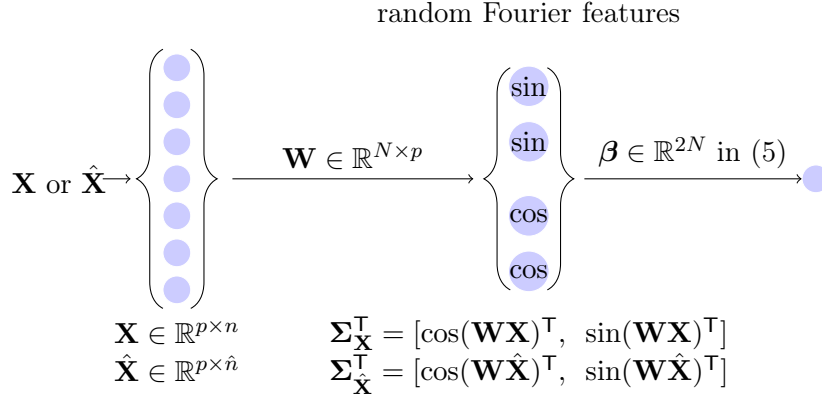

For training data $\X \in \RR^{p \times n}$ of size $n$, the associated random Fourier features, $\bSigma_\X \in \RR^{2N \times n}$, are obtained by computing $\W \X \in\RR^{N \times n}$ for standard Gaussian random matrix $\W \in \RR^{N \times p}$, and then applying entry-wise cosine and sine nonlinearities on $\W \X$, i.e.,
\[
    \bSigma_\X^\T  = \begin{bmatrix} \cos(\W \X)^\T & \sin(\W\X)^\T \end{bmatrix}, \quad \W_{ij} \sim \NN(0,1).
\]
Given this setup, the RFF ridge regressor $\bbeta \in \RR^{2N}$ is given by
\begin{equation}\label{eq:def-beta}
    \bbeta \equiv \frac1n \bSigma_\X \left( \frac1n \bSigma_\X^\T \bSigma_\X + \lambda \I_n \right)^{-1}\y \cdot 1_{2N > n} + \left( \frac1n \bSigma_\X \bSigma_\X^\T + \lambda \I_{2N} \right)^{-1} \frac1n \bSigma_\X\,\y \cdot 1_{2N < n}.
\end{equation}
The two forms of $\bbeta$ in (\ref{eq:def-beta}) are equivalent for any $\lambda > 0$ and minimize the (ridge-regularized) squared loss $\frac1n \| \y - \bSigma_\X^\T \bbeta \|^2 + \lambda \| \bbeta\|^2$ on the training set $(\X, \y)$. 
Our objective is to characterize the large $n,p,N$ asymptotics of both the \emph{training MSE}, $E_{\train}$, and the \emph{test MSE}, $E_{\test}$, defined as
\begin{equation}\label{eq:def-MSE}
    E_{\train} = \frac1n \| \y - \bSigma_\X^\T \bbeta \|^2, \quad 
    E_{\test}  = \frac1{\hat n} \| \hat \y - \bSigma_{\hat \X}^\T \bbeta \|^2, 
\end{equation}
with $\bSigma_{\hat \X}^\T \equiv \begin{bmatrix} \cos(\W \hat \X)^\T & \sin(\W \hat \X)^\T \end{bmatrix} \in \RR^{\hat n \times 2N}$ on a test set $(\hat \X, \hat \y)$ of size $\hat n$, and from this to characterize the phase transition behavior (as a function of the model complexity $N/n$) as mentioned in Section~\ref{sec:introduction}.
In the training phase, the random weight matrix $\W$ is drawn once and kept fixed; and the RFF ridge regressor $\bbeta$ is given explicitly as a function of $\W$ and the training set $(\X,\y)$, as per \eqref{eq:def-beta}. 
In the test phase, for $\bbeta$ now fixed, the model takes the test data $\hat \X$ as input, and it outputs $\bSigma_{\hat \X}^\T \bbeta$ that should be compared to the corresponding target~$\hat \y$ to measure the model test performance.

\subsection{Asymptotic Deterministic Equivalent}
\label{subsec:asy-equiv}

To start, we observe that the training MSE, $E_{\train}$, in \eqref{eq:def-MSE}, can be written as 
\begin{equation}\label{eq:E_train_rewrite}
    E_{\train} = \frac{\lambda^2}n \| \Q (\lambda) \y \|^2  = - \frac{\lambda^2}n \y^\T \frac{\partial \Q(\lambda)}{\partial \lambda} \y  ,
\end{equation}
and it depends on the quadratic form $\y^\T \Q(\lambda) \y$ of
\begin{equation}\label{eq:def-Q}
  \Q (\lambda) \equiv \left( \frac1n \bSigma_\X^\T \bSigma_\X + \lambda \I_n \right)^{-1} \in \RR^{n \times n},
\end{equation}
which is the \emph{resolvent} of $\frac1n \bSigma_\X^\T \bSigma_\X$ (also denoted $\Q$ when there is no ambiguity) with $\lambda > 0$. 
To see this, from \eqref{eq:def-MSE} we have $E_{\train} = \frac1n \| \y - \frac1n \bSigma_\X^\T \bSigma_\X ( \frac1n \bSigma_\X^\T \bSigma_\X + \lambda \I_n )^{-1}\y \|^2 = \frac{\lambda^2}n \| \Q(\lambda) \y \|^2 = - \frac{\lambda^2}n \y^\T \frac{\partial \Q(\lambda)}{\partial \lambda} \y$, with $\frac{\partial \Q(\lambda)}{\partial \lambda} = - \Q^2(\lambda)$. 

To assess the asymptotic training MSE, according to our discussion in Section~\ref{subsec:SCM-and-MP}, it suffices to find a deterministic equivalent for $\Q(\lambda)$ (i.e., a \emph{deterministic} matrix that captures the asymptotic behavior of the latter).
One possibility is its expectation $\EE_\W[\Q(\lambda)]$.
Informally, if the training MSE $E_{\train}$ (that is random due to random $\W$) is ``close to'' some deterministic $\bar E_{\train}$, in the large $n,p,N$ limit, then $\bar E_{\train}$ must have the same limit, as $\EE_\W[E_{\train}] = - \frac{\lambda^2}n \frac{\partial \y^\T \EE_\W[\Q(\lambda)] \y}{\partial \lambda}$ for $n,p,N \to \infty$.
However, $\EE_\W[\Q]$ involves integration (with no closed form due to the matrix inverse), and it is not a convenient quantity with which to work. 
Our objective is to find an asymptotic ``alternative'' for $\EE_\W[\Q]$ that is (i) close to $\EE_\W[\Q]$ in the large $n,p,N \to \infty$ limit and (ii) numerically more accessible.

In the following theorem, we introduce an asymptotic equivalent for $\EE_\W[\Q]$.
Instead of being directly related to the Gaussian kernel matrix $\K_\X = \K_{\cos} + \K_{\sin}$ as suggested by \eqref{eq:Gram-large-N} in the large-$N$ limit, $\EE_\W[\Q]$ depends on the two components $\K_{\cos}$ and $\K_{\sin}$ in a more involved manner, where we recall that
\[
    [\K_{\cos}]_{ij} = e^{-\frac12 (\| \x_i \|^2 + \| \x_j \|^2) } \cosh(\x_i^\T \x_j), \quad [\K_{\sin}]_{ij} = e^{-\frac12 (\| \x_i \|^2 + \| \x_j \|^2) } \sinh(\x_i^\T \x_j)
\]
for $\x_i, \x_j \in \RR^p$ the $i$-th and $j$-th column of $\X$, respectively.
Importantly, the proposed equivalent $\bar \Q$ can be numerically evaluated by running simple fixed-point iterations on $\K_{\cos}$ and $\K_{\sin}$.

\begin{Theorem}[Asymptotic equivalent for \texorpdfstring{$\EE_\W[\Q]$}{E[Q]}]\label{theo:asy-behavior-E[Q]}
Under Assumption~\ref{ass:high-dim}, for $\Q$ defined in \eqref{eq:def-Q} and $\lambda >0$, we have, as $n \to \infty$
\[
    \| \EE_\W[\Q] - \bar \Q \| \to 0
\]
for $\bar \Q \equiv \left( \frac{N}n \left(\frac{\K_{\cos} }{1+\delta_{\cos}} +  \frac{\K_{\sin} }{1+\delta_{\sin}} \right) + \lambda \I_n\right)^{-1}$ and $\K_{\cos} \equiv \K_{\cos}(\X,\X), \K_{\sin} \equiv \K_{\sin}(\X,\X) \in \RR^{n \times n}$ with
\begin{equation}\label{eq:def-K}
    [\K_{\cos}(\X,\X')]_{ij} = e^{-\frac12 (\| \x_i \|^2 + \| \x_j' \|^2) } \cosh(\x_i^\T \x_j'),\quad [\K_{\sin}(\X,\X')]_{ij} = e^{-\frac12 (\| \x_i \|^2 + \| \x_j' \|^2) } \sinh(\x_i^\T \x_j')
\end{equation}
where $(\delta_{\cos}, \delta_{\sin})$ is the unique positive solution to
\[
    \delta_{\cos} = \frac1n \tr (\K_{\cos} \bar \Q), \quad \delta_{\sin} = \frac1n \tr (\K_{\sin} \bar \Q).
\]
\end{Theorem}
\begin{proof}
See Section~\ref{sec:proof-of-theo-E[Q]} of the appendix.
\end{proof}

\begin{Remark}
\normalfont
Since $\frac{\K_{\cos} }{1+\delta_{\cos}} + \frac{\K_{\sin} }{1+\delta_{\sin}} \succeq \frac{\K}{1+\max(\delta_{\cos}, \delta_{\sin})}$, in the positive definite order, for $\K \equiv \K_{\cos} + \K_{\sin}$ the Gaussian kernel (see again Lemma~\ref{lem:expectation}), $\frac{\K_{\cos}}{1+\delta_{\cos}} + \frac{\K_{\sin}}{1+\delta_{\sin}}$ is positive definite, if $\x_1\, \ldots, \x_n$ are all distinct; see~\cite[Theorem~2.18]{scholkopf2001learning}. 
\end{Remark}

\begin{Remark}
\normalfont
Taking $N/n \to \infty$, one has $\delta_{\cos} \to 0$, $\delta_{\sin} \to 0$, so that 
\[
    \frac{\K_{\cos} }{1+\delta_{\cos}} + \frac{\K_{\sin} }{1+\delta_{\sin}} \to \K_{\cos} + \K_{\sin} = \K ,~\textmd{and}~\bar \Q \to \left( \frac{N}n \K + \lambda \I_n \right)^{-1} \sim \frac{n}N \K^{-1}  ,
\]
for $\lambda>0$, in accordance with the large-$N$ asymptotic prediction. 
In this sense, the pair $(\delta_{\cos}, \delta_{\sin})$ introduced in Theorem~\ref{theo:asy-behavior-E[Q]} accounts for the ``correction'' due to the non-trivial $N/n$ limit, as opposed to the $N \to \infty$ alone limit. 
Also, in the $N/n \to \infty$ limit, when the number of features $N$ is large, the regularization effect of $\lambda$ flattens out and $\bar \Q$ behaves like of (a scaled version of) the inverse Gaussian kernel matrix $\K^{-1}$.
\end{Remark}
 
\begin{Remark}
\label{rem:angle}
\normalfont
Since $\bar \Q$ shares the same eigenspace with $\frac{\K_{\cos} }{1+\delta_{\cos}} + \frac{\K_{\sin} }{1+\delta_{\sin}}$, one can geometrically interpret $(\delta_{\cos}, \delta_{\sin})$ as a sort of ``angle'' between the eigenspace of $\K_{\cos}, \K_{\sin} \in \RR^{n \times n}$ and that of $\frac{\K_{\cos} }{1+\delta_{\cos}} + \frac{\K_{\sin} }{1+\delta_{\sin}}$, weighted by the associated eigenvalues. 
For fixed $n$, as $N \to \infty$, we have
\[
    \frac1N \sum_{t=1}^N \cos(\X^\T \w_t) \cos(\w_t^\T \X) \to \K_{\cos}, \quad \frac1N \sum_{t=1}^N \sin(\X^\T \w_t) \sin(\w_t^\T \X) \to \K_{\sin},
\]
the eigenspaces of which are ``orthogonal'' to each other, so that $\delta_{\cos}, \delta_{\sin} \to 0$. On the other hand, as $N,n \to \infty$, the eigenspaces of $\K_{\cos}$ and $\K_{\sin}$ ``intersect'' with each other, captured by the non-trivial correction $(\delta_{\cos}, \delta_{\sin})$. 
\end{Remark}

\subsection{Asymptotic Training Performance}

Theorem~\ref{theo:asy-behavior-E[Q]} provides an asymptotically more tractable approximation of $\EE_\W[\Q]$ under the form of a fixed-point equation.
Together with some additional concentration arguments (e.g., from \cite[Theorem~2]{louart2018random}), this permits us to provide a complete description of (i) bilinear forms $\mathbf{a}^\T \Q \mathbf{b}$, for $\mathbf{a}, \mathbf{b} \in \RR^n$ of bounded norms, with $\mathbf{a}^\T \Q \mathbf{b} - \mathbf{a}^\T \bar \Q \mathbf{b}~\asto~0$, as $n,p,N \to \infty$; and (ii) the (normalized) trace of the type $\frac1n \tr \A \Q - \frac1n \tr \A \bar \Q~\asto~0$, for $\A$ of bounded operator norm.

The item (i), together with \eqref{eq:E_train_rewrite}, leads to the following result on the asymptotic training error. 

\begin{Theorem}[Asymptotic training performance]\label{theo:asy-training-MSE}
Under Assumption~\ref{ass:high-dim}, we have, for training MSE, $E_{\train}$ defined in \eqref{eq:def-MSE}, that, as $n \to \infty$
\begin{equation}
  E_{\train} - \bar E_{\train}~\asto~0, \quad \bar E_{\train} = \frac{\lambda^2}n \| \bar \Q \y \|^2  + \frac{N}n \frac{\lambda^2}n \begin{bmatrix} \frac{ \frac1n \tr (\bar \Q \K_{\cos} \bar \Q)}{ (1+\delta_{\cos})^2 } & \frac{ \frac1n \tr (\bar \Q \K_{\sin} \bar \Q) }{ (1+\delta_{\sin})^2 } \end{bmatrix} \bOmega \begin{bmatrix} \y^\T \bar \Q \K_{\cos} \bar \Q \y \\ \y^\T \bar \Q \K_{\sin} \bar \Q \y \end{bmatrix}
\end{equation}
for $\bar \Q$ defined in Theorem~\ref{theo:asy-behavior-E[Q]} and
\begin{equation}\label{eq:def-Omega}
  \bOmega^{-1} \equiv \I_2 - \frac{N}n \begin{bmatrix} \frac{ \frac1n \tr (\bar \Q \K_{\cos} \bar \Q \K_{\cos}) }{ (1+\delta_{\cos})^2 } & \frac{ \frac1n \tr (\bar \Q \K_{\cos} \bar \Q \K_{\sin}) }{ (1+\delta_{\sin})^2 } \\ \frac{ \frac1n \tr (\bar \Q \K_{\cos} \bar \Q \K_{\sin}) }{ (1+\delta_{\cos})^2 } & \frac{ \frac1n \tr (\bar \Q \K_{\sin} \bar \Q \K_{\sin}) }{ (1+\delta_{\sin})^2 } \end{bmatrix}.
\end{equation}
\end{Theorem}
\begin{proof}
See Section~\ref{sec:proof-theo-training-MSE} of the appendix.
\end{proof}

\begin{Remark}\label{rem:E-train-large-N}
\normalfont
Since $E_{\train} = \frac{\lambda^2}n \y^\T \Q^2 \y$, we can see in the expression of $\bar E_{\train}$ that there is not only a first-order (large $n,p,N$) correction in the first $\frac{\lambda^2}n \| \bar \Q \y \|^2$ term (which is different than $\frac{\lambda^2}n \| \Q \y \|^2$), but there is also a second-order correction, appearing in the form of $\bar \Q \K_\sigma \bar \Q$ or $\bar \Q \K_\sigma \bar \Q \K_\sigma$ for $\sigma \in \{\cos, \sin\}$, as in the second term.
This has a similar interpretation to Remark~\ref{rem:angle}, where the pair $(\delta_{\cos}, \delta_{\sin})$ in $\bar \Q$ is (geometrically) interpreted as the eigenspace ``intersection'' due to a non-vanishing $n/N$. 
In particular, taking $N/n \to \infty$, we have $ \bar \Q \sim \frac{n}N  \K^{-1}$, $\bOmega \to \I_2$ so that $ \bar E_{\train} = 0$ and the model interpolates the entire training set, as expected.
\end{Remark}

\begin{Remark}
\normalfont
One can show that (i) for a given $n$ and fixed $\lambda > 0$, the error $\bar E_{\train}$ decreases as the model size $N$ increases; and (ii) for a given ratio $N/n$, $\bar E_{\train}$ increases as the regularization $\lambda$ grows large. 
\end{Remark}

\subsection{Asymptotic Test Performance}

Theorem~\ref{theo:asy-training-MSE} holds without any restriction on the training set, $(\X,\y)$, except for Assumption~\ref{ass:high-dim}, since only the randomness of $\W$ is involved, and thus one can simply treat $(\X,\y)$ as known in this result. 
This is no longer the case when the test error is considered. 
Intuitively, the test data $\hat \X$ cannot be chosen arbitrarily (with respect to the training data), and one must ensure that the test data ``behave'' statistically like the training data, in a ``well-controlled'' manner, so that the test MSE is asymptotically deterministic and bounded as $n,\hat n,p,N \to \infty$. 
Following this intuition, we work under the following assumption. 
\begin{Assumption}[Data as concentrated random vectors \cite{louart2018concentration}]
\label{ass:data-concent}
The training data $\x_i \in \RR^p, i \in \{1,\ldots,n\}$, are independently drawn (non-necessarily uniformly) from one of $K>0$ distribution classes\footnote{$K \ge 2$ is included to cover multi-class classification problems; and $K$ should remain fixed as $n,p \to \infty$.} $\mu_1, \ldots, \mu_K$. There exist constants $C, \eta, q > 0$ such that for any $\x_i \sim \mu_k, k \in \{1,\ldots,K\}$ and any $1$-Lipschitz function $f: \RR^p \to \RR$, we have 
\begin{equation}\label{eq:def-concentration}
  \mathbb P \left( \left| f(\x_i) - \EE[f(\x_i)] \right| \ge t \right) \le C e^{-(t/\eta)^q}, \quad t \ge 0.
\end{equation}
The test data $\hat \x_i \sim \mu_k$, $i \in \{1,\ldots,\hat n\}$ are mutually independent, but \emph{may depend on} training data $\X$ and $\| \EE[\sigma(\W \X) - \sigma(\W \hat \X)] \| = O(\sqrt n)$ for $\sigma \in \{ \cos, \sin\}$. 
\end{Assumption}

To facilitate discussion of the phase transition and the double descent, we do not assume independence between training data and test data (but we do assume independence between different columns within $\X$ and $\hat \X$).
In this respect, Assumption~\ref{ass:data-concent} is weaker than the classical i.i.d.~assumption on the training and test samples. 
In particular, under Assumption~\ref{ass:data-concent} we permit $\hat \X = \X$, so that the test MSE coincides with the training MSE, as well as $\hat \X = \X + \boldsymbol{\varepsilon}$, for some independent random noise $\boldsymbol{\varepsilon}$.
This permits us to illustrate the impact of training-test data similarity on the RFF model performance (Section~\ref{subsec:impact-train-test-similarity}).

The simplest example of concentrated random vectors satisfying \eqref{eq:def-concentration} is the standard Gaussian vector $\mathcal N(\zo, \I_p)$ \cite{ledoux2005concentration}. 
Moreover, since the concentration property in \eqref{eq:def-concentration} is stable over Lipschitz transformations \cite{louart2018concentration}, for any $1$-Lipschitz mapping $g: \RR^d \mapsto \RR^p$ and $\z \sim \mathcal N(\zo, \I_d)$,  $g(\z)$ also satisfies \eqref{eq:def-concentration}. 
In this respect, Assumption~\ref{ass:data-concent}, although seemingly quite restrictive, represents a large family of ``generative models'' (including ``fake images'' generated by modern generative adversarial networks (GANs) that are, by construction, Lipschitz transformations of large random Gaussian vectors \cite{goodfellow2014generative,seddik2020random}). 
As such, from a practical consideration, Assumption~\ref{ass:data-concent} can provide a more realistic and flexible statistical model for real-world data.

With Assumption~\ref{ass:data-concent}, we now present the following result on the asymptotic test error.

\begin{Theorem}[Asymptotic test performance]\label{theo:asy-test-MSE}
Under Assumptions~\ref{ass:high-dim}~and~\ref{ass:data-concent}, we have, for $\lambda > 0$, test MSE $E_{\test}$ defined in \eqref{eq:def-MSE} and test data $(\hat \X, \hat \y)$ satisfying $ \limsup_{\hat n} \| \hat \X \| < \infty$, $ \limsup_{\hat n} \| \hat \y \|_\infty < \infty$ with $\hat n/ n \in (0, \infty)$ that, as $n \to \infty$
\begin{equation}
  E_{\test} - \bar E_{\test}~\asto~0, \quad \bar E_{\test} = \frac1{\hat n} \| \hat \y - \frac{N}n \hat \bPhi \bar \Q \y \|^2  + \frac{N^2}{n^2} \frac1{\hat n} \begin{bmatrix} \frac{\Theta_{\cos}}{(1+\delta_{\cos})^2} & \frac{\Theta_{\sin}}{(1+\delta_{\sin})^2} \end{bmatrix} \bOmega \begin{bmatrix} \y^\T \bar \Q \K_{\cos} \bar \Q \y \\ \y^\T \bar \Q \K_{\sin} \bar \Q \y \end{bmatrix}, 
\end{equation}
for $\bOmega$ in \eqref{eq:def-Omega},
\begin{equation}\label{eq:def-Theta}
    \Theta_\sigma = \frac1N \tr \K_\sigma (\hat \X, \hat \X) + \frac{N}n \frac1n \tr \bar \Q \hat \bPhi^\T \hat \bPhi \bar \Q \K_\sigma - \frac2n \tr \bar \Q \hat \bPhi^\T \K_\sigma (\hat \X, \X), \quad \sigma \in \{\cos, \sin\}, 
\end{equation}
and
\[
  \bPhi \equiv \frac{\K_{\cos}}{1+\delta_{\cos}} + \frac{\K_{\sin}}{1+\delta_{\sin}},~\hat \bPhi \equiv \frac{\K_{\cos}(\hat \X, \X)}{1+\delta_{\cos}} + \frac{\K_{\sin}(\hat \X, \X)}{1+\delta_{\sin}}  ,
\]
with $\K_{\cos}(\hat \X, \X), \K_{\sin}(\hat \X, \X) \in \RR^{\hat n \times n}$ and  $\K_{\cos}(\hat \X, \hat \X), \K_{\sin}(\hat \X, \hat \X) \in \RR^{\hat n \times \hat n}$ defined in \eqref{eq:def-K}.
\end{Theorem}

\begin{proof}
See Section~\ref{sec:proof-theo-test-MSE} of the appendix.
\end{proof}

\begin{Remark}
\normalfont
Similar to Theorem~\ref{theo:asy-training-MSE} on $\bar E_{\train}$, here the expression for $\bar E_{\test}$ is also given as the sum of first- and second-order corrections. 
To see this, one can confirm, by taking $(\hat \X, \hat \y) = (\X,\y)$, that the first term in $\bar E_{\test}$ becomes
\[
    \frac1{\hat n} \| \hat \y - \frac{N}n \hat \bPhi \bar \Q \y \|^2 = \frac1{n} \| \y - \frac{N}n \bPhi \bar \Q \y \|^2 = \frac{\lambda^2}n \| \bar \Q \y \|^2
\]
and is equal to the first term in $\bar E_{\train}$, where we used the fact that $\frac{N}n \bPhi \bar \Q = \I_n  - \lambda \bar \Q$. 
The same also holds for the second term, so that one obtains $\bar E_{\test} = \bar E_{\train}$, with $(\hat \X, \hat \y) = (\X,\y)$, as expected. 
From this perspective, Theorem~\ref{theo:asy-test-MSE} can be seen as an extension of Theorem~\ref{theo:asy-training-MSE}, with the ``interaction'' between training and test data (e.g., test-versus-test $\K_\sigma(\hat \X, \hat \X)$ and test-versus-train $\K_\sigma(\hat \X, \X)$ interaction matrices) summarized in the scalar parameter $\Theta_\sigma$ defined in \eqref{eq:def-Theta}, for $\sigma \in \{\cos, \sin \}$.
\end{Remark}

\begin{Remark}
\normalfont
By taking $N/n \to \infty$, we have that $\bar \Q \sim \frac{n}N \K^{-1}$, $\Theta_\sigma \sim N^{-1}$, $\bOmega \to \I_2$, and consequently
\[
    \lim_{N/n \to \infty} \bar E_{\test} =  \frac1{\hat n} \| \hat \y - \K(\hat \X, \X) \K^{-1} \y \|^2   .
\]
This is the test MSE of classical Gaussian kernel regression, with $\K(\hat \X,\X)\equiv \K_{\cos}(\hat \X, \X) + \K_{\sin}(\hat \X, \X) \in \RR^{\hat n \times n}$ the test-versus-train Gaussian kernel matrix. 
As opposed to the training MSE discussed in Remark~\ref{rem:E-train-large-N}, here $\bar E_{\test}$ generally has a non-zero limit (that is, however, \emph{independent} of $\lambda$) as $N/n \to \infty$.
\end{Remark}

\section{Empirical Performance and Practical Implications}
\label{sec:empirical_main}

In this section, we provide a detailed empirical evaluation, including a discussion of the behavior of the fixed point equation in Theorem~\ref{theo:asy-behavior-E[Q]}, and its consequences in Theorem~\ref{theo:asy-training-MSE} and Theorem~\ref{theo:asy-test-MSE}.
In particular, we describe the behavior of the pair $(\delta_{\cos}, \delta_{\sin})$ that characterizes the necessary correction in the large $n,p,N$ regime, as a function of the regularization penalty $\lambda$ and the ratio $N/n$. 
This explains: (i) the mismatch between empirical regression errors from the Gaussian kernel prediction (Figure~\ref{fig:compare-kernel-RMT});
(ii) the behavior of $(\delta_{\cos}, \delta_{\sin})$ as a function of $\lambda$ (Figure~\ref{fig:trend-delta-lambda});
(iii) the behavior of $(\delta_{\cos}, \delta_{\sin})$ as a function of $N/n$, which clearly indicates two phases of learning and the transition between them (Figure~\ref{fig:trend-delta-N/n});
and (iv) the corresponding double descent test error curves (Figure~\ref{fig:double-descent-regularization-2}). 

\subsection{Correction due to the Large $n,p,N$ Regime}
\label{subsec:discuss-delta-versus-lambda}

The nonlinear Gram matrix $\frac1n \bSigma_\X^\T \bSigma_\X$ is \emph{not} close to the classical Gaussian kernel matrix $\K$ in the large $n,p,N$ regime and, as a consequence, its resolvent $\Q$, as well the training and test MSE, $E_{\train}$ and $E_{\test}$ (that are functions of $\Q$), behave quite differently from the Gaussian kernel predictions. 
Indeed, for $\lambda > 0$, the following equation determines the pair $(\delta_{\cos}, \delta_{\sin})$ that characterizes the correction when considering $n,p,N$ all large, compared to the large-$N$ only asymptotic regime:
\begin{equation}\label{eq:fixed-point-delta}
    \delta_{\cos} = \frac1n \tr \K_{\cos} \bar \Q, \quad \delta_{\sin} = \frac1n \tr \K_{\sin} \bar \Q, \quad \bar \Q = \left( \frac{N}n \left(\frac{\K_{\cos} }{1+\delta_{\cos}} +  \frac{\K_{\sin} }{1+\delta_{\sin}} \right) + \lambda \I_n \right)^{-1}  .
\end{equation}

To start, Figure~\ref{fig:compare-kernel-RMT} compares the training MSE of RFF ridge regression to the predictions from Gaussian kernel regression and to the predictions from our Theorem~\ref{theo:asy-training-MSE}, on the popular MNIST data set \cite{lecun1998gradient}. 
Observe that there is a significant gap for training errors between empirical results and the classical Gaussian kernel predictions, especially when $N/n < 1$, while our predictions \emph{consistently} fit empirical observations almost~perfectly.

Next, from \eqref{eq:fixed-point-delta}, we know that both $\delta_{\cos}$ and $\delta_{\sin}$ are decreasing function of $\lambda$.
(See Lemma~\ref{lem:delta-derivative-lambda} in Appendix~\ref{sec:detail-section-double-descent} for a proof of this fact.)
Figure~\ref{fig:trend-delta-lambda} shows that: 
(i) over a range of different $N/n$, both $\delta_{\cos}$ and $\delta_{\sin}$ decrease monotonically as $\lambda$ increases; 
(ii) the behavior for $N/n <1$, which is decreasing from an initial value of $\delta \gg 1$, is very different than the behavior for $N/n \gtrsim 1$, with an initially flat region where $\delta <1$ for all values of $\lambda$; and 
(iii) the impact of regularization $\lambda$ becomes less significant as the ratio $N/n$ becomes large.
This is in accordance with the limiting behavior of $\bar \Q \sim \frac{n}N \K^{-1}$ that is \emph{independent} of $\lambda$ as $N/n \to \infty$ in Remark~\ref{rem:E-train-large-N}.

\begin{figure}[t] 
\vskip 0.1in
\begin{center}
\begin{minipage}[b]{0.24\columnwidth}%
  \begin{tikzpicture}[font=\scriptsize]
    \pgfplotsset{every major grid/.style={style=densely dashed}}
    \begin{loglogaxis}[
      width=1.15\linewidth,
      xmin=1e-4,
      xmax=1e2,
      ymin=1e-6,
      ymax=1,
      grid=major,
      scaled ticks=true,
      xlabel={ $\lambda$},
      ylabel={ $E_{\train}$ },
      legend style = {at={(0.02,0.98)}, anchor=north west, font=\scriptsize}
      ]
      \addplot[BLUE,only marks,mark=o,line width=.5pt] coordinates{
          (0.000100,0.013769)(0.000178,0.014453)(0.000316,0.013338)(0.000562,0.013595)(0.001000,0.014577)(0.001778,0.013156)(0.003162,0.014088)(0.005623,0.015177)(0.010000,0.015686)(0.017783,0.018574)(0.031623,0.020165)(0.056234,0.023947)(0.100000,0.028299)(0.177828,0.034164)(0.316228,0.042351)(0.562341,0.053571)(1.000000,0.066200)(1.778279,0.083076)(3.162278,0.109571)(5.623413,0.147856)(10.000000,0.202520)(17.782794,0.303887)(31.622777,0.431313)(56.234133,0.580624)(100.000000,0.702857)
      };
      \addplot[black,densely dashed,smooth,line width=1.5pt] coordinates{
          (0.000100,0.000001)(0.000178,0.000003)(0.000316,0.000008)(0.000562,0.000022)(0.001000,0.000063)(0.001778,0.000145)(0.003162,0.000365)(0.005623,0.000794)(0.010000,0.001584)(0.017783,0.003403)(0.031623,0.005795)(0.056234,0.009471)(0.100000,0.014700)(0.177828,0.022086)(0.316228,0.030253)(0.562341,0.042542)(1.000000,0.057042)(1.778279,0.075512)(3.162278,0.101243)(5.623413,0.139996)(10.000000,0.198620)(17.782794,0.298390)(31.622777,0.428876)(56.234133,0.576972)(100.000000,0.705101)
      };
      \addplot[RED,smooth,line width=1pt] coordinates{
          (0.000100,0.013630)(0.000178,0.013974)(0.000316,0.013182)(0.000562,0.013512)(0.001000,0.014120)(0.001778,0.013493)(0.003162,0.014084)(0.005623,0.014739)(0.010000,0.015250)(0.017783,0.018365)(0.031623,0.020085)(0.056234,0.023310)(0.100000,0.027776)(0.177828,0.034376)(0.316228,0.041439)(0.562341,0.052913)(1.000000,0.066630)(1.778279,0.084426)(3.162278,0.109764)(5.623413,0.148265)(10.000000,0.206256)(17.782794,0.304771)(31.622777,0.433377)(56.234133,0.579565)(100.000000,0.706364)
      };
      \node[draw] at (axis cs:3,0.00001) { $N/n=1/4$ };
    \end{loglogaxis}
  \end{tikzpicture}
\end{minipage}
\hfill{}
\begin{minipage}[b]{0.24\columnwidth}%
  \begin{tikzpicture}[font=\scriptsize]
    \pgfplotsset{every major grid/.style={style=densely dashed}}
    \begin{loglogaxis}[
      width=1.15\linewidth,
      xmin=1e-4,
      xmax=1e2,
      ymin=1e-6,
      ymax=1,
      grid=major,
      scaled ticks=true,
      xlabel={ $\lambda$},
      ylabel=\empty,
      legend style = {at={(0.02,0.98)}, anchor=north west, font=\scriptsize}
      ]
      \addplot[BLUE,only marks,mark=o,line width=.5pt] coordinates{
          (0.000100,0.000329)(0.000178,0.000443)(0.000316,0.000611)(0.000562,0.000818)(0.001000,0.001130)(0.001778,0.001641)(0.003162,0.002103)(0.005623,0.003017)(0.010000,0.004223)(0.017783,0.005733)(0.031623,0.008210)(0.056234,0.010849)(0.100000,0.014646)(0.177828,0.020296)(0.316228,0.026155)(0.562341,0.034948)(1.000000,0.044086)(1.778279,0.057289)(3.162278,0.075305)(5.623413,0.099155)(10.000000,0.132538)(17.782794,0.184752)(31.622777,0.280802)(56.234133,0.403896)(100.000000,0.552536) 
      };
      \addplot[black,densely dashed,smooth,line width=1.5pt] coordinates{
          (0.000100,0.000000)(0.000178,0.000001)(0.000316,0.000002)(0.000562,0.000006)(0.001000,0.000019)(0.001778,0.000053)(0.003162,0.000117)(0.005623,0.000302)(0.010000,0.000701)(0.017783,0.001384)(0.031623,0.002934)(0.056234,0.005036)(0.100000,0.008292)(0.177828,0.014020)(0.316228,0.019879)(0.562341,0.028832)(1.000000,0.039238)(1.778279,0.052319)(3.162278,0.071141)(5.623413,0.093997)(10.000000,0.127910)(17.782794,0.184153)(31.622777,0.276933)(56.234133,0.403245)(100.000000,0.548810)
      };
      \addplot[RED,smooth,line width=1pt] coordinates{
          (0.000100,0.000311)(0.000178,0.000412)(0.000316,0.000575)(0.000562,0.000778)(0.001000,0.001094)(0.001778,0.001580)(0.003162,0.002038)(0.005623,0.002932)(0.010000,0.004204)(0.017783,0.005534)(0.031623,0.008062)(0.056234,0.010678)(0.100000,0.014287)(0.177828,0.020240)(0.316228,0.025673)(0.562341,0.034390)(1.000000,0.044376)(1.778279,0.057052)(3.162278,0.075610)(5.623413,0.098272)(10.000000,0.131987)(17.782794,0.188044)(31.622777,0.280295)(56.234133,0.405705)(100.000000,0.550283)
      };
      \node[draw] at (axis cs:3,0.00001) { $N/n=1/2$ };
    \end{loglogaxis}
  \end{tikzpicture}
\end{minipage}
\hfill{}
\begin{minipage}[b]{0.24\columnwidth}%
  \begin{tikzpicture}[font=\scriptsize]
    \pgfplotsset{every major grid/.style={style=densely dashed}}
    \begin{loglogaxis}[
      width=1.15\linewidth,
      xmin=1e-4,
      xmax=1e2,
      ymin=1e-6,
      ymax=1,
      grid=major,
      scaled ticks=true,
      xlabel={ $\lambda$},
      ylabel=\empty,
      legend style = {at={(0.02,0.98)}, anchor=north west, font=\scriptsize}
      ]
      \addplot[BLUE,only marks,mark=o,line width=.5pt] coordinates{
          (0.000100,0.000001)(0.000178,0.000002)(0.000316,0.000006)(0.000562,0.000014)(0.001000,0.000038)(0.001778,0.000097)(0.003162,0.000220)(0.005623,0.000459)(0.010000,0.000838)(0.017783,0.001529)(0.031623,0.002768)(0.056234,0.004509)(0.100000,0.006809)(0.177828,0.010368)(0.316228,0.015788)(0.562341,0.021664)(1.000000,0.029330)(1.778279,0.039828)(3.162278,0.052944)(5.623413,0.070049)(10.000000,0.089943)(17.782794,0.122746)(31.622777,0.174428)(56.234133,0.254073)(100.000000,0.370963)
      };
      \addplot[black,densely dashed,smooth,line width=1.5pt] coordinates{
          (0.000100,0.000000)(0.000178,0.000000)(0.000316,0.000001)(0.000562,0.000002)(0.001000,0.000005)(0.001778,0.000015)(0.003162,0.000041)(0.005623,0.000115)(0.010000,0.000262)(0.017783,0.000576)(0.031623,0.001288)(0.056234,0.002505)(0.100000,0.004383)(0.177828,0.007580)(0.316228,0.012667)(0.562341,0.018681)(1.000000,0.026442)(1.778279,0.037151)(3.162278,0.050265)(5.623413,0.067510)(10.000000,0.087702)(17.782794,0.120447)(31.622777,0.171216)(56.234133,0.253271)(100.000000,0.371138)
      };
      \addplot[RED,smooth,line width=1pt] coordinates{
          (0.000100,0.000001)(0.000178,0.000002)(0.000316,0.000005)(0.000562,0.000014)(0.001000,0.000036)(0.001778,0.000089)(0.003162,0.000200)(0.005623,0.000435)(0.010000,0.000811)(0.017783,0.001478)(0.031623,0.002685)(0.056234,0.004403)(0.100000,0.006704)(0.177828,0.010272)(0.316228,0.015603)(0.562341,0.021574)(1.000000,0.029177)(1.778279,0.039714)(3.162278,0.052664)(5.623413,0.069745)(10.000000,0.089850)(17.782794,0.122521)(31.622777,0.173196)(56.234133,0.255019)(100.000000,0.372472)
      };
      \node[draw] at (axis cs:3,0.00001) { $N/n=1$ };
    \end{loglogaxis}
  \end{tikzpicture}
\end{minipage}
\hfill{}
\begin{minipage}[b]{0.24\columnwidth}%
  \begin{tikzpicture}[font=\scriptsize]
    \pgfplotsset{every major grid/.style={style=densely dashed}}
    \begin{loglogaxis}[
      width=1.15\linewidth,
      xmin=1e-4,
      xmax=1e2,
      ymin=1e-6,
      ymax=1,
      grid=major,
      scaled ticks=true,
      xlabel={ $\lambda$},
      ylabel=\empty,
      legend style = {at={(0.02,0.98)}, anchor=north west, font=\scriptsize}
      ]
      \addplot[BLUE,only marks,mark=o,line width=.5pt] coordinates{
          (0.000100,0.000000)(0.000178,0.000000)(0.000316,0.000000)(0.000562,0.000001)(0.001000,0.000004)(0.001778,0.000011)(0.003162,0.000030)(0.005623,0.000074)(0.010000,0.000169)(0.017783,0.000397)(0.031623,0.000770)(0.056234,0.001697)(0.100000,0.002941)(0.177828,0.005157)(0.316228,0.008412)(0.562341,0.012754)(1.000000,0.018967)(1.778279,0.026155)(3.162278,0.036115)(5.623413,0.048119)(10.000000,0.064174)(17.782794,0.082394)(31.622777,0.114578)(56.234133,0.160737)(100.000000,0.232891)  
      };
      \addplot[black,densely dashed,smooth,line width=1.5pt] coordinates{
          (0.000100,0.000000)(0.000178,0.000000)(0.000316,0.000000)(0.000562,0.000000)(0.001000,0.000001)(0.001778,0.000004)(0.003162,0.000012)(0.005623,0.000033)(0.010000,0.000082)(0.017783,0.000213)(0.031623,0.000460)(0.056234,0.001142)(0.100000,0.002142)(0.177828,0.004086)(0.316228,0.007051)(0.562341,0.011227)(1.000000,0.017516)(1.778279,0.024737)(3.162278,0.034659)(5.623413,0.046711)(10.000000,0.063097)(17.782794,0.081341)(31.622777,0.112937)(56.234133,0.160678)(100.000000,0.232717)
      };
      \addplot[RED,smooth,line width=1pt] coordinates{
          (0.000100,0.000000)(0.000178,0.000000)(0.000316,0.000000)(0.000562,0.000001)(0.001000,0.000004)(0.001778,0.000010)(0.003162,0.000028)(0.005623,0.000071)(0.010000,0.000162)(0.017783,0.000381)(0.031623,0.000754)(0.056234,0.001669)(0.100000,0.002907)(0.177828,0.005129)(0.316228,0.008313)(0.562341,0.012596)(1.000000,0.018945)(1.778279,0.026107)(3.162278,0.035963)(5.623413,0.047913)(10.000000,0.064239)(17.782794,0.082418)(31.622777,0.113979)(56.234133,0.161679)(100.000000,0.233623)
      };
      \node[draw] at (axis cs:3,0.00001) { $N/n=2$ };
    \end{loglogaxis}
  \end{tikzpicture}
\end{minipage}
\end{center}
\caption{ Training MSEs of RFF ridge regression on MNIST data (class $3$ versus $7$), as a function of regression parameter $\lambda$, for $p=784$, $n=1\,000$, $N=250, 500, 1\,000, 2\,000$. Empirical results displayed in {\BLUE \bf blue} circles; Gaussian kernel predictions (assuming $N \to \infty$ alone) in {\bf black} dashed lines; and our predictions from Theorems~\ref{theo:asy-training-MSE}~and~\ref{theo:asy-test-MSE} in {\RED \bf red} solid lines. Results obtained by averaging over $30$ runs.}
\label{fig:compare-kernel-RMT}
\end{figure}
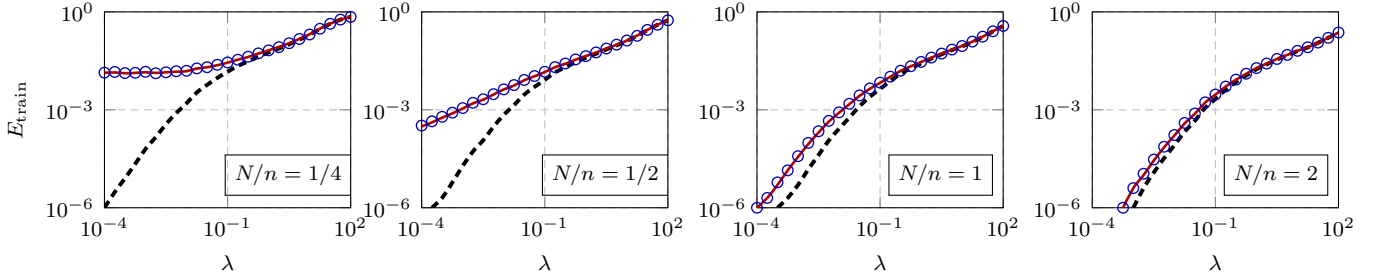

Note also that, while $\delta_{\cos}$ and $\delta_{\sin}$ can be geometrically interpreted as a sort of weighted ``angle'' between different kernel matrices, and therefore one might expect to have $\delta \in [0,1]$, this is not the case for the leftmost plot with $N/n = 1/4$. 
There, for small values of $\lambda$ (say $\lambda \lesssim 0.1$), both $\delta_{\cos}$ and $\delta_{\sin}$ scale like $\lambda^{-1}$, while they are observed to saturate to a fixed $O(1)$ value for $N/n=1,4,16$. 
This corresponds to two different phases of learning in the ``ridgeless'' $\lambda \to 0$ case.
As we shall see in more detail later in Section~\ref{subsec:two-regime}, depending on whether we are in the ``under-parameterized'' ($2N<n$) or the ``over-parameterized'' ($2N>n$) regime, the model behaves fundamentally differently.

\begin{figure}[t] 
\vskip 0.1in
\begin{center}
\begin{minipage}[b]{0.24\columnwidth}%
  \begin{tikzpicture}[font=\scriptsize]
    \pgfplotsset{every major grid/.style={style=densely dashed}}
    \begin{loglogaxis}[
      width=1.15\linewidth,
      xmin=1e-4,
      xmax=1e2,
      ymin=1e-3,
      ymax=1e3,
      grid=major,
      xlabel={ $\lambda$},
      ylabel={ $\delta$ },
      legend style = {at={(0.02,0.02)}, anchor=south west, font=\footnotesize}
      ]
      \addplot+[RED,only marks,mark=x,line width=0.5pt] coordinates{
          (0.000100,343.697222)(0.000178,193.940121)(0.000316,109.715969)(0.000562,62.337118)(0.001000,35.667722)(0.001778,20.629805)(0.003162,12.115505)(0.005623,7.253357)(0.010000,4.434916)(0.017783,2.765472)(0.031623,1.750869)(0.056234,1.118397)(0.100000,0.715918)(0.177828,0.456515)(0.316228,0.288737)(0.562341,0.180783)(1.000000,0.112135)(1.778279,0.069151)(3.162278,0.042667)(5.623413,0.026582)(10.000000,0.016905)(17.782794,0.011078)(31.622777,0.007493)(56.234133,0.005170)(100.000000,0.003555)
      };
      \addlegendentry{ {$ \delta_{\cos} $} }; %
      \addplot+[BLUE,only marks,mark=x,line width=0.5pt] coordinates{
          (0.000100,292.006201)(0.000178,164.900818)(0.000316,93.415480)(0.000562,53.201555)(0.001000,30.563475)(0.001778,17.795790)(0.003162,10.562498)(0.005623,6.425739)(0.010000,4.019884)(0.017783,2.585534)(0.031623,1.703720)(0.056234,1.143743)(0.100000,0.777370)(0.177828,0.531744)(0.316228,0.364101)(0.562341,0.248353)(1.000000,0.167970)(1.778279,0.112141)(3.162278,0.073600)(5.623413,0.047322)(10.000000,0.029736)(17.782794,0.018243)(31.622777,0.010938)(56.234133,0.006429)(100.000000,0.003721)
      };
      \addlegendentry{ {$ \delta_{\sin} $} }; %
      \node[draw] at (axis cs:1,200) { $N/n=1/4$ };
    \end{loglogaxis}
  \end{tikzpicture}
\end{minipage}
\hfill{}
\begin{minipage}[b]{0.24\columnwidth}%
  \begin{tikzpicture}[font=\scriptsize]
    \pgfplotsset{every major grid/.style={style=densely dashed}}
    \begin{loglogaxis}[
      width=1.15\linewidth,
      xmin=1e-4,
      xmax=1e2,
      ymin=1e-3,
      ymax=2,
      grid=major,
      scaled ticks=true,
      xlabel={ $\lambda$},
      ylabel=\empty,
      legend style = {at={(0.02,0.02)}, anchor=south west, font=\footnotesize}
      ]
      \addplot+[RED,only marks,mark=x,line width=0.5pt] coordinates{
          (0.000100,1.157958)(0.000178,1.152719)(0.000316,1.143631)(0.000562,1.128146)(0.001000,1.102514)(0.001778,1.061921)(0.003162,1.001529)(0.005623,0.918570)(0.010000,0.814527)(0.017783,0.695695)(0.031623,0.571447)(0.056234,0.451412)(0.100000,0.343196)(0.177828,0.251397)(0.316228,0.177684)(0.562341,0.121427)(1.000000,0.080485)(1.778279,0.051969)(3.162278,0.032871)(5.623413,0.020505)(10.000000,0.012716)(17.782794,0.007916)(31.622777,0.005003)(56.234133,0.003245)(100.000000,0.002173)
      };
      \addlegendentry{ {$ \delta_{\cos} $} }; %
      \addplot+[BLUE,only marks,mark=x,line width=0.5pt] coordinates{
          (0.000100,0.854035)(0.000178,0.850636)(0.000316,0.844738)(0.000562,0.834679)(0.001000,0.818003)(0.001778,0.791529)(0.003162,0.751985)(0.005623,0.697335)(0.010000,0.628193)(0.017783,0.548257)(0.031623,0.463306)(0.056234,0.379478)(0.100000,0.301823)(0.177828,0.233635)(0.316228,0.176432)(0.562341,0.130290)(1.000000,0.094288)(1.778279,0.066960)(3.162278,0.046675)(5.623413,0.031905)(10.000000,0.021349)(17.782794,0.013958)(31.622777,0.008903)(56.234133,0.005535)(100.000000,0.003357)
      };
      \addlegendentry{ {$ \delta_{\sin} $} }; %
      \node[draw] at (axis cs:1,1) { $N/n=1$ };
    \end{loglogaxis}
  \end{tikzpicture}
\end{minipage}
\hfill{}
\begin{minipage}[b]{0.24\columnwidth}%
  \begin{tikzpicture}[font=\scriptsize]
    \pgfplotsset{every major grid/.style={style=densely dashed}}
    \begin{loglogaxis}[
      width=1.15\linewidth,
      xmin=1e-4,
      xmax=1e2,
      ymin=1e-3,
      ymax=2,
      grid=major,
      scaled ticks=true,
      xlabel={ $\lambda$},
      ylabel=\empty,
      legend style = {at={(0.02,0.02)}, anchor=south west, font=\footnotesize}
      ]
      \addplot+[RED,only marks,mark=x,line width=0.5pt] coordinates{
          (0.000100,0.162912)(0.000178,0.162851)(0.000316,0.162743)(0.000562,0.162551)(0.001000,0.162213)(0.001778,0.161617)(0.003162,0.160580)(0.005623,0.158796)(0.010000,0.155802)(0.017783,0.150953)(0.031623,0.143502)(0.056234,0.132821)(0.100000,0.118738)(0.177828,0.101811)(0.316228,0.083326)(0.562341,0.064952)(1.000000,0.048243)(1.778279,0.034242)(3.162278,0.023339)(5.623413,0.015373)(10.000000,0.009857)(17.782794,0.006201)(31.622777,0.003859)(56.234133,0.002399)(100.000000,0.001506)
      };
      \addlegendentry{ {$ \delta_{\cos} $} }; %
      \addplot+[BLUE,only marks,mark=x,line width=0.5pt] coordinates{
          (0.000100,0.123360)(0.000178,0.123321)(0.000316,0.123253)(0.000562,0.123132)(0.001000,0.122918)(0.001778,0.122541)(0.003162,0.121884)(0.005623,0.120754)(0.010000,0.118852)(0.017783,0.115762)(0.031623,0.110986)(0.056234,0.104080)(0.100000,0.094858)(0.177828,0.083576)(0.316228,0.070957)(0.562341,0.058012)(1.000000,0.045749)(1.778279,0.034920)(3.162278,0.025908)(5.623413,0.018755)(10.000000,0.013280)(17.782794,0.009205)(31.622777,0.006241)(56.234133,0.004133)(100.000000,0.002669)
      };
      \addlegendentry{ {$ \delta_{\sin} $} }; %
      \node[draw] at (axis cs:1,1) { $N/n=4$ };
    \end{loglogaxis}
  \end{tikzpicture}
\end{minipage}
\hfill{}
\begin{minipage}[b]{0.24\columnwidth}%
  \begin{tikzpicture}[font=\scriptsize]
    \pgfplotsset{every major grid/.style={style=densely dashed}}
    \begin{loglogaxis}[
      width=1.15\linewidth,
      xmin=1e-4,
      xmax=1e2,
      ymin=1e-3,
      ymax=2,
      grid=major,
      scaled ticks=true,
      xlabel={ $\lambda$},
      ylabel=\empty,
      legend style = {at={(0.02,0.02)}, anchor=south west, font=\footnotesize}
      ]
      \addplot+[RED,only marks,mark=x,line width=0.5pt] coordinates{
          (0.000100,0.036516)(0.000178,0.036513)(0.000316,0.036508)(0.000562,0.036499)(0.001000,0.036483)(0.001778,0.036455)(0.003162,0.036406)(0.005623,0.036319)(0.010000,0.036167)(0.017783,0.035902)(0.031623,0.035449)(0.056234,0.034695)(0.100000,0.033486)(0.177828,0.031655)(0.316228,0.029070)(0.562341,0.025720)(1.000000,0.021765)(1.778279,0.017532)(3.162278,0.013422)(5.623413,0.009781)(10.000000,0.006814)(17.782794,0.004568)(31.622777,0.002968)(56.234133,0.001884)(100.000000,0.001178)
      };
      \addlegendentry{ {$ \delta_{\cos} $} }; %
      \addplot+[BLUE,only marks,mark=x,line width=0.5pt] coordinates{
          (0.000100,0.028030)(0.000178,0.028028)(0.000316,0.028025)(0.000562,0.028019)(0.001000,0.028009)(0.001778,0.027992)(0.003162,0.027960)(0.005623,0.027905)(0.010000,0.027809)(0.017783,0.027641)(0.031623,0.027353)(0.056234,0.026871)(0.100000,0.026097)(0.177828,0.024915)(0.316228,0.023228)(0.562341,0.021007)(1.000000,0.018330)(1.778279,0.015384)(3.162278,0.012417)(5.623413,0.009662)(10.000000,0.007278)(17.782794,0.005332)(31.622777,0.003813)(56.234133,0.002667)(100.000000,0.001825)
      };
      \addlegendentry{ {$ \delta_{\sin} $} }; %
      \node[draw] at (axis cs:1,1) { $N/n=16$ };
    \end{loglogaxis}
  \end{tikzpicture}
\end{minipage}
\end{center}
\caption{ Behavior of $(\delta_{\cos}, \delta_{\sin})$ in \eqref{eq:fixed-point-delta} on MNIST data set (class $3$ versus $7$), as a function of the regularization parameter $\lambda$, for $p=784$, $n=1\,000$, $N = 250, 1\,000, 4\,000, 16\,000$. }
\label{fig:trend-delta-lambda}
\end{figure}
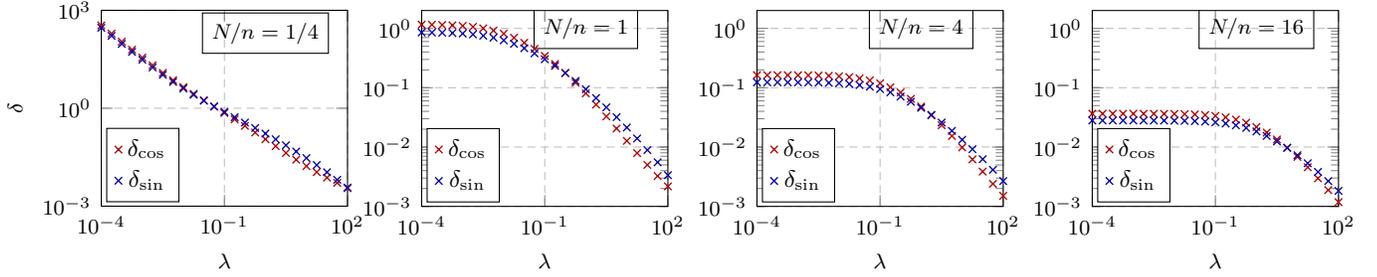

\subsection{Phase Transition and Corresponding Double Descent}
\label{subsec:discuss-delta-versus-N/n}

Both $\delta_{\cos}$ and $\delta_{\sin}$ in \eqref{eq:fixed-point-delta} are decreasing functions of $N$, as depicted in Figure~\ref{fig:trend-delta-N/n}. 
(See Lemma~\ref{lem:delta-derivative-N} in Appendix~\ref{sec:detail-section-double-descent} for a proof.)
More importantly, Figure~\ref{fig:trend-delta-N/n} also illustrates that $\delta_{\cos}$ and $\delta_{\sin}$ exhibit qualitatively different behavior, depending on the ratio $N/n$.
For $\lambda$ not too small ($\lambda = 1$ or $10$), we observe a rather ``smooth'' behavior, as a function of the ratio $N/n$, and they both decrease smoothly, as $N/n$ grows large. 
However, for $\lambda$ relatively small ($\lambda=10^{-3}$ and $10^{-7}$), we observe a \emph{sharp} ``phase transition'' on two sides of the interpolation threshold $2 N =n$. 
(Note that the scale of the y-axis is very different in different subfigures.)
More precisely, in the leftmost plot with $\lambda = 10^{-7}$, the values of $\delta_{\cos}$ and $\delta_{\sin}$ ``jumps'' from order $O(1)$ (when $2N>n$) to much higher values of the order of $\lambda^{-1}$ (when $2N<n$). A similar behavior is also observed for $\lambda = 10^{-3}$.

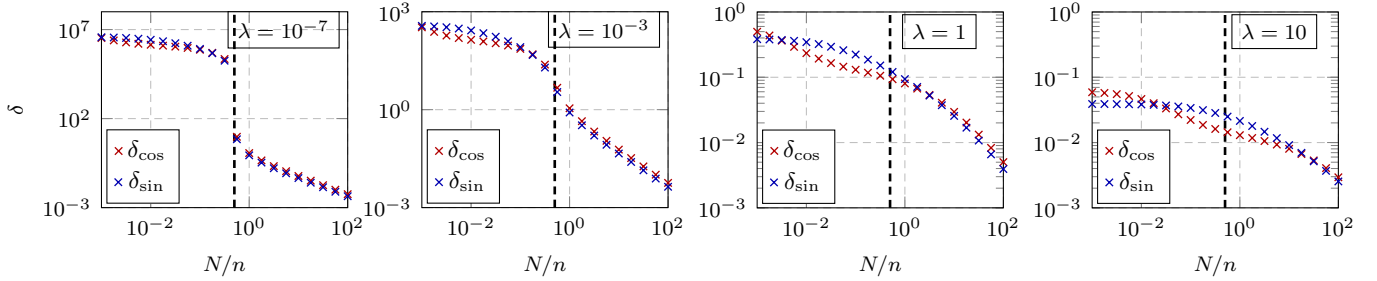
\begin{figure}[t] 
\vskip 0.1in
\begin{center}
\begin{minipage}[b]{0.24\columnwidth}%
  \begin{tikzpicture}[font=\scriptsize]
    \pgfplotsset{every major grid/.style={style=densely dashed}}
    \begin{loglogaxis}[
      width=1.15\linewidth,
      xmin=1e-3,
      xmax=1e2,
      ymin=1e-3,
      ymax=1e8,
      grid=major,
      xlabel={ $N/n$},
      ylabel={ $\delta$ },
      legend style = {at={(0.02,0.02)}, anchor=south west, font=\footnotesize}
      ]
      \addplot+[RED,only marks,mark=x,line width=.5pt] coordinates{
          (0.001000,3354028.559413)(0.001778,2453601.178573)(0.003162,1896644.556645)(0.005623,1593529.574719)(0.010000,1402793.261451)(0.017783,1249867.601372)(0.031623,1098822.350742)(0.056234,930642.995753)(0.100000,732919.785102)(0.177828,496546.639860)(0.316228,222076.130795)(0.562341,9.535147)(1.000000,1.160396)(1.778279,0.448714)(3.162278,0.214196)(5.623413,0.110959)(10.000000,0.059740)(17.782794,0.032807)(31.622777,0.018209)(56.234133,0.010165)(100.000000,0.005693)
      };
      \addlegendentry{ {$ \delta_{\cos} $} }; %
      \addplot+[BLUE,only marks,mark=x,line width=.5pt] coordinates{
          (0.001000,3716724.609243)(0.001778,3583613.894463)(0.003162,3373241.740777)(0.005623,3069157.608946)(0.010000,2672021.017968)(0.017783,2205932.064173)(0.031623,1713229.512816)(0.056234,1240166.262690)(0.100000,821018.005189)(0.177828,467894.587296)(0.316228,178327.521016)(0.562341,6.884026)(1.000000,0.861765)(1.778279,0.337986)(3.162278,0.162544)(5.623413,0.084541)(10.000000,0.045618)(17.782794,0.025083)(31.622777,0.013931)(56.234133,0.007780)(100.000000,0.004358)
      };
      \addlegendentry{ {$ \delta_{\sin} $} }; %
      \addplot[densely dashed,black,line width=1pt] coordinates{(0.5,0.001)(0.5,100000000)};
      \node[draw] at (axis cs:5,1e7) { $\lambda = 10^{-7}$ };
    \end{loglogaxis}
  \end{tikzpicture}
\end{minipage}
\hfill{}
\begin{minipage}[b]{0.24\columnwidth}%
  \begin{tikzpicture}[font=\scriptsize]
    \pgfplotsset{every major grid/.style={style=densely dashed}}
    \begin{loglogaxis}[
      width=1.15\linewidth,
      xmin=1e-3,
      xmax=1e2,
      ymin=1e-3,
      ymax=1e3,
      grid=major,
      scaled ticks=true,
      xlabel={ $N/n$},
      ylabel=\empty,
      legend style = {at={(0.02,0.02)}, anchor=south west, font=\footnotesize}
      ]
      \addplot+[RED,only marks,mark=x,line width=.5pt] coordinates{
          (0.001000,332.382501)(0.001778,241.374833)(0.003162,185.999500)(0.005623,156.279326)(0.010000,137.782179)(0.017783,123.035487)(0.031623,108.442159)(0.056234,92.103284)(0.100000,72.823428)(0.177828,49.867971)(0.316228,24.032581)(0.562341,4.690642)(1.000000,1.103588)(1.778279,0.443367)(3.162278,0.213539)(5.623413,0.110976)(10.000000,0.059835)(17.782794,0.032883)(31.622777,0.018258)(56.234133,0.010194)(100.000000,0.005710)};
          \addlegendentry{ {$ \delta_{\cos} $} }; %
      \addplot+[BLUE,only marks,mark=x,line width=.5pt] coordinates{
          (0.001000,368.656974)(0.001778,355.671461)(0.003162,335.221129)(0.005623,305.591156)(0.010000,266.663062)(0.017783,220.636705)(0.031623,171.641623)(0.056234,124.355627)(0.100000,82.391632)(0.177828,47.296396)(0.316228,19.462217)(0.562341,3.462995)(1.000000,0.817874)(1.778279,0.332095)(3.162278,0.160988)(5.623413,0.083973)(10.000000,0.045369)(17.782794,0.024962)(31.622777,0.013869)(56.234133,0.007747)(100.000000,0.004340)
      };
      \addlegendentry{ {$ \delta_{\sin} $} }; %
      \addplot[densely dashed,black,line width=1pt] coordinates{(0.5,0.001)(0.5,100000000)};
      \node[draw] at (axis cs:5,300) { $\lambda = 10^{-3}$ };
    \end{loglogaxis}
  \end{tikzpicture}
\end{minipage}
\hfill{}
\begin{minipage}[b]{0.24\columnwidth}%
  \begin{tikzpicture}[font=\scriptsize]
    \pgfplotsset{every major grid/.style={style=densely dashed}}
    \begin{loglogaxis}[
      width=1.15\linewidth,
      xmin=1e-3,
      xmax=1e2,
      ymin=1e-3,
      ymax=1,
      grid=major,
      scaled ticks=true,
      xlabel={ $N/n$},
      ylabel=\empty,
      legend style = {at={(0.02,0.02)}, anchor=south west, font=\footnotesize}
      ]
      \addplot+[RED,only marks,mark=x,line width=.5pt] coordinates{
          (0.001000,0.497800)(0.001778,0.436647)(0.003162,0.363688)(0.005623,0.292138)(0.010000,0.234057)(0.017783,0.192693)(0.031623,0.164703)(0.056234,0.145243)(0.100000,0.130430)(0.177828,0.117729)(0.316228,0.105633)(0.562341,0.093311)(1.000000,0.080411)(1.778279,0.067005)(3.162278,0.053554)(5.623413,0.040791)(10.000000,0.029498)(17.782794,0.020247)(31.622777,0.013234)(56.234133,0.008291)(100.000000,0.005019)
      };
      \addlegendentry{ {$ \delta_{\cos} $} }; %
      \addplot+[BLUE,only marks,mark=x,line width=.5pt] coordinates{
          (0.001000,0.383235)(0.001778,0.379128)(0.003162,0.372321)(0.005623,0.361486)(0.010000,0.345203)(0.017783,0.322512)(0.031623,0.293525)(0.056234,0.259602)(0.100000,0.223006)(0.177828,0.186297)(0.316228,0.151686)(0.562341,0.120622)(1.000000,0.093736)(1.778279,0.071063)(3.162278,0.052356)(5.623413,0.037303)(10.000000,0.025597)(17.782794,0.016886)(31.622777,0.010724)(56.234133,0.006584)(100.000000,0.003932)
      };
      \addlegendentry{ {$ \delta_{\sin} $} }; %
      \addplot[densely dashed,black,line width=1pt] coordinates{(0.5,0.001)(0.5,100000000)};
      \node[draw] at (axis cs:5,0.5) { $\lambda = 1$ };
    \end{loglogaxis}
  \end{tikzpicture}
\end{minipage}
\hfill{}
\begin{minipage}[b]{0.24\columnwidth}%
  \begin{tikzpicture}[font=\scriptsize]
    \pgfplotsset{every major grid/.style={style=densely dashed}}
    \begin{loglogaxis}[
      width=1.15\linewidth,
      xmin=1e-3,
      xmax=1e2,
      ymin=1e-3,
      ymax=1,
      grid=major,
      scaled ticks=true,
      xlabel={ $N/n$},
      ylabel=\empty,
      legend style = {at={(0.02,0.02)}, anchor=south west, font=\footnotesize}
      ]
      \addplot+[RED,only marks,mark=x,line width=.5pt] coordinates{
          (0.001000,0.058980)(0.001778,0.057541)(0.003162,0.055192)(0.005623,0.051567)(0.010000,0.046452)(0.017783,0.040078)(0.031623,0.033263)(0.056234,0.027060)(0.100000,0.022148)(0.177828,0.018608)(0.316228,0.016143)(0.562341,0.014369)(1.000000,0.012972)(1.778279,0.011741)(3.162278,0.010544)(5.623413,0.009310)(10.000000,0.008014)(17.782794,0.006670)(31.622777,0.005328)(56.234133,0.004059)(100.000000,0.002938)
      };
      \addlegendentry{ {$ \delta_{\cos} $} }; %
      \addplot+[BLUE,only marks,mark=x,line width=.5pt] coordinates{
          (0.001000,0.038958)(0.001778,0.038900)(0.003162,0.038799)(0.005623,0.038622)(0.010000,0.038319)(0.017783,0.037811)(0.031623,0.036987)(0.056234,0.035713)(0.100000,0.033869)(0.177828,0.031405)(0.316228,0.028389)(0.562341,0.024990)(1.000000,0.021428)(1.778279,0.017921)(3.162278,0.014642)(5.623413,0.011699)(10.000000,0.009138)(17.782794,0.006962)(31.622777,0.005153)(56.234133,0.003686)(100.000000,0.002538)
      };
      \addlegendentry{ {$ \delta_{\sin} $} }; %
      \addplot[densely dashed,black,line width=1pt] coordinates{(0.5,0.001)(0.5,100000000)};
      \node[draw] at (axis cs:5,0.5) { $\lambda = 10$ };
    \end{loglogaxis}
  \end{tikzpicture}
\end{minipage}
\end{center}
\caption{ Behavior of $(\delta_{\cos}, \delta_{\sin})$ in \eqref{eq:fixed-point-delta} on MNIST data set (class $3$ versus $7$), as a function of the ratio $N/n$, for $p=784$, $n=1\,000$, $\lambda = 10^{-7}, 10^{-3}, 1, 10$. The {\bf black} dashed line represents the interpolation threshold $2 N =n$. }
\label{fig:trend-delta-N/n}
\end{figure}

As a consequence of this phase transition, different behaviors are expected for training and test MSEs in the $2N < n$ and $2N > n$ regime. 
Figure~\ref{fig:double-descent-regularization-2} depicts the empirical and theoretical test MSEs with different regularization penalty $\lambda$. 
In particular, for $\lambda = 10^{-7}$ and $\lambda = 10^{-3}$, a double descent-type behavior is observed, with a singularity at $2 N =n$, while for larger values of $\lambda$ ($\lambda = 0.2, 10$), a smoother and monotonically decreasing curve for test error is observed, as a function of $N/n$. 
Figure~\ref{fig:double-descent-regularization-2} also illustrates that: (i) for a fixed regularization $\lambda > 0$, the minimum test error is always obtained in the over-parametrization $2N>n$ regime; and (ii) the global optimal design (over $N$ and $\lambda$) is achieved by highly over-parametrized system with a (problem-dependent) non-vanishing $\lambda$.
This is in accordance with the observations in \cite{mei2019generalization} for Gaussian data.

\begin{figure}[t] 
\begin{center}
\begin{minipage}[b]{0.24\columnwidth}%
  \begin{tikzpicture}[font=\footnotesize]
    \pgfplotsset{every major grid/.style={style=densely dashed}}
    \begin{axis}[
      width=1.15\linewidth,
      xmin=0,
      xmax=5.00,
      ymin=0,
      ymax=1,
      symbolic x coords={0,0.05,0.10,0.15,0.20,0.25,0.30,0.35,0.40,0.45,0.5,0.55,0.60,0.65,0.70,0.75,0.80,0.85,0.90,0.95,1,1.50,2.00,2.50,3.00,3.5,4.00,4.50,5.00},
      ytick={0,0.5,1},
      grid=major,
      ymajorgrids=false,
      scaled ticks=true,
      xlabel={ $N/n$ },
      ylabel={ Test MSE },
      legend style = {at={(0.98,0.98)}, anchor=north east, font=\footnotesize}
      ]
      \addplot[RED,densely dashed,line width=1pt] coordinates{
      (0,1.000000)(0.05,0.320586)(0.10,0.248547)(0.15,0.234888)(0.20,0.235489)(0.25,0.257586)(0.30,0.290233)(0.35,0.357254)(0.40,0.478638)(0.45,0.861248)(0.5,3.572740)(0.55,0.856188)(0.60,0.451331)(0.65,0.329333)(0.70,0.265796)(0.75,0.221444)(0.80,0.204296)(0.85,0.179956)(0.90,0.157613)(0.95,0.156390)(1,0.149693)(1.50,0.107046)(2.00,0.096170)(2.50,0.088525)(3.00,0.084455)(3.5,0.085498)(4.00,0.080577)(4.50,0.079445)(5.00,0.076547)
      };
      \addplot+[only marks,mark=x,BLUE,line width=1pt] coordinates{
      (0,1.000000)(0.05,0.324289)(0.10,0.259979)(0.15,0.225615)(0.20,0.228612)(0.25,0.253655)(0.30,0.281658)(0.35,0.343953)(0.40,0.485780)(0.45,0.847222)(0.5,59.935998)(0.55,0.890565)(0.60,0.454359)(0.65,0.321645)(0.70,0.273403)(0.75,0.218586)(0.80,0.205166)(0.85,0.181495)(0.90,0.171261)(0.95,0.162471)(1,0.147626)(1.50,0.111626)(2.00,0.095726)(2.50,0.090401)(3.00,0.086514)(3.5,0.082154)(4.00,0.084921)(4.50,0.078919)(5.00,0.075109)
      };
      \addplot[densely dashed,black,line width=1pt] coordinates{(0.5,0)(0.5,1)};
      \node[draw] at (axis cs:1,.8) { $\lambda = 10^{-7}$ };
    \end{axis}
    \end{tikzpicture}
\end{minipage}
\hfill{}
\begin{minipage}[b]{0.24\columnwidth}%
  \begin{tikzpicture}[font=\footnotesize]
    \pgfplotsset{every major grid/.style={style=densely dashed}}
    \begin{axis}[
      width=1.15\linewidth,
      xmin=0,
      xmax=5.00,
      ymin=0,
      ymax=1,
      symbolic x coords={0,0.05,0.10,0.15,0.20,0.25,0.30,0.35,0.40,0.45,0.5,0.55,0.60,0.65,0.70,0.75,0.80,0.85,0.90,0.95,1,1.50,2.00,2.50,3.00,3.5,4.00,4.50,5.00},
      ytick={0,0.5,1},
      grid=major,
      ymajorgrids=false,
      scaled ticks=true,
      xlabel={ $N/n$ },
      ylabel= \empty,
      legend style = {at={(0.98,0.98)}, anchor=north east, font=\footnotesize}
      ]
      \addplot[RED,densely dashed,line width=1pt] coordinates{
      (0,1.000000)(0.05,0.323522)(0.10,0.246915)(0.15,0.232208)(0.20,0.232504)(0.25,0.246767)(0.30,0.275426)(0.35,0.323682)(0.40,0.405126)(0.45,0.487596)(0.5,0.539986)(0.55,0.461931)(0.60,0.365846)(0.65,0.295414)(0.70,0.248013)(0.75,0.212174)(0.80,0.190551)(0.85,0.177480)(0.90,0.164274)(0.95,0.154285)(1,0.145114)(1.50,0.107579)(2.00,0.097592)(2.50,0.091346)(3.00,0.083665)(3.5,0.080192)(4.00,0.078272)(4.50,0.078403)(5.00,0.079177)
      };
      \addplot+[only marks,mark=x,BLUE,line width=1pt] coordinates{
      (0,1.000000)(0.05,0.322168)(0.10,0.239250)(0.15,0.233509)(0.20,0.239558)(0.25,0.244649)(0.30,0.278792)(0.35,0.326610)(0.40,0.398159)(0.45,0.493695)(0.5,0.526170)(0.55,0.444608)(0.60,0.365679)(0.65,0.291070)(0.70,0.244749)(0.75,0.213911)(0.80,0.186056)(0.85,0.179926)(0.90,0.164709)(0.95,0.154558)(1,0.148462)(1.50,0.108911)(2.00,0.096805)(2.50,0.090360)(3.00,0.088080)(3.5,0.080751)(4.00,0.079636)(4.50,0.082188)(5.00,0.074142)
      };
      \addplot[densely dashed,black,line width=1pt] coordinates{(0.5,0)(0.5,1)};
      \node[draw] at (axis cs:1,.8) { $\lambda = 10^{-3}$ };
    \end{axis}
    \end{tikzpicture}
\end{minipage}
\hfill{}
\begin{minipage}[b]{0.24\columnwidth}%
  \begin{tikzpicture}[font=\footnotesize]
    \pgfplotsset{every major grid/.style={style=densely dashed}}
    \begin{axis}[
      width=1.15\linewidth,
      xmin=0,
      xmax=5.00,
      ymin=0,
      ymax=1,
      symbolic x coords={0,0.05,0.10,0.15,0.20,0.25,0.30,0.35,0.40,0.45,0.5,0.55,0.60,0.65,0.70,0.75,0.80,0.85,0.90,0.95,1,1.50,2.00,2.50,3.00,3.5,4.00,4.50,5.00},
      ytick={0,0.5,1},
      grid=major,
      ymajorgrids=false,
      scaled ticks=true,
      xlabel={ $N/n$ },
      ylabel= \empty,
      legend style = {at={(0.98,0.98)}, anchor=north east, font=\footnotesize}
      ]
      \addplot[RED,densely dashed,line width=1pt] coordinates{
      (0,1)(0.05,0.321489)(0.10,0.223752)(0.15,0.191018)(0.20,0.168997)(0.25,0.160045)(0.30,0.148337)(0.35,0.139562)(0.40,0.134957)(0.45,0.128903)(0.5,0.125498)(0.55,0.116367)(0.60,0.118209)(0.65,0.116252)(0.70,0.111674)(0.75,0.114472)(0.80,0.110044)(0.85,0.106574)(0.90,0.106508)(0.95,0.105378)(1,0.104841)(1.50,0.094813)(2.00,0.092217)(2.50,0.087950)(3.00,0.087601)(3.5,0.079538)(4.00,0.080164)(4.50,0.079918)(5.00,0.077566)
      };
      \addplot+[only marks,mark=x,BLUE,line width=1pt] coordinates{
      (0,1.000000)(0.05,0.300005)(0.10,0.225684)(0.15,0.194063)(0.20,0.166734)(0.25,0.153290)(0.30,0.144291)(0.35,0.141326)(0.40,0.132308)(0.45,0.128224)(0.5,0.125239)(0.55,0.119556)(0.60,0.121836)(0.65,0.115089)(0.70,0.112291)(0.75,0.107167)(0.80,0.111574)(0.85,0.108285)(0.90,0.103312)(0.95,0.106415)(1,0.103499)(1.50,0.095257)(2.00,0.088071)(2.50,0.085968)(3.00,0.081872)(3.5,0.085500)(4.00,0.078181)(4.50,0.078020)(5.00,0.079003)
      };
      \addplot[densely dashed,black,line width=1pt] coordinates{(0.5,0)(0.5,1)};
      \node[draw] at (axis cs:1,.8) { $\lambda_{opt} = 0.2$ };
    \end{axis}
    \end{tikzpicture}
\end{minipage}
\hfill{}
\begin{minipage}[b]{0.24\columnwidth}%
  \begin{tikzpicture}[font=\footnotesize]
    \pgfplotsset{every major grid/.style={style=densely dashed}}
    \begin{axis}[
      width=1.15\linewidth,
      xmin=0,
      xmax=5.00,
      ymin=0,
      ymax=1,
      symbolic x coords={0,0.05,0.10,0.15,0.20,0.25,0.30,0.35,0.40,0.45,0.5,0.55,0.60,0.65,0.70,0.75,0.80,0.85,0.90,0.95,1,1.50,2.00,2.50,3.00,3.5,4.00,4.50,5.00},
      ytick={0,0.5,1},
      grid=major,
      ymajorgrids=false,
      scaled ticks=true,
      xlabel={ $N/n$ },
      ylabel= \empty,
      legend style = {at={(0.98,0.98)}, anchor=north east, font=\footnotesize}
      ]
      \addplot[RED,densely dashed,line width=1pt] coordinates{
      (0,1.000000)(0.05,0.798016)(0.10,0.668310)(0.15,0.561214)(0.20,0.499226)(0.25,0.450972)(0.30,0.405281)(0.35,0.370093)(0.40,0.351104)(0.45,0.324571)(0.5,0.313120)(0.55,0.293542)(0.60,0.279269)(0.65,0.272947)(0.70,0.258665)(0.75,0.247095)(0.80,0.241979)(0.85,0.238583)(0.90,0.229134)(0.95,0.223694)(1,0.210455)(1.50,0.183700)(2.00,0.159485)(2.50,0.148746)(3.00,0.142476)(3.5,0.134409)(4.00,0.125272)(4.50,0.121265)(5.00,0.117717)
      };
      \addplot+[only marks,mark=x,BLUE,line width=1pt] coordinates{
      (0,1.000000)(0.05,0.800628)(0.10,0.669346)(0.15,0.572246)(0.20,0.505067)(0.25,0.453947)(0.30,0.411620)(0.35,0.378967)(0.40,0.354003)(0.45,0.333199)(0.5,0.310479)(0.55,0.298592)(0.60,0.280434)(0.65,0.270641)(0.70,0.259722)(0.75,0.253286)(0.80,0.239215)(0.85,0.229615)(0.90,0.228431)(0.95,0.223308)(1,0.215833)(1.50,0.179224)(2.00,0.161857)(2.50,0.150000)(3.00,0.140210)(3.5,0.134244)(4.00,0.124540)(4.50,0.122380)(5.00,0.121325)
      };
      \addplot[densely dashed,black,line width=1pt] coordinates{(0.5,0)(0.5,1)};
      \node[draw] at (axis cs:1,.8) { $\lambda = 10$ };
    \end{axis}
    \end{tikzpicture}
\end{minipage}
\end{center}
\caption{ Empirical (\textbf{\BLUE blue} crosses) and theoretical (\textbf{\RED red} dashed lines) test error of RFF regression as a function of the ratio $N/n$ on MNIST data (class $3$ versus $7$), for $p=784$, $n=500$, $\lambda = 10^{-7}, 10^{-3}, 0.2, 10$. The {\bf black} dashed line represents the interpolation threshold $2 N =n$. 
}
\label{fig:double-descent-regularization-2}
\end{figure}
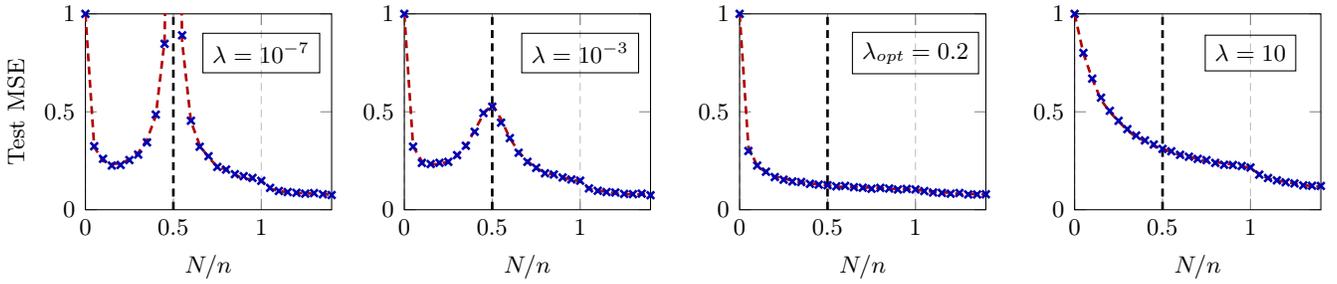

\begin{Remark}\label{rem:ridge-reg}
\normalfont
Performing ridge regularization (with $\lambda$ as a control parameter and choosing $\lambda > 0$) is known to help alleviate the sharp performance drop around $2N = n$~\cite{hastie2019surprises,mei2019generalization}. 
Our Theorem~\ref{theo:asy-test-MSE} can serve as a convenient alternative to evaluate the effect of small $\lambda$ around $2N= n$, as well as to determine an optimal $\lambda$, for not-too-small $n,p,N$.
In the setup of Figure~\ref{fig:double-descent-regularization-2}, a grid search can be used to find the regularization that minimizes $\bar E_{\test}$.
For this choice of $\lambda$ ($\lambda_{opt} \approx 0.2$), no singular peak at $2N = n$ is observed.
\end{Remark}

\begin{Remark}
\normalfont
While the double descent phenomenon has received considerable attention recently, our analysis makes it clear that in this model (and presumably many others) it is a natural consequence of the phase transition between two qualitatively different phases of learning~\cite{MM17_TR}.
\end{Remark}

\section{Additional Discussion and Results}
\label{sec:empirical_additional}

In this section, we provide additional discussions and empirical results, to complement and extend those of Section~\ref{sec:empirical_main}.
We start, in Section~\ref{subsec:two-regime}, by discussing in more detail the two different phases of learning for $2 N < n$ and $2 N > n$, including the \emph{sharp} phase transition at $2N = n$, for $(\delta_{\cos}, \delta_{\sin})$, as well as the (asymptotic) test MSE, in the ridgeless $\lambda \to 0$ case.
Then, in Section~\ref{subsec:impact-train-test-similarity}, we discuss the impact of training-test similarly on the test MSE by considering the example of test data $\hat \X$ as an additive perturbation of the training data $\X$. 
Finally, in Section~\ref{sec:num}, we present empirical results on additional real-world data sets to demonstrate the wide applicability of our results.

\subsection{Two Different Learning Regimes in the Ridgeless Limit}
\label{subsec:two-regime}

We chose to present our theoretical results in Section~\ref{sec:main} (Theorems~\ref{theo:asy-behavior-E[Q]}-\ref{theo:asy-test-MSE}) in the same form, regardless of whether $2N > n$ or $2N < n$. 
This comes at the cost of requiring a strictly positive ridge regularization $\lambda > 0$, as $n,p,N \to \infty$.
As discussed in Section~\ref{sec:empirical_main}, for small values of $\lambda$, depending on the sign of $2 N - n$, we observe totally different behaviors for $(\delta_{\cos}, \delta_{\sin})$ and thus for the key resolvent $\bar \Q$. As a matter of fact, for $\lambda= 0$ and $2N < n$, the (random) resolvent $\Q(\lambda=0)$ in \eqref{eq:def-Q} is simply undefined, as it involves inverting a singular matrix $\bSigma_\X^\T \bSigma_\X \in \RR^{n \times n}$ that is of rank at most $2N < n$. As a consequence, we expect to see $\bar \Q \sim \lambda^{-1}$ as $\lambda \to 0$ for $2N < n$, while for $2N > n$ this is no longer the case.

These two phases of learning can be theoretically justified by considering the \emph{ridgeless} $\lambda \to 0$ limit in Theorem~\ref{theo:asy-behavior-E[Q]}, with the unified variables $\gamma_{\cos}$ and $\gamma_{\sin}$.
\begin{enumerate}
\item 
For $2N < n$ and $\lambda \to 0$, we obtain
\begin{equation}\label{eq:def-theta}
  \begin{cases}
    \lambda \delta_{\cos} \to \gamma_{\cos} \equiv \frac1n \tr \K_{\cos} \left( \frac{N}n \left(\frac{\K_{\cos} }{\gamma_{\cos}} + \frac{\K_{\sin} }{\gamma_{\sin}} \right) + \I_n\right)^{-1} \\ 
    \lambda \delta_{\sin} \to \gamma_{\sin} \equiv \frac1n \tr \K_{\sin} \left( \frac{N}n \left(\frac{\K_{\cos} }{\gamma_{\cos}} + \frac{\K_{\sin} }{\gamma_{\sin}} \right) + \I_n\right)^{-1}
  \end{cases}  ,
\end{equation}
in such as way that $\delta_{\cos}$, $\delta_{\sin}$ and $ \bar \Q$ scale like $\lambda^{-1}$. We have in particular $\EE[\lambda \Q] \sim \lambda \bar \Q \sim \left( \frac{N}n \left(\frac{\K_{\cos} }{\gamma_{\cos}} + \frac{\K_{\sin} }{\gamma_{\sin}} \right) + \I_n\right)^{-1}$ with $(\gamma_{\cos},\gamma_{\sin})$ of order $O(1)$. 
\item
For $2N > n$ and $\lambda \to 0$, we obtain
\begin{equation}\label{eq:classi-fixed-point}
  \begin{cases}
    \delta_{\cos} \to \gamma_{\cos} = \frac1N \tr \K_{\cos} \left(\frac{\K_{\cos} }{1+\gamma_{\cos}} + \frac{\K_{\sin} }{1+\gamma_{\sin}} \right)^{-1} \\ 
    \delta_{\sin} \to \gamma_{\cos} = \frac1N \tr \K_{\sin} \left(\frac{\K_{\cos} }{1+\gamma_{\cos}} + \frac{\K_{\sin} }{1+\gamma_{\sin}} \right)^{-1}
  \end{cases}  ,
\end{equation}
by taking directly $\lambda \to 0$ in Theorem~\ref{theo:asy-behavior-E[Q]}.
\end{enumerate}
Note that the expressions in (\ref{eq:def-theta}) and (\ref{eq:classi-fixed-point}) only hold in the $\lambda \to 0$ limit. For $\lambda > 0$ the expression in \eqref{eq:fixed-point-delta} should be used instead.

\medskip

As a consequence, \emph{in the ridgeless limit $\lambda \to 0$}, Theorem~\ref{theo:asy-behavior-E[Q]} exhibits the following \emph{two learning phases}:
\begin{enumerate}
  \item 
  \emph{Under-parameterized phase}: with $2N < n$.
  Here, $\Q$ is not well defined (indeed $ \Q \sim \lambda^{-1}$) and one must consider instead the properly scaled $\gamma_{\cos},\gamma_{\sin}$ and $\lambda \bar \Q $ in \eqref{eq:def-theta}.
  Like $\delta_{\cos}$ and $\delta_{\sin}$, $\gamma_{\cos}$ and $\gamma_{\sin}$ also decrease as $N/n$ grows large.
  In particular, one has $\gamma_{\cos}, \gamma_{\sin}, \| \lambda \bar \Q \| \to 0$ as $2N - n \uparrow 0$.
  \item 
  \emph{Over-parameterized phase}: with $2N>n$.
  Here, one can consider $\delta_{\cos}, \delta_{\sin}$ and $\| \bar \Q \|$.
  One has particularly that $\delta_{\cos}, \delta_{\sin}, \| \bar \Q \| \to \infty$ as $2N - n \downarrow 0$ and tend to zero as $N/n \to \infty$.
\end{enumerate}
With this discussion on the two phases of learning, we now understand why:
\begin{itemize}
  \item 
  in the leftmost plot of Figure~\ref{fig:trend-delta-lambda} with $2N < n$, $\delta_{\cos}$ and $\delta_{\sin}$ behave rather differently from other plots and approximately scale as $\lambda^{-1}$ for small values of $\lambda$; and
  \item 
  in the first and second leftmost plots of Figure~\ref{fig:trend-delta-N/n}, a ``jump'' in the values of $\delta$ occurs at the transition point $2N = n$, and the $\delta$'s are numerically of the same order of $\lambda^{-1}$ for $2N < n$, as predicted.
\end{itemize}
To characterize the phase transition from \eqref{eq:def-theta} and \eqref{eq:classi-fixed-point} in the $\lambda \to 0$ setting, we consider the scaled~variables
\begin{equation}\label{eq:def-gamma}
    \begin{cases}
        \gamma_\sigma = \lambda \delta_\sigma & \textmd{for $2 N < n$}\\ 
        \gamma_\sigma = \delta_\sigma & \textmd{for $2 N > n$}\\ 
    \end{cases}, \quad \sigma \in \{\cos, \sin\}.
\end{equation}
An advantage of using these scaled variables is that they are of order $O(1)$ as $n,p,N \to \infty$ and $\lambda \to 0$.
The behavior of $(\gamma_{\cos}, \gamma_{\sin})$ is reported in Figure~\ref{fig:trend-delta-N/n-ridgeless}, in the same setting as Figure~\ref{fig:trend-delta-N/n}. 
Observe the \emph{sharp} transition between the $2N < n$ and $2N > n$ regime, in particular for $\lambda = 10^{-7}$ and $\lambda=10^{-3}$, and that this transition is smoothed out for $\lambda = 1$.
(A ``transition'' is also seen for $\lambda = 10$, but this is potentially misleading.  It is true that $\gamma_{\cos}$ and $\gamma_{\sin}$ do change in this way, as a function of $N/n$, but unless $\lambda \approx 0$, these quantities are \emph{not} solutions of the corresponding fixed point equations.)

On account of these two different phases of learning (under- and over-parameterized, in \eqref{eq:def-theta}~and~\eqref{eq:classi-fixed-point}, respectively) and the sharp transition of $(\gamma_{\cos}, \gamma_{\sin})$ in Figure~\ref{fig:trend-delta-N/n-ridgeless}, it is not surprising to observe a ``singular'' behavior at $2N=n$ , when no regularization is applied. 
We next examine the asymptotic training and test error in more detail.

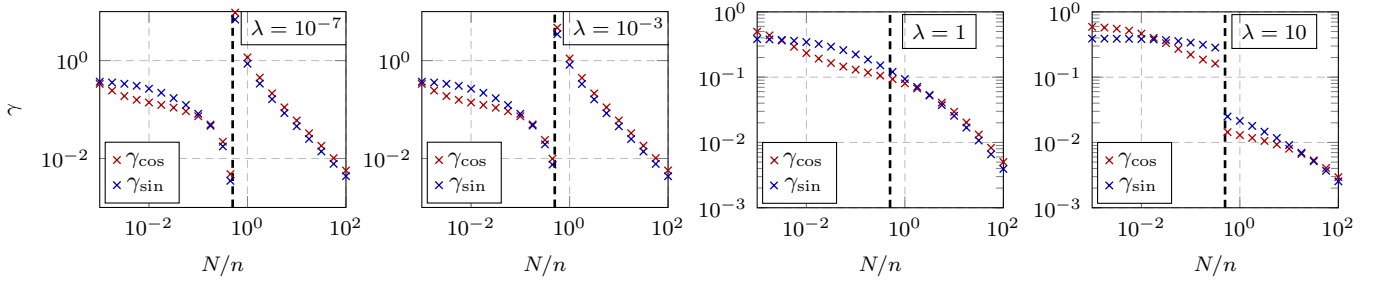
\begin{figure}[t] 
\begin{center}
\begin{minipage}[b]{0.24\columnwidth}%
  \begin{tikzpicture}[font=\scriptsize]
    \pgfplotsset{every major grid/.style={style=densely dashed}}
    \begin{loglogaxis}[
      width=1.15\linewidth,
      xmin=1e-3,
      xmax=1e2,
      ymin=1e-3,
      ymax=10,
      grid=major,
      xlabel={ $N/n$},
      ylabel={ $\gamma$ },
      legend style = {at={(0.02,0.02)}, anchor=south west, font=\footnotesize}
      ]
      \addplot+[RED,only marks,mark=x,line width=.5pt] coordinates{
          (0.001000,0.333957)(0.001778,0.243580)(0.003162,0.188070)(0.005623,0.158015)(0.010000,0.139172)(0.017783,0.124102)(0.031623,0.109219)(0.056234,0.092596)(0.100000,0.072960)(0.177828,0.049392)(0.316228,0.021976)(0.45,0.0048)(0.562341,9.535147)(1.000000,1.160396)(1.778279,0.448714)(3.162278,0.214196)(5.623413,0.110959)(10.000000,0.059740)(17.782794,0.032807)(31.622777,0.018209)(56.234133,0.010165)(100.000000,0.005693)
      };
      \addlegendentry{ {$ \gamma_{\cos} $} }; %
      \addplot+[BLUE,only marks,mark=x,line width=.5pt] coordinates{
          (0.001000,0.370038)(0.001778,0.356667)(0.003162,0.335688)(0.005623,0.305482)(0.010000,0.266118)(0.017783,0.219956)(0.031623,0.171100)(0.056234,0.124011)(0.100000,0.082085)(0.177828,0.046671)(0.316228,0.017665)(0.45,0.0035)(0.562341,6.884026)(1.000000,0.861765)(1.778279,0.337986)(3.162278,0.162544)(5.623413,0.084541)(10.000000,0.045618)(17.782794,0.025083)(31.622777,0.013931)(56.234133,0.007780)(100.000000,0.004358)
      };
      \addlegendentry{ {$ \gamma_{\sin} $} }; %
      \addplot[densely dashed,black,line width=1pt] coordinates{(0.5,0.001)(0.5,10)};
      \node[draw] at (axis cs:10,5) { $\lambda = 10^{-7}$ };
    \end{loglogaxis}
  \end{tikzpicture}
\end{minipage}
\hfill{}
\begin{minipage}[b]{0.24\columnwidth}%
  \begin{tikzpicture}[font=\scriptsize]
    \pgfplotsset{every major grid/.style={style=densely dashed}}
    \begin{loglogaxis}[
      width=1.15\linewidth,
      xmin=1e-3,
      xmax=1e2,
      ymin=1e-3,
      ymax=10,
      grid=major,
      scaled ticks=true,
      xlabel={ $N/n$},
      ylabel=\empty,
      legend style = {at={(0.02,0.02)}, anchor=south west, font=\footnotesize}
      ]
      \addplot+[RED,only marks,mark=x,line width=.5pt] coordinates{
          (0.001000,0.334168)(0.001778,0.243661)(0.003162,0.187962)(0.005623,0.157774)(0.010000,0.138851)(0.017783,0.123736)(0.031623,0.108837)(0.056234,0.092260)(0.100000,0.072821)(0.177828,0.049778)(0.316228,0.023902)(0.45,0.0099)(0.562341,4.690642)(1.000000,1.103588)(1.778279,0.443367)(3.162278,0.213539)(5.623413,0.110976)(10.000000,0.059835)(17.782794,0.032883)(31.622777,0.018258)(56.234133,0.010194)(100.000000,0.005710)
      };
      \addlegendentry{ {$ \gamma_{\cos} $} }; %
      \addplot+[BLUE,only marks,mark=x,line width=.5pt] coordinates{
          (0.001000,0.369888)(0.001778,0.356409)(0.003162,0.335294)(0.005623,0.305006)(0.010000,0.265644)(0.017783,0.219483)(0.031623,0.170581)(0.056234,0.123543)(0.100000,0.081929)(0.177828,0.047156)(0.316228,0.019443)(0.45,0.0075)(0.562341,3.462995)(1.000000,0.817874)(1.778279,0.332095)(3.162278,0.160988)(5.623413,0.083973)(10.000000,0.045369)(17.782794,0.024962)(31.622777,0.013869)(56.234133,0.007747)(100.000000,0.004340)
      };
      \addlegendentry{ {$ \gamma_{\sin} $} }; %
      \addplot[densely dashed,black,line width=1pt] coordinates{(0.5,0.001)(0.5,100)};
      \node[draw] at (axis cs:10,5) { $\lambda = 10^{-3}$ };
    \end{loglogaxis}
  \end{tikzpicture}
\end{minipage}
\hfill{}
\begin{minipage}[b]{0.24\columnwidth}%
  \begin{tikzpicture}[font=\scriptsize]
    \pgfplotsset{every major grid/.style={style=densely dashed}}
    \begin{loglogaxis}[
      width=1.15\linewidth,
      xmin=1e-3,
      xmax=1e2,
      ymin=1e-3,
      ymax=1,
      grid=major,
      scaled ticks=true,
      xlabel={ $N/n$},
      ylabel=\empty,
      legend style = {at={(0.02,0.02)}, anchor=south west, font=\footnotesize}
      ]
      \addplot+[RED,only marks,mark=x,line width=.5pt] coordinates{
          (0.001000,0.497800)(0.001778,0.436647)(0.003162,0.363688)(0.005623,0.292138)(0.010000,0.234057)(0.017783,0.192693)(0.031623,0.164703)(0.056234,0.145243)(0.100000,0.130430)(0.177828,0.117729)(0.316228,0.105633)(0.562341,0.093311)(1.000000,0.080411)(1.778279,0.067005)(3.162278,0.053554)(5.623413,0.040791)(10.000000,0.029498)(17.782794,0.020247)(31.622777,0.013234)(56.234133,0.008291)(100.000000,0.005019)
      };
      \addlegendentry{ {$ \gamma_{\cos} $} }; %
      \addplot+[BLUE,only marks,mark=x,line width=.5pt] coordinates{
          (0.001000,0.383235)(0.001778,0.379128)(0.003162,0.372321)(0.005623,0.361486)(0.010000,0.345203)(0.017783,0.322512)(0.031623,0.293525)(0.056234,0.259602)(0.100000,0.223006)(0.177828,0.186297)(0.316228,0.151686)(0.562341,0.120622)(1.000000,0.093736)(1.778279,0.071063)(3.162278,0.052356)(5.623413,0.037303)(10.000000,0.025597)(17.782794,0.016886)(31.622777,0.010724)(56.234133,0.006584)(100.000000,0.003932)
      };
      \addlegendentry{ {$ \gamma_{\sin} $} }; %
      \addplot[densely dashed,black,line width=1pt] coordinates{(0.5,0.001)(0.5,100000000)};
      \node[draw] at (axis cs:5,0.5) { $\lambda = 1$ };
    \end{loglogaxis}
  \end{tikzpicture}
\end{minipage}
\hfill{}
\begin{minipage}[b]{0.24\columnwidth}%
  \begin{tikzpicture}[font=\scriptsize]
    \pgfplotsset{every major grid/.style={style=densely dashed}}
    \begin{loglogaxis}[
      width=1.15\linewidth,
      xmin=1e-3,
      xmax=1e2,
      ymin=1e-3,
      ymax=1,
      grid=major,
      scaled ticks=true,
      xlabel={ $N/n$},
      ylabel=\empty,
      legend style = {at={(0.02,0.02)}, anchor=south west, font=\footnotesize}
      ]
      \addplot+[RED,only marks,mark=x,line width=.5pt] coordinates{
          (0.001000,0.58980)(0.001778,0.57541)(0.003162,0.55192)(0.005623,0.51567)(0.010000,0.46452)(0.017783,0.40078)(0.031623,0.33263)(0.056234,0.27060)(0.100000,0.22148)(0.177828,0.18608)(0.316228,0.16143)(0.562341,0.014369)(1.000000,0.012972)(1.778279,0.011741)(3.162278,0.010544)(5.623413,0.009310)(10.000000,0.008014)(17.782794,0.006670)(31.622777,0.005328)(56.234133,0.004059)(100.000000,0.002938)
      };
      \addlegendentry{ {$ \gamma_{\cos} $} }; %
      \addplot+[BLUE,only marks,mark=x,line width=.5pt] coordinates{
          (0.001000,0.38958)(0.001778,0.38900)(0.003162,0.38799)(0.005623,0.38622)(0.010000,0.38319)(0.017783,0.37811)(0.031623,0.36987)(0.056234,0.35713)(0.100000,0.33869)(0.177828,0.31405)(0.316228,0.28389)(0.562341,0.024990)(1.000000,0.021428)(1.778279,0.017921)(3.162278,0.014642)(5.623413,0.011699)(10.000000,0.009138)(17.782794,0.006962)(31.622777,0.005153)(56.234133,0.003686)(100.000000,0.002538)
      };
      \addlegendentry{ {$ \gamma_{\sin} $} }; %
      \addplot[densely dashed,black,line width=1pt] coordinates{(0.5,0.001)(0.5,100000000)};
      \node[draw] at (axis cs:5,0.5) { $\lambda = 10$ };
    \end{loglogaxis}
  \end{tikzpicture}
\end{minipage}
\end{center}
\caption{ Behavior of $(\gamma_{\cos}, \gamma_{\sin})$ in \eqref{eq:fixed-point-delta} on MNIST data set (class $3$ versus $7$), as a function of the ratio $N/n$, for $p=784$, $n=1\,000$, $\lambda = 10^{-7}, 10^{-3}, 1, 10$. The {\bf black} dashed line represents the interpolation threshold $2 N =n$. }
\label{fig:trend-delta-N/n-ridgeless}
\end{figure}

\paragraph{Asymptotic training MSE as $\lambda \to 0$.}
In the under-parameterized regime with $2N<n$, combining \eqref{eq:def-theta}  we have that both $\lambda \bar \Q$ and $\frac{\bar \Q}{1 + \delta_\sigma} \sim \frac{\lambda \bar \Q}{\gamma_\sigma}, \sigma \in \{ \cos, \sin \}$ are well-behaved and are generally not zero. 
As a consequence, by Theorem~\ref{theo:asy-training-MSE}, the asymptotic training error $\bar E_{\train}$ tends to a nonzero limit as $\lambda \to 0$, measuring the residual information in the training set that is not captured by the regressor $\bbeta \in \RR^{2N}$. 
As $2N - n \uparrow 0$, we have $\gamma_{\cos}, \gamma_{\sin} \to 0$ and $\| \lambda \bar \Q \| \to 0$ so that $\bar E_{\train} \to 0$ and $\bbeta$ interpolates the entire training set. 
On the other hand, in the over-parameterized $2 N > n$ regime, one \emph{always} has $\bar E_{\train} = 0$.
This particularly implies the training error is ``continuous'' around the point $2N = n$.

\paragraph{Asymptotic test MSE as $\lambda \to 0$.}
Again, in the under-parameterized regime with $2N<n$, now consider the more involved asymptotic test error in Theorem~\ref{theo:asy-test-MSE}.
In particular, we will focus here on the case $\hat \X \neq \X$ (or, more precisely, they are sufficiently different from each other in such a way that $\| \X - \hat \X \| \not \to 0$ as $n,p,N \to \infty$ and $\lambda \to 0$; see further discussion below in Section~\ref{subsec:impact-train-test-similarity}) so that $\K_{\sigma}(\X,\X) \neq \K_{\sigma}(\hat \X,\X)$ and $\frac{N}n \hat \bPhi \bar \Q \neq \I_n  - \lambda \bar \Q$. 
In this case, the two-by-two matrix $ \bOmega$ in $\bar E_{\test}$ diverges to infinity at $2N = n$ in the $\lambda \to 0$ limit. 
(Indeed, the determinant $\det(\bOmega^{-1})$ scales as $\lambda$, per Lemma~\ref{lem:property-of-Delta}.) 
As a consequence, we have $\bar E_{\test} \to \infty$ as $2N \rightarrow n$, resulting in a sharp deterioration of the test performance around $2N = n$.
(Of course, this holds if no additional regularization is applied as discussed in Remark~\ref{rem:ridge-reg}.)
It is also interesting to note that, while $\bOmega$ also appears in $\bar E_{\train}$, we still obtain (asymptotically) zero training MSE at $2N = n$, despite the divergence of $\bOmega$, again due to the prefactor $\lambda^2$ in $\bar E_{\train}$.
If $\lambda \gtrsim 1$, then $\det(\bOmega^{-1})$ exhibits much more regular properties (Figure~\ref{fig:trend-inv-Omega}), as one would~expect.

\begin{figure}[t] 
\vskip 0.1in
\begin{center}
\begin{minipage}[b]{0.24\columnwidth}%
  \begin{tikzpicture}[font=\footnotesize]
    \pgfplotsset{every major grid/.style={style=densely dashed}}
    \begin{axis}[
      width=1.15\linewidth,
      xmin=0,
      xmax=1.5,
      ymin=0,
      ymax=1,
      ytick={0,0.5,1},
      grid=major,
      ymajorgrids=false,
      scaled ticks=true,
      xlabel={ $N/n$ },
      ylabel={ $\det(\bOmega^{-1})$ },
      legend style = {at={(0.98,0.98)}, anchor=north east, font=\footnotesize}
      ]
      \addplot[RED,densely dashed,line width=1.5pt] coordinates{
          (0.000000,1.000000)(0.050000,0.471290)(0.100000,0.380608)(0.150000,0.322382)(0.200000,0.278170)(0.250000,0.239452)(0.300000,0.201133)(0.350000,0.160290)(0.400000,0.115081)(0.410000,0.105361)(0.420000,0.095371)(0.430000,0.085087)(0.440000,0.074479)(0.450000,0.063509)(0.460000,0.052132)(0.470000,0.040297)(0.480000,0.027967)(0.490000,0.015285)(0.510000,0.015773)(0.520000,0.029031)(0.530000,0.042205)(0.540000,0.055028)(0.550000,0.067466)(0.560000,0.079528)(0.570000,0.091231)(0.580000,0.102595)(0.590000,0.113639)(0.600000,0.124381)(0.700000,0.218208)(0.800000,0.293633)(0.900000,0.356036)(1.000000,0.408587)(1.100000,0.453429)(1.200000,0.492111)(1.300000,0.525794)(1.400000,0.555370)(1.500000,0.581533)
      };
      \addplot[densely dashed,black,line width=1pt] coordinates{(0.5,0.001)(0.5,10)};
      \node[draw] at (axis cs:1,.8) { $\lambda = 10^{-7}$ };
    \end{axis}
    \end{tikzpicture}
\end{minipage}
\hfill{}
\begin{minipage}[b]{0.24\columnwidth}%
  \begin{tikzpicture}[font=\footnotesize]
    \pgfplotsset{every major grid/.style={style=densely dashed}}
    \begin{axis}[
      width=1.15\linewidth,
      xmin=0,
      xmax=1.5,
      ymin=0,
      ymax=1,
      ytick={0,0.5,1},
      grid=major,
      ymajorgrids=false,
      scaled ticks=true,
      xlabel={ $N/n$ },
      ylabel= \empty,
      legend style = {at={(0.98,0.98)}, anchor=north east, font=\footnotesize}
      ]
      \addplot[RED,densely dashed,line width=1.5pt] coordinates{
          (0.000000,1.000000)(0.050000,0.588413)(0.100000,0.505249)(0.150000,0.443079)(0.200000,0.388819)(0.250000,0.339492)(0.300000,0.294874)(0.350000,0.256345)(0.400000,0.226875)(0.410000,0.222429)(0.420000,0.218548)(0.430000,0.215263)(0.440000,0.212603)(0.450000,0.210588)(0.460000,0.209234)(0.470000,0.208547)(0.480000,0.208525)(0.490000,0.209156)(0.510000,0.212291)(0.520000,0.214732)(0.530000,0.217704)(0.540000,0.221165)(0.550000,0.225068)(0.560000,0.229368)(0.570000,0.234020)(0.580000,0.238978)(0.590000,0.244203)(0.600000,0.249653)(0.700000,0.310166)(0.800000,0.370209)(0.900000,0.423850)(1.000000,0.470470)(1.100000,0.510801)(1.200000,0.545797)(1.300000,0.576339)(1.400000,0.603167)(1.500000,0.626887)
      };
      \addplot[densely dashed,black,line width=1pt] coordinates{(0.5,0)(0.5,1)};
      \node[draw] at (axis cs:1,.8) { $\lambda = 10^{-3}$ };
    \end{axis}
    \end{tikzpicture}
\end{minipage}
\hfill{}
\begin{minipage}[b]{0.24\columnwidth}%
  \begin{tikzpicture}[font=\footnotesize]
    \pgfplotsset{every major grid/.style={style=densely dashed}}
    \begin{axis}[
      width=1.15\linewidth,
      xmin=0,
      xmax=1.5,
      ymin=0,
      ymax=1,
      ytick={0,0.5,1},
      grid=major,
      ymajorgrids=false,
      scaled ticks=true,
      xlabel={ $N/n$ },
      ylabel= \empty,
      legend style = {at={(0.98,0.98)}, anchor=north east, font=\footnotesize}
      ]
      \addplot[RED,densely dashed,line width=1.5pt] coordinates{
          (0.000000,1.000000)(0.050000,0.939052)(0.100000,0.942547)(0.150000,0.944446)(0.200000,0.945943)(0.250000,0.947219)(0.300000,0.948337)(0.350000,0.949331)(0.400000,0.950226)(0.410000,0.950395)(0.420000,0.950560)(0.430000,0.950723)(0.440000,0.950882)(0.450000,0.951038)(0.460000,0.951192)(0.470000,0.951343)(0.480000,0.951492)(0.490000,0.951638)(0.510000,0.951922)(0.520000,0.952061)(0.530000,0.952198)(0.540000,0.952333)(0.550000,0.952465)(0.560000,0.952596)(0.570000,0.952724)(0.580000,0.952851)(0.590000,0.952975)(0.600000,0.953098)(0.700000,0.954239)(0.800000,0.955244)(0.900000,0.956143)(1.000000,0.956956)(1.100000,0.957700)(1.200000,0.958385)(1.300000,0.959020)(1.400000,0.959614)(1.500000,0.960171)
      };
      \addplot[densely dashed,black,line width=1pt] coordinates{(0.5,0)(0.5,1)};
      \node[draw] at (axis cs:1,.8) { $\lambda = 1$ };
    \end{axis}
    \end{tikzpicture}
\end{minipage}
\hfill{}
\begin{minipage}[b]{0.24\columnwidth}%
  \begin{tikzpicture}[font=\footnotesize]
    \pgfplotsset{every major grid/.style={style=densely dashed}}
    \begin{axis}[
      width=1.15\linewidth,
      xmin=0,
      xmax=1.5,
      ymin=0,
      ymax=1,
      ytick={0,0.5,1},
      grid=major,
      ymajorgrids=false,
      scaled ticks=true,
      xlabel={ $N/n$ },
      ylabel= \empty,
      legend style = {at={(0.98,0.98)}, anchor=north east, font=\footnotesize}
      ]
      \addplot[RED,densely dashed,line width=1.5pt] coordinates{
          (0.000000,1.000000)(0.050000,0.988070)(0.100000,0.989846)(0.150000,0.990814)(0.200000,0.991382)(0.250000,0.991754)(0.300000,0.992019)(0.350000,0.992219)(0.400000,0.992379)(0.410000,0.992407)(0.420000,0.992435)(0.430000,0.992461)(0.440000,0.992486)(0.450000,0.992511)(0.460000,0.992535)(0.470000,0.992558)(0.480000,0.992580)(0.490000,0.992602)(0.510000,0.992644)(0.520000,0.992664)(0.530000,0.992683)(0.540000,0.992702)(0.550000,0.992721)(0.560000,0.992739)(0.570000,0.992757)(0.580000,0.992774)(0.590000,0.992791)(0.600000,0.992807)(0.700000,0.992956)(0.800000,0.993081)(0.900000,0.993190)(1.000000,0.993287)(1.100000,0.993374)(1.200000,0.993454)(1.300000,0.993528)(1.400000,0.993596)(1.500000,0.993660)
      };
      \addplot[densely dashed,black,line width=1pt] coordinates{(0.5,0)(0.5,1)};
      \node[draw] at (axis cs:1,.8) { $\lambda = 10$ };
    \end{axis}
    \end{tikzpicture}
\end{minipage}
\end{center}
\caption{ Behavior of $\det(\bOmega^{-1})$ on MNIST data set (class $3$ versus $7$), as a function of $N/n$, for $p=784$, $n=1\,000$ and $\lambda = 10^{-7}, 10^{-3}, 1, 10$. The {\bf black} dashed line represents the interpolation threshold $2 N =n$. 
}
\label{fig:trend-inv-Omega}
\end{figure}

\subsection{Impact of Training-test Similarity}
\label{subsec:impact-train-test-similarity}

Continuing our discussion of the RFF performance in the large $n,p,N$ limit, we can see that the (asymptotic) test error behaves entirely differently, depending on whether $\hat \X $ is ``close to'' $ \X$ or not. 
For $\hat \X = \X$, one has $\bar E_{\test} = \bar E_{\train}$ that decreases monotonically as $N$ grows large; while for $\hat \X$ ``sufficiently'' different from $\X$, $\bar E_{\test}$ diverges to infinity at $2N= n$. 
To have a more quantitative assessment of the influence of training-test data similarity on the test error, consider the special case $\hat n = n$ and $\hat \y = \y$.
In this case, it follows from Theorem~\ref{theo:asy-test-MSE} that
\begin{align*}
  &\Theta_\sigma  = \frac1N \tr (\K_\sigma + \K_\sigma(\hat \X, \hat \X) - 2 \K_\sigma(\hat \X, \X)) + \frac2n \tr \bar \Q \Delta \bPhi^\T \Delta \K_\sigma \\ 
  & + \frac{N}n \frac1n \tr \bar \Q \Delta \bPhi^\T \Delta \bPhi \bar \Q \K_\sigma + \frac{n}N \frac{\lambda^2}n \tr \bar \Q \K_\sigma \bar \Q - \frac{2\lambda}N \tr \bar \Q \Delta \K_\sigma - \frac{2\lambda}n \tr \bar \Q \Delta \bPhi^\T \bar \Q \K_\sigma ,
\end{align*}
for $\sigma \in \{\cos,\sin \}$, $\Delta \K_\sigma = \K_\sigma - \K_\sigma (\hat \X, \X)$ and $\Delta \bPhi \equiv \hat \bPhi - \bPhi$. 
Since in the ridgeless $\lambda \to 0$ limit the entries of $\bOmega$ scale as $\lambda^{-1}$, one must scale $\Theta_\sigma$ with $\lambda$ so that $\bar E_{\test}$ does not diverge at $2N= n$ as $\lambda \to 0$. 
One example is the case where the test data is a small (additive) perturbation of the training data such that, in the kernel feature space
\[
  \K_\sigma - \K_\sigma(\hat \X ,\X) = \lambda \bXi_\sigma,~\K_\sigma (\hat \X, \hat \X) - \K_\sigma(\hat \X ,\X) = \lambda \hat \bXi_\sigma
\]
for $\bXi_\sigma, \hat \bXi_\sigma \in \RR^{n \times n}$ of bounded spectral norms. 
In this setting, we have $\Theta_\sigma = \frac{\lambda}N \tr (\bXi_\sigma + \hat \bXi_\sigma) + O(\lambda^2)$ so that the asymptotic test error does not diverge to infinity at $2N = n$ as $\lambda \to 0$. 
This is supported by Figure~\ref{fig:training-test-similarity}, where the test data are generated by adding Gaussian white noise of variance $\sigma^2$ to the training data, i.e.,
\begin{equation}\label{eq:similarity-model}
    \hat \x_i = \x_i + \sigma \boldsymbol{\varepsilon}_i
\end{equation}
for independent $\boldsymbol{\varepsilon}_i \sim \mathcal N(\zo, \I_p/p )$. 
In Figure~\ref{fig:training-test-similarity}, we observe that (i) below the threshold $\sigma^2 = \lambda$, test error coincides with the training error and both are close to zero; and (ii) as soon as $\sigma^2 > \lambda$, the test error diverges from the training error and grows large (but linearly in $\sigma^2$) as the noise level increases. Note also from the two rightmost plots of Figure~\ref{fig:training-test-similarity} that, the training-to-test ``transition'' at $\sigma^2 \sim \lambda$ is \emph{sharp} only for relatively small values of $\lambda$, as predicted by our theory.

\begin{figure}[t] 
\vskip 0.1in
\begin{center}
\begin{minipage}[b]{0.24\columnwidth}%
  \begin{tikzpicture}[font=\scriptsize]
    \pgfplotsset{every major grid/.style={style=densely dashed}}
    \begin{loglogaxis}[
      width=1.2\linewidth,
      xmin=1e-10,
      xmax=1e-2,
      xtick={1e-10,1e-7,1e-4},
    grid=major,
    ymajorgrids=false,
    scaled ticks=true,
    xlabel={ Noise variance $\sigma^2$ },
    ylabel={ MSE },
    legend style = {at={(0.02,0.75)}, anchor=north west, font=\scriptsize}
      ]
      \addplot+[BLUE,only marks,mark=x,line width=.5pt] coordinates{
      (0.000000000100,0.000026)(0.000000000158,0.000033)(0.000000000251,0.000039)(0.000000000398,0.000039)(0.000000000631,0.000029)(0.000000001000,0.000026)(0.000000001585,0.000032)(0.000000002512,0.000023)(0.000000003981,0.000032)(0.000000006310,0.000053)(0.000000010000,0.000030)(0.000000015849,0.000028)(0.000000025119,0.000051)(0.000000039811,0.000042)(0.000000063096,0.000040)(0.000000100000,0.000053)(0.000000158489,0.000045)(0.000000251189,0.000063)(0.000000398107,0.000085)(0.000000630957,0.000105)(0.000001000000,0.000174)(0.000001584893,0.000261)(0.000002511886,0.000431)(0.000003981072,0.000515)(0.000006309573,0.000685)(0.000010000000,0.001104)(0.000015848932,0.002788)(0.000025118864,0.003005)(0.000039810717,0.004328)(0.000063095734,0.008595)(0.000100000000,0.012543)(0.000158489319,0.018436)(0.000251188643,0.032821)(0.000398107171,0.041604)(0.000630957344,0.070976)(0.001000000000,0.130489)(0.001584893192,0.216937)(0.002511886432,0.288067)(0.003981071706,0.560017)(0.006309573445,0.728798)(0.010000000000,1.502515)
      };
      \addlegendentry{ {$ E_{\test} $} }; %
      \addplot+[RED,only marks,mark=o,line width=.5pt,mark size=1.5pt] coordinates{
      (0.000000000100,0.000026)(0.000000000158,0.000033)(0.000000000251,0.000039)(0.000000000398,0.000039)(0.000000000631,0.000029)(0.000000001000,0.000026)(0.000000001585,0.000032)(0.000000002512,0.000022)(0.000000003981,0.000031)(0.000000006310,0.000053)(0.000000010000,0.000029)(0.000000015849,0.000026)(0.000000025119,0.000047)(0.000000039811,0.000037)(0.000000063096,0.000031)(0.000000100000,0.000040)(0.000000158489,0.000027)(0.000000251189,0.000035)(0.000000398107,0.000033)(0.000000630957,0.000033)(0.000001000000,0.000026)(0.000001584893,0.000044)(0.000002511886,0.000063)(0.000003981072,0.000027)(0.000006309573,0.000030)(0.000010000000,0.000036)(0.000015848932,0.000030)(0.000025118864,0.000024)(0.000039810717,0.000027)(0.000063095734,0.000055)(0.000100000000,0.000042)(0.000158489319,0.000032)(0.000251188643,0.000030)(0.000398107171,0.000033)(0.000630957344,0.000023)(0.001000000000,0.000039)(0.001584893192,0.000027)(0.002511886432,0.000031)(0.003981071706,0.000030)(0.006309573445,0.000042)(0.010000000000,0.000033)
      };
      \addlegendentry{ {$ E_{\train} $} }; %
      \node[draw] at (axis cs:1e-8,1) {$\lambda = 10^{-7}$ };
    \end{loglogaxis}
  \end{tikzpicture}
\end{minipage}
\hfill{}
\begin{minipage}[b]{0.24\columnwidth}%
  \begin{tikzpicture}[font=\scriptsize]
    \pgfplotsset{every major grid/.style={style=densely dashed}}
    \begin{loglogaxis}[
      width=1.2\linewidth,
      xmin=1e-6,
      xmax=1,
      xtick={1e-5,1e-3,1e-1},
      grid=major,
      ymajorgrids=false,
      scaled ticks=true,
      xlabel={ Noise variance $\sigma^2$ },
      ylabel=\empty,
      legend style = {at={(0.02,0.02)}, anchor=south west, font=\footnotesize}
      ]
      \addplot+[BLUE,only marks,mark=x,line width=.5pt] coordinates{
      (0.000001000000,0.002735)(0.000001584893,0.002788)(0.000002511886,0.002718)(0.000003981072,0.002630)(0.000006309573,0.002877)(0.000010000000,0.002701)(0.000015848932,0.002683)(0.000025118864,0.002643)(0.000039810717,0.002673)(0.000063095734,0.002963)(0.000100000000,0.002794)(0.000158489319,0.003021)(0.000251188643,0.003079)(0.000398107171,0.003203)(0.000630957344,0.003356)(0.001000000000,0.003913)(0.001584893192,0.004627)(0.002511886432,0.005953)(0.003981071706,0.007795)(0.006309573445,0.010345)(0.010000000000,0.015188)(0.015848931925,0.022537)(0.025118864315,0.033272)(0.039810717055,0.052432)(0.063095734448,0.081200)(0.100000000000,0.120379)(0.158489319246,0.193727)(0.251188643151,0.300936)(0.398107170553,0.480772)(0.630957344480,0.658030)(1.000000000000,0.986512)
      };
      \addplot+[RED,only marks,mark=o,line width=.5pt,mark size=1.5pt] coordinates{
      (0.000001000000,0.002734)(0.000001584893,0.002786)(0.000002511886,0.002715)(0.000003981072,0.002626)(0.000006309573,0.002868)(0.000010000000,0.002691)(0.000015848932,0.002666)(0.000025118864,0.002608)(0.000039810717,0.002625)(0.000063095734,0.002889)(0.000100000000,0.002664)(0.000158489319,0.002817)(0.000251188643,0.002749)(0.000398107171,0.002702)(0.000630957344,0.002595)(0.001000000000,0.002672)(0.001584893192,0.002694)(0.002511886432,0.002752)(0.003981071706,0.002753)(0.006309573445,0.002642)(0.010000000000,0.002815)(0.015848931925,0.002826)(0.025118864315,0.002679)(0.039810717055,0.002743)(0.063095734448,0.002746)(0.100000000000,0.002671)(0.158489319246,0.002887)(0.251188643151,0.002761)(0.398107170553,0.002906)(0.630957344480,0.002742)(1.000000000000,0.002765)
      };
     \node[draw] at (axis cs:5e-5,0.8) {$\lambda = 10^{-3}$ };
    \end{loglogaxis}
  \end{tikzpicture}
\end{minipage}
\hfill{}
\begin{minipage}[b]{0.24\columnwidth}%
  \begin{tikzpicture}[font=\scriptsize]
    \pgfplotsset{every major grid/.style={style=densely dashed}}
    \begin{loglogaxis}[
      width=1.2\linewidth,
      xmin=1e-6,
      xmax=10,
      ymax=5,
      xtick={1e-5,1e-3,1e-1},
      grid=major,
      ymajorgrids=false,
      scaled ticks=true,
      xlabel={ Noise variance $\sigma^2$ },
      ylabel=\empty,
      legend style = {at={(0.02,0.02)}, anchor=south west, font=\footnotesize}
      ]
      \addplot+[BLUE,only marks,mark=x,line width=.5pt] coordinates{
      (0.000001000000,0.077177)(0.000001584893,0.076451)(0.000002511886,0.078551)(0.000003981072,0.077435)(0.000006309573,0.076806)(0.000010000000,0.077717)(0.000015848932,0.077647)(0.000025118864,0.077255)(0.000039810717,0.078225)(0.000063095734,0.078181)(0.000100000000,0.077839)(0.000158489319,0.078192)(0.000251188643,0.079879)(0.000398107171,0.078948)(0.000630957344,0.077592)(0.001000000000,0.078543)(0.001584893192,0.076577)(0.002511886432,0.077477)(0.003981071706,0.078227)(0.006309573445,0.078301)(0.010000000000,0.079202)(0.015848931925,0.078072)(0.025118864315,0.080334)(0.039810717055,0.079789)(0.063095734448,0.083742)(0.100000000000,0.088536)(0.158489319246,0.096122)(0.251188643151,0.106642)(0.398107170553,0.135957)(0.630957344480,0.184268)(1.000000000000,0.267565)(1.584893192461,0.403337)(2.511886431510,0.599158)(3.981071705535,0.804987)(6.309573444802,0.951182)(10.000000000000,1.013940)(15.848931924611,1.025010)(25.118864315096,1.028581)(39.810717055350,1.031739)(63.095734448019,1.028122)(100.000000000000,1.025809)
      };
      \addplot+[RED,only marks,mark=o,line width=.5pt,mark size=1.5pt] coordinates{
      (0.000001000000,0.077178)(0.000001584893,0.076451)(0.000002511886,0.078551)(0.000003981072,0.077436)(0.000006309573,0.076804)(0.000010000000,0.077717)(0.000015848932,0.077643)(0.000025118864,0.077252)(0.000039810717,0.078227)(0.000063095734,0.078171)(0.000100000000,0.077831)(0.000158489319,0.078178)(0.000251188643,0.079853)(0.000398107171,0.078914)(0.000630957344,0.077527)(0.001000000000,0.078425)(0.001584893192,0.076470)(0.002511886432,0.077256)(0.003981071706,0.077893)(0.006309573445,0.077817)(0.010000000000,0.078222)(0.015848931925,0.076688)(0.025118864315,0.077927)(0.039810717055,0.076049)(0.063095734448,0.077619)(0.100000000000,0.078139)(0.158489319246,0.078051)(0.251188643151,0.075862)(0.398107170553,0.078358)(0.630957344480,0.080272)(1.000000000000,0.078647)(1.584893192461,0.077522)(2.511886431510,0.079564)(3.981071705535,0.078292)(6.309573444802,0.076867)(10.000000000000,0.078048)(15.848931924611,0.077146)(25.118864315096,0.077366)(39.810717055350,0.079470)(63.095734448019,0.076871)(100.000000000000,0.077142)
      };
     \node[draw] at (axis cs:3e-5,3) {$\lambda = 1$ };
    \end{loglogaxis}
  \end{tikzpicture}
\end{minipage}
\hfill{}
\begin{minipage}[b]{0.24\columnwidth}%
  \begin{tikzpicture}[font=\scriptsize]
    \pgfplotsset{every major grid/.style={style=densely dashed}}
    \begin{loglogaxis}[
      width=1.2\linewidth,
      xmin=1e-6,
      xmax=10,
      ymax=5,
      ymin=1e-1,
      xtick={1e-5,1e-3,1e-1},
      grid=major,
      ymajorgrids=false,
      scaled ticks=true,
      xlabel={ Noise variance $\sigma^2$ },
      ylabel=\empty,
      legend style = {at={(0.02,0.02)}, anchor=south west, font=\footnotesize}
      ]
      \addplot+[BLUE,only marks,mark=x,line width=.5pt] coordinates{
      (0.000001000000,0.198615)(0.000001584893,0.199663)(0.000002511886,0.196371)(0.000003981072,0.194391)(0.000006309573,0.196825)(0.000010000000,0.194874)(0.000015848932,0.197561)(0.000025118864,0.195060)(0.000039810717,0.194077)(0.000063095734,0.196349)(0.000100000000,0.195351)(0.000158489319,0.192463)(0.000251188643,0.196159)(0.000398107171,0.196978)(0.000630957344,0.197578)(0.001000000000,0.196525)(0.001584893192,0.199955)(0.002511886432,0.196484)(0.003981071706,0.191418)(0.006309573445,0.196083)(0.010000000000,0.197796)(0.015848931925,0.199932)(0.025118864315,0.199614)(0.039810717055,0.202861)(0.063095734448,0.206464)(0.100000000000,0.212631)(0.158489319246,0.224206)(0.251188643151,0.242882)(0.398107170553,0.269523)(0.630957344480,0.321153)(1.000000000000,0.400175)(1.584893192461,0.521788)(2.511886431510,0.679246)(3.981071705535,0.842081)(6.309573444802,0.953523)(10.000000000000,0.999080)(15.848931924611,1.007154)(25.118864315096,1.008755)(39.810717055350,1.006763)(63.095734448019,1.007870)(100.000000000000,1.007030)
      };
      \addplot+[RED,only marks,mark=o,line width=.5pt,mark size=1.5pt] coordinates{
      (0.000001000000,0.198615)(0.000001584893,0.199662)(0.000002511886,0.196369)(0.000003981072,0.194392)(0.000006309573,0.196824)(0.000010000000,0.194873)(0.000015848932,0.197555)(0.000025118864,0.195056)(0.000039810717,0.194068)(0.000063095734,0.196336)(0.000100000000,0.195333)(0.000158489319,0.192451)(0.000251188643,0.196126)(0.000398107171,0.196915)(0.000630957344,0.197477)(0.001000000000,0.196347)(0.001584893192,0.199703)(0.002511886432,0.196003)(0.003981071706,0.190747)(0.006309573445,0.195103)(0.010000000000,0.196144)(0.015848931925,0.197312)(0.025118864315,0.195291)(0.039810717055,0.196135)(0.063095734448,0.195809)(0.100000000000,0.195104)(0.158489319246,0.195677)(0.251188643151,0.196274)(0.398107170553,0.193571)(0.630957344480,0.195957)(1.000000000000,0.196849)(1.584893192461,0.195940)(2.511886431510,0.196917)(3.981071705535,0.197598)(6.309573444802,0.200675)(10.000000000000,0.194670)(15.848931924611,0.195564)(25.118864315096,0.196093)(39.810717055350,0.197776)(63.095734448019,0.196146)(100.000000000000,0.197109)
      };
     \node[draw] at (axis cs:3e-5,3) {$\lambda = 10$ };
    \end{loglogaxis}
  \end{tikzpicture}
\end{minipage}
\end{center}
\caption{ { Empirical training and test errors of RFF ridgeless regression on MNIST data (class $3$ versus $7$), when modeling training-test similarity as $\hat \X = \X + \sigma \boldsymbol{\varepsilon}$, with $\boldsymbol{\varepsilon}$ having i.i.d~$\mathcal{N}(0,1/p)$ entries, as a function of the noise level $\sigma^2$, for $N = 512$, $p=784$, $n = \hat n = 1\,024 = 2N$, $\lambda = 10^{-7}, 10^{-3}, 1, 10$. Results obtained by averaging over $30$~runs.  } }
\label{fig:training-test-similarity}
\end{figure}
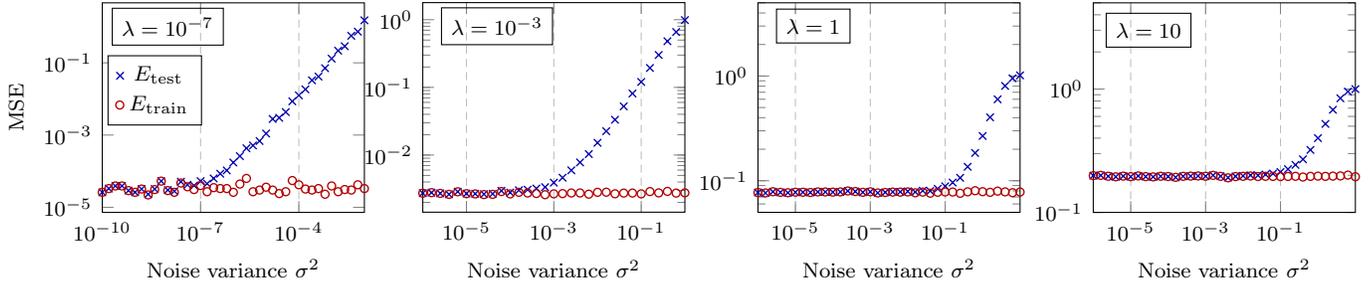

\subsection{Additional Real-world Data sets}
\label{sec:num}

So far, we have presented results in detail for one particular real-world data set, but we have extensive empirical results demonstrating that similar conclusions hold  more broadly.
As an example of these additional results, here we present a numerical evaluation of our results on several other real-world image data sets. 
We consider the classification task on another two MNIST-like data sets composed of $28 \times 28$ grayscale images: the Fashion-MNIST \cite{xiao2017fashion} and the Kannada-MNIST \cite{prabhu2019kannada} data sets. 
Each image is represented as a $p = 784$-dimensional vector and the output targets $\y, \hat \y$ are taken to have $-1,+1$ entries depending on the image class. 
As a consequence, both the training and test MSEs in \eqref{eq:def-MSE} are approximately $1$ for $N=0$ and significantly small $\lambda$, as observed in Figure~\ref{fig:double-descent-regularization-2} (and Figure~\ref{fig:other-MNIST-double-descent} below). 
For each data set, images were jointly centered and scaled so to fall close to the setting of Assumption~\ref{ass:high-dim} on $\X$ and $\hat \X$.

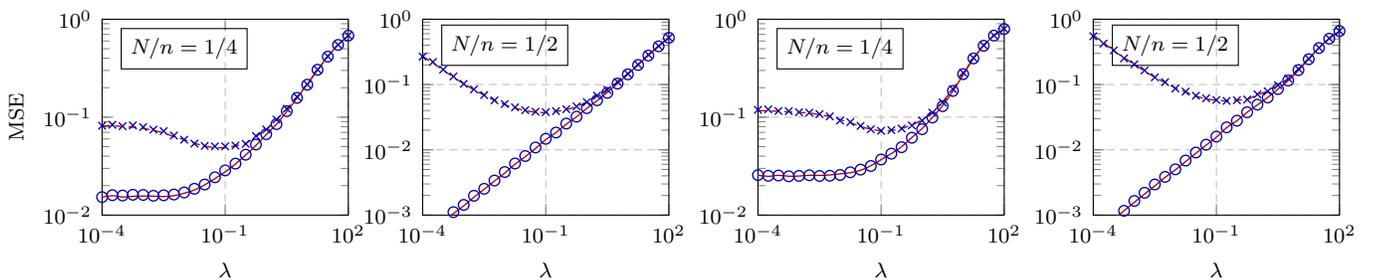
\begin{figure}[t] 
\vskip 0.1in
\begin{center}
\begin{minipage}[b]{0.24\columnwidth}%
  \begin{tikzpicture}[font=\scriptsize]
    \pgfplotsset{every major grid/.style={style=densely dashed}}
    \begin{loglogaxis}[
      width=1.15\linewidth,
      xmin=1e-4,
    xmax=1e2,
    ymin=1e-2,
    ymax=1,
    grid=major,
    ymajorgrids=false,
    scaled ticks=true,
    xlabel={ $\lambda$},
    ylabel={ MSE },
    legend style = {at={(0.02,0.98)}, anchor=north west, font=\footnotesize}
    ]
    \addplot+[BLUE,only marks,mark=o,line width=.5pt] coordinates{
    (0.000100,0.015266)(0.000178,0.015982)(0.000316,0.015850)(0.000562,0.016137)(0.001000,0.016070)(0.001778,0.015764)(0.003162,0.015945)(0.005623,0.016219)(0.010000,0.017028)(0.017783,0.018574)(0.031623,0.020527)(0.056234,0.024321)(0.100000,0.028576)(0.177828,0.033548)(0.316228,0.041378)(0.562341,0.052590)(1.000000,0.067051)(1.778279,0.084936)(3.162278,0.115226)(5.623413,0.157541)(10.000000,0.214581)(17.782794,0.304958)(31.622777,0.415423)(56.234133,0.546439)(100.000000,0.681059)
    };
    \addplot+[BLUE,only marks,mark=x,line width=.5pt] coordinates{
    (0.000100,0.081918)(0.000178,0.083620)(0.000316,0.080917)(0.000562,0.082466)(0.001000,0.079324)(0.001778,0.074606)(0.003162,0.072067)(0.005623,0.064991)(0.010000,0.058864)(0.017783,0.053763)(0.031623,0.050941)(0.056234,0.050338)(0.100000,0.050352)(0.177828,0.051324)(0.316228,0.053453)(0.562341,0.064233)(1.000000,0.074931)(1.778279,0.094641)(3.162278,0.121760)(5.623413,0.161369)(10.000000,0.216030)(17.782794,0.306309)(31.622777,0.417084)(56.234133,0.548762)(100.000000,0.682293)
    };
    \addplot[RED,smooth,line width=.5pt] coordinates{
    (0.000100,0.015214)(0.000178,0.015799)(0.000316,0.015304)(0.000562,0.015557)(0.001000,0.015666)(0.001778,0.015613)(0.003162,0.015482)(0.005623,0.015786)(0.010000,0.016725)(0.017783,0.018063)(0.031623,0.020341)(0.056234,0.023607)(0.100000,0.027668)(0.177828,0.033216)(0.316228,0.041137)(0.562341,0.051532)(1.000000,0.066275)(1.778279,0.084833)(3.162278,0.113978)(5.623413,0.156749)(10.000000,0.217284)(17.782794,0.301552)(31.622777,0.417221)(56.234133,0.548305)(100.000000,0.681746)
    }; 
    \addplot[RED,densely dashed,line width=.5pt] coordinates{
    (0.000100,0.081155)(0.000178,0.082218)(0.000316,0.080035)(0.000562,0.079085)(0.001000,0.076932)(0.001778,0.073367)(0.003162,0.069392)(0.005623,0.062969)(0.010000,0.059066)(0.017783,0.054188)(0.031623,0.050557)(0.056234,0.048753)(0.100000,0.048795)(0.177828,0.050987)(0.316228,0.053675)(0.562341,0.063462)(1.000000,0.074341)(1.778279,0.094317)(3.162278,0.120721)(5.623413,0.160857)(10.000000,0.219440)(17.782794,0.303123)(31.622777,0.418753)(56.234133,0.550608)(100.000000,0.682692)
    };
     \node[draw] at (axis cs:1e-2,0.5) {$N/n=1/4$ };
    \end{loglogaxis}
  \end{tikzpicture}
\end{minipage}
\hfill{}
\begin{minipage}[b]{0.24\columnwidth}%
  \begin{tikzpicture}[font=\scriptsize]
    \pgfplotsset{every major grid/.style={style=densely dashed}}
    \begin{loglogaxis}[
      width=1.15\linewidth,
      xmin=1e-4,
    xmax=1e2,
    ymin=1e-3,
    ymax=1,
      grid=major,
      scaled ticks=true,
      xlabel={ $\lambda$},
      ylabel=\empty,
      legend style = {at={(0.02,0.02)}, anchor=south west, font=\footnotesize}
      ]
      \addplot+[BLUE,only marks,mark=o,line width=.5pt] coordinates{
    (0.000100,0.000430)(0.000178,0.000648)(0.000316,0.000824)(0.000562,0.001110)(0.001000,0.001433)(0.001778,0.001978)(0.003162,0.002530)(0.005623,0.003385)(0.010000,0.004569)(0.017783,0.006172)(0.031623,0.008033)(0.056234,0.010854)(0.100000,0.014752)(0.177828,0.018396)(0.316228,0.025226)(0.562341,0.032204)(1.000000,0.043497)(1.778279,0.057439)(3.162278,0.074575)(5.623413,0.103397)(10.000000,0.143416)(17.782794,0.199256)(31.622777,0.278325)(56.234133,0.384921)(100.000000,0.520389)
    };
    \addplot+[BLUE,only marks,mark=x,line width=.5pt] coordinates{
    (0.000100,0.267531)(0.000178,0.221936)(0.000316,0.169568)(0.000562,0.132901)(0.001000,0.102452)(0.001778,0.084269)(0.003162,0.068954)(0.005623,0.057311)(0.010000,0.050300)(0.017783,0.044720)(0.031623,0.040523)(0.056234,0.038342)(0.100000,0.037768)(0.177828,0.039207)(0.316228,0.041877)(0.562341,0.046265)(1.000000,0.053738)(1.778279,0.067646)(3.162278,0.082918)(5.623413,0.109567)(10.000000,0.146241)(17.782794,0.201937)(31.622777,0.286400)(56.234133,0.389542)(100.000000,0.522865)
    };
    \addplot[RED,smooth,line width=.5pt] coordinates{
    (0.000100,0.000434)(0.000178,0.000594)(0.000316,0.000809)(0.000562,0.001058)(0.001000,0.001384)(0.001778,0.001882)(0.003162,0.002493)(0.005623,0.003333)(0.010000,0.004399)(0.017783,0.006024)(0.031623,0.007850)(0.056234,0.010494)(0.100000,0.014350)(0.177828,0.018408)(0.316228,0.024707)(0.562341,0.031994)(1.000000,0.043231)(1.778279,0.056450)(3.162278,0.076121)(5.623413,0.103830)(10.000000,0.143776)(17.782794,0.200797)(31.622777,0.278746)(56.234133,0.387949)(100.000000,0.521108)
    }; 
    \addplot[RED,densely dashed,line width=.5pt] coordinates{
    (0.000100,0.272188)(0.000178,0.212874)(0.000316,0.166990)(0.000562,0.129311)(0.001000,0.102404)(0.001778,0.082804)(0.003162,0.068916)(0.005623,0.056969)(0.010000,0.049077)(0.017783,0.043541)(0.031623,0.039706)(0.056234,0.037626)(0.100000,0.037053)(0.177828,0.038781)(0.316228,0.040854)(0.562341,0.046114)(1.000000,0.053602)(1.778279,0.067067)(3.162278,0.084511)(5.623413,0.109989)(10.000000,0.147119)(17.782794,0.203118)(31.622777,0.286812)(56.234133,0.391970)(100.000000,0.523480)
    };
     \node[draw] at (axis cs:1e-2,0.4) {$N/n=1/2$ };
    \end{loglogaxis}
  \end{tikzpicture}
\end{minipage}
\hfill{}
\begin{minipage}[b]{0.24\columnwidth}%
  \begin{tikzpicture}[font=\scriptsize]
    \pgfplotsset{every major grid/.style={style=densely dashed}}
    \begin{loglogaxis}[
      width=1.15\linewidth,
      xmin=1e-4,
    xmax=1e2,
    ymin=1e-2,
    ymax=1,
      grid=major,
      scaled ticks=true,
      xlabel={ $\lambda$},
      ylabel=\empty,
      legend style = {at={(0.02,0.02)}, anchor=south west, font=\footnotesize}
      ]
      \addplot+[BLUE,only marks,mark=o,line width=.5pt] coordinates{
    (0.000100,0.025632)(0.000178,0.025196)(0.000316,0.025372)(0.000562,0.024903)(0.001000,0.025142)(0.001778,0.025516)(0.003162,0.025097)(0.005623,0.025218)(0.010000,0.025700)(0.017783,0.027373)(0.031623,0.029150)(0.056234,0.031746)(0.100000,0.036971)(0.177828,0.042379)(0.316228,0.049359)(0.562341,0.061651)(1.000000,0.075314)(1.778279,0.099279)(3.162278,0.129558)(5.623413,0.185451)(10.000000,0.275128)(17.782794,0.397962)(31.622777,0.536327)(56.234133,0.680788)(100.000000,0.794677)
    };
    \addplot+[BLUE,only marks,mark=x,line width=.5pt] coordinates{
    (0.000100,0.118873)(0.000178,0.120163)(0.000316,0.115656)(0.000562,0.115344)(0.001000,0.112286)(0.001778,0.111173)(0.003162,0.107513)(0.005623,0.101916)(0.010000,0.092429)(0.017783,0.089112)(0.031623,0.081438)(0.056234,0.075307)(0.100000,0.072658)(0.177828,0.073806)(0.316228,0.076450)(0.562341,0.082400)(1.000000,0.092275)(1.778279,0.111102)(3.162278,0.141521)(5.623413,0.193733)(10.000000,0.284174)(17.782794,0.403901)(31.622777,0.541233)(56.234133,0.682713)(100.000000,0.798454)
    };
    \addplot[RED,smooth,line width=.5pt] coordinates{
    (0.000100,0.025551)(0.000178,0.024882)(0.000316,0.025473)(0.000562,0.024951)(0.001000,0.025057)(0.001778,0.025488)(0.003162,0.025426)(0.005623,0.025689)(0.010000,0.026388)(0.017783,0.027113)(0.031623,0.028687)(0.056234,0.031912)(0.100000,0.036821)(0.177828,0.042544)(0.316228,0.049873)(0.562341,0.062156)(1.000000,0.076455)(1.778279,0.099437)(3.162278,0.131391)(5.623413,0.184435)(10.000000,0.270177)(17.782794,0.395434)(31.622777,0.538771)(56.234133,0.681810)(100.000000,0.793475)
    }; 
    \addplot[RED,densely dashed,line width=.5pt] coordinates{
    (0.000100,0.116925)(0.000178,0.115850)(0.000316,0.115204)(0.000562,0.113150)(0.001000,0.111764)(0.001778,0.110667)(0.003162,0.106004)(0.005623,0.102209)(0.010000,0.094336)(0.017783,0.088139)(0.031623,0.079996)(0.056234,0.075360)(0.100000,0.072309)(0.177828,0.072906)(0.316228,0.076499)(0.562341,0.082802)(1.000000,0.093795)(1.778279,0.110762)(3.162278,0.143548)(5.623413,0.192376)(10.000000,0.278769)(17.782794,0.401952)(31.622777,0.543669)(56.234133,0.684236)(100.000000,0.796715)
    };
     \node[draw] at (axis cs:1e-2,0.5) {$N/n=1/4$ };
    \end{loglogaxis}
  \end{tikzpicture}
\end{minipage}
\hfill{}
\begin{minipage}[b]{0.24\columnwidth}%
  \begin{tikzpicture}[font=\scriptsize]
    \pgfplotsset{every major grid/.style={style=densely dashed}}
    \begin{loglogaxis}[
      width=1.15\linewidth,
      xmin=1e-4,
    xmax=1e2,
    ymin=1e-3,
    ymax=1,
      grid=major,
      scaled ticks=true,
      xlabel={ $\lambda$},
      ylabel=\empty,
      legend style = {at={(0.02,0.02)}, anchor=south west, font=\footnotesize}
      ]
      \addplot+[BLUE,only marks,mark=o,line width=.5pt] coordinates{
    (0.000100,0.000490)(0.000178,0.000649)(0.000316,0.000913)(0.000562,0.001163)(0.001000,0.001649)(0.001778,0.002198)(0.003162,0.002946)(0.005623,0.003849)(0.010000,0.005021)(0.017783,0.006847)(0.031623,0.009089)(0.056234,0.011988)(0.100000,0.015941)(0.177828,0.022076)(0.316228,0.027916)(0.562341,0.037600)(1.000000,0.049213)(1.778279,0.063398)(3.162278,0.085376)(5.623413,0.115296)(10.000000,0.166313)(17.782794,0.244029)(31.622777,0.362291)(56.234133,0.503619)(100.000000,0.658570)
    };
    \addplot+[BLUE,only marks,mark=x,line width=.5pt] coordinates{
    (0.000100,0.549335)(0.000178,0.423090)(0.000316,0.333316)(0.000562,0.251528)(0.001000,0.203819)(0.001778,0.164379)(0.003162,0.129057)(0.005623,0.108721)(0.010000,0.086952)(0.017783,0.077913)(0.031623,0.067433)(0.056234,0.061557)(0.100000,0.058028)(0.177828,0.056148)(0.316228,0.057410)(0.562341,0.058929)(1.000000,0.070780)(1.778279,0.078897)(3.162278,0.097464)(5.623413,0.123699)(10.000000,0.173397)(17.782794,0.249112)(31.622777,0.364655)(56.234133,0.509136)(100.000000,0.662081)
    };
    \addplot[RED,smooth,line width=.5pt] coordinates{
    (0.000100,0.000503)(0.000178,0.000674)(0.000316,0.000860)(0.000562,0.001166)(0.001000,0.001571)(0.001778,0.002102)(0.003162,0.002803)(0.005623,0.003791)(0.010000,0.004934)(0.017783,0.006610)(0.031623,0.008852)(0.056234,0.011833)(0.100000,0.015584)(0.177828,0.021870)(0.316228,0.027835)(0.562341,0.037627)(1.000000,0.048763)(1.778279,0.063447)(3.162278,0.085113)(5.623413,0.114717)(10.000000,0.165084)(17.782794,0.244613)(31.622777,0.365779)(56.234133,0.506139)(100.000000,0.653405)
    }; 
    \addplot[RED,densely dashed,line width=.5pt] coordinates{
    (0.000100,0.570018)(0.000178,0.438140)(0.000316,0.325364)(0.000562,0.258014)(0.001000,0.203011)(0.001778,0.160791)(0.003162,0.128753)(0.005623,0.106812)(0.010000,0.088084)(0.017783,0.076624)(0.031623,0.066340)(0.056234,0.060056)(0.100000,0.057426)(0.177828,0.055940)(0.316228,0.057823)(0.562341,0.058736)(1.000000,0.069824)(1.778279,0.078837)(3.162278,0.096962)(5.623413,0.122630)(10.000000,0.172803)(17.782794,0.249011)(31.622777,0.368278)(56.234133,0.511864)(100.000000,0.656849)
    };
     \node[draw] at (axis cs:1e-2,0.4) {$N/n=1/2$ };
    \end{loglogaxis}
  \end{tikzpicture}
\end{minipage}
\end{center}
\caption{ { MSEs of RFF regression on Fashion-MNIST (\textbf{left two}) and Kannada-MNIST (\textbf{right two}) data (class $5$ versus $6$), as a function of regression parameter $\lambda$, for $p=784$, $n = \hat n=1\,024$, $N=256$ and $512$. Empirical results displayed in {\BLUE \textbf{blue}} (circles for training and crosses for test); and the asymptotics from Theorem~\ref{theo:asy-training-MSE}~and~\ref{theo:asy-test-MSE} displayed in {\RED \textbf{red}} (sold lines for training and dashed for test). Results obtained by averaging over $30$ runs.  } }
\label{fig:fashion-and-kannada-MNIST-gamma}
\end{figure}

In Figure~\ref{fig:fashion-and-kannada-MNIST-gamma}, we compare the empirical training and test errors with their limiting behaviors derived in Theorem~\ref{theo:asy-training-MSE}~and~\ref{theo:asy-test-MSE}, as a function of the penalty parameter $\lambda$, on a training set of size $n=1\,024$ ($512$ images from class $5$ and $512$ images from class $6$) with feature dimension $N=256$ and $N=512$, on both data sets. A close fit between theory and practice is observed, for moderately large values of $n,p,N$, demonstrating thus a wide practical applicability of the proposed asymptotic analyses, particularly compared to the (limiting) Gaussian kernel predictions per Figure~\ref{fig:compare-kernel-RMT}.

In Figure~\ref{fig:more-data-trend-delta-N/n}, we report the behavior of the pair $(\delta_{\cos}, \delta_{\sin})$ for small values of $\lambda = 10^{-7} $ and $ 10^{-3}$. 
Similar to the two leftmost plots in Figure~\ref{fig:trend-delta-N/n} for MNIST, a jump from the under- to over-parameterized regime occurs at the interpolation threshold $2N = n$, in both Fashion- and Kannada-MNIST data sets, clearly indicating the two phases of learning and the phase transition between them.

%
%
%
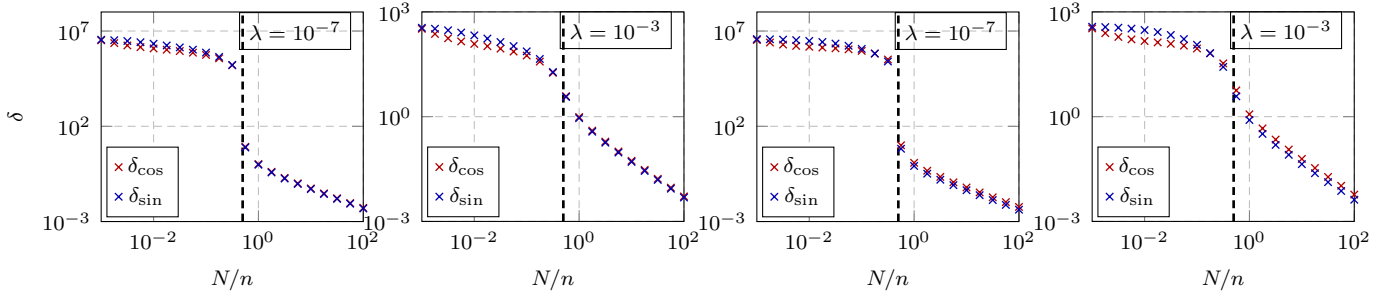
\begin{figure}[t] 
\vskip 0.1in
\begin{center}
\begin{minipage}[b]{0.24\columnwidth}%
  \begin{tikzpicture}[font=\scriptsize]
    \pgfplotsset{every major grid/.style={style=densely dashed}}
    \begin{loglogaxis}[
      width=1.2\linewidth,
      xmin=1e-3,
      xmax=1e2,
      ymin=1e-3,
      ymax=1e8,
      grid=major,
      xlabel={ $N/n$},
      ylabel={ $\delta$ },
      legend style = {at={(0.02,0.02)}, anchor=south west, font=\scriptsize}
      ]
      \addplot+[RED,only marks,mark=x,line width=.5pt] coordinates{
          (0.001000,3273602.056286)(0.001778,2327813.973018)(0.003162,1751244.067339)(0.005623,1429261.532996)(0.010000,1217149.157950)(0.017783,1047389.627170)(0.031623,890551.427939)(0.056234,731518.943278)(0.100000,560163.639364)(0.177828,369112.698758)(0.316228,159460.883742)(0.562341,8.403651)(1.000000,1.043027)(1.778279,0.407041)(3.162278,0.195205)(5.623413,0.101369)(10.000000,0.054650)(17.782794,0.030034)(31.622777,0.016676)(56.234133,0.009312)(100.000000,0.005216)
      };
      \addlegendentry{ {$ \delta_{\cos} $} }; %
      \addplot+[BLUE,only marks,mark=x,line width=.5pt] coordinates{
          (0.001000,3502054.859843)(0.001778,3296268.289882)(0.003162,2988373.118535)(0.005623,2589542.528168)(0.010000,2155856.209041)(0.017783,1748534.988995)(0.031623,1386279.550004)(0.056234,1055811.662417)(0.100000,741380.335537)(0.177828,439408.727822)(0.316228,166054.546595)(0.562341,7.663342)(1.000000,0.958735)(1.778279,0.375614)(3.162278,0.180504)(5.623413,0.093839)(10.000000,0.050621)(17.782794,0.027829)(31.622777,0.015455)(56.234133,0.008631)(100.000000,0.004835)
      };
      \addlegendentry{ {$ \delta_{\sin} $} }; %
      \addplot[densely dashed,black,line width=1pt] coordinates{(0.5,0.001)(0.5,100000000)};
      \node[draw] at (axis cs:5,1e7) { $\lambda = 10^{-7}$ };
    \end{loglogaxis}
  \end{tikzpicture}
\end{minipage}
\hfill{}
\begin{minipage}[b]{0.24\columnwidth}%
  \begin{tikzpicture}[font=\scriptsize]
    \pgfplotsset{every major grid/.style={style=densely dashed}}
    \begin{loglogaxis}[
      width=1.2\linewidth,
      xmin=1e-3,
      xmax=1e2,
      ymin=1e-3,
      ymax=1e3,
      grid=major,
      scaled ticks=true,
      xlabel={ $N/n$},
      ylabel=\empty,
      legend style = {at={(0.02,0.02)}, anchor=south west, font=\scriptsize}
      ]
      \addplot+[RED,only marks,mark=x,line width=.5pt] coordinates{
          (0.001000,327.851114)(0.001778,233.249263)(0.003162,175.147525)(0.005623,142.540330)(0.010000,121.106171)(0.017783,104.128213)(0.031623,88.616715)(0.056234,72.982670)(0.100000,56.204955)(0.177828,37.682125)(0.316228,18.019100)(0.562341,3.860870)(1.000000,0.975173)(1.778279,0.398568)(3.162278,0.193323)(5.623413,0.100820)(10.000000,0.054459)(17.782794,0.029958)(31.622777,0.016643)(56.234133,0.009295)(100.000000,0.005207)
          };
          \addlegendentry{ {$ \delta_{\cos} $} }; %
      \addplot+[BLUE,only marks,mark=x,line width=.5pt] coordinates{
          (0.001000,350.751067)(0.001778,329.404800)(0.003162,297.458263)(0.005623,256.473291)(0.010000,212.754500)(0.017783,172.513309)(0.031623,137.115941)(0.056234,104.877480)(0.100000,74.203323)(0.177828,44.910856)(0.316228,19.000964)(0.562341,3.667369)(1.000000,0.907480)(1.778279,0.370246)(3.162278,0.179604)(5.623413,0.093689)(10.000000,0.050616)(17.782794,0.027848)(31.622777,0.015471)(56.234133,0.008642)(100.000000,0.004841)
      };
      \addlegendentry{ {$ \delta_{\sin} $} }; %
      \addplot[densely dashed,black,line width=1pt] coordinates{(0.5,0.001)(0.5,100000000)};
      \node[draw] at (axis cs:5,300) { $\lambda = 10^{-3}$ };
    \end{loglogaxis}
  \end{tikzpicture}
\end{minipage}
\hfill{}
\begin{minipage}[b]{0.24\columnwidth}%
  \begin{tikzpicture}[font=\scriptsize]
    \pgfplotsset{every major grid/.style={style=densely dashed}}
    \begin{loglogaxis}[
      width=1.2\linewidth,
      xmin=1e-3,
      xmax=1e2,
      ymin=1e-3,
      ymax=1e8,
      grid=major,
      xlabel={ $N/n$},
      ylabel= \empty,
      legend style = {at={(0.02,0.02)}, anchor=south west, font=\scriptsize}
      ]
      \addplot+[RED,only marks,mark=x,line width=.5pt] coordinates{
          (0.001000,3385710.247339)(0.001778,2495579.205811)(0.003162,1947891.170050)(0.005623,1658632.558176)(0.010000,1488633.815742)(0.017783,1362214.069210)(0.031623,1239477.864217)(0.056234,1094940.838922)(0.100000,907401.992917)(0.177828,655487.202567)(0.316228,321076.039854)(0.562341,10.050943)(1.000000,1.207323)(1.778279,0.464203)(3.162278,0.220979)(5.623413,0.114309)(10.000000,0.061497)(17.782794,0.033758)(31.622777,0.018732)(56.234133,0.010456)(100.000000,0.005855)
      };
      \addlegendentry{ {$ \delta_{\cos} $} }; %
      \addplot+[BLUE,only marks,mark=x,line width=.5pt] coordinates{
          (0.001000,3777216.436690)(0.001778,3690256.965848)(0.003162,3549716.635858)(0.005623,3337582.992221)(0.010000,3041312.700198)(0.017783,2658573.014758)(0.031623,2198542.592862)(0.056234,1683958.614535)(0.100000,1154368.124511)(0.177828,661612.584356)(0.316228,251952.097957)(0.562341,6.618747)(1.000000,0.828273)(1.778279,0.325041)(3.162278,0.156394)(5.623413,0.081367)(10.000000,0.043913)(17.782794,0.024148)(31.622777,0.013413)(56.234133,0.007491)(100.000000,0.004196)
      };
      \addlegendentry{ {$ \delta_{\sin} $} }; %
      \addplot[densely dashed,black,line width=1pt] coordinates{(0.5,0.001)(0.5,100000000)};
      \node[draw] at (axis cs:5,1e7) { $\lambda = 10^{-7}$ };
    \end{loglogaxis}
  \end{tikzpicture}
\end{minipage}
\hfill{}
\begin{minipage}[b]{0.24\columnwidth}%
  \begin{tikzpicture}[font=\scriptsize]
    \pgfplotsset{every major grid/.style={style=densely dashed}}
    \begin{loglogaxis}[
      width=1.2\linewidth,
      xmin=1e-3,
      xmax=1e2,
      ymin=1e-3,
      ymax=1e3,
      grid=major,
      scaled ticks=true,
      xlabel={ $N/n$},
      ylabel=\empty,
      legend style = {at={(0.02,0.02)}, anchor=south west, font=\scriptsize}
      ]
      \addplot+[RED,only marks,mark=x,line width=.5pt] coordinates{
          (0.001000,336.418906)(0.001778,246.657959)(0.003162,191.948172)(0.005623,163.255463)(0.010000,146.470610)(0.017783,134.030530)(0.031623,121.970462)(0.056234,107.780563)(0.100000,89.437569)(0.177828,65.024878)(0.316228,33.545280)(0.562341,5.645283)(1.000000,1.168907)(1.778279,0.461038)(3.162278,0.220726)(5.623413,0.114410)(10.000000,0.061606)(17.782794,0.033832)(31.622777,0.018778)(56.234133,0.010483)(100.000000,0.005871)
          };
          \addlegendentry{ {$ \delta_{\cos} $} }; %
      \addplot+[BLUE,only marks,mark=x,line width=.5pt] coordinates{
          (0.001000,375.788719)(0.001778,367.123709)(0.003162,353.124699)(0.005623,331.954016)(0.010000,302.326370)(0.017783,264.024741)(0.031623,218.014452)(0.056234,166.659229)(0.100000,114.111693)(0.177828,65.757348)(0.316228,26.543390)(0.562341,3.822990)(1.000000,0.800661)(1.778279,0.321210)(3.162278,0.155297)(5.623413,0.080935)(10.000000,0.043712)(17.782794,0.024046)(31.622777,0.013359)(56.234133,0.007462)(100.000000,0.004180)
      };
      \addlegendentry{ {$ \delta_{\sin} $} }; %
      \addplot[densely dashed,black,line width=1pt] coordinates{(0.5,0.001)(0.5,100000000)};
      \node[draw] at (axis cs:5,300) { $\lambda = 10^{-3}$ };
    \end{loglogaxis}
  \end{tikzpicture}
\end{minipage}
\end{center}
\caption{ Behavior of $(\delta_{\cos}, \delta_{\sin})$ in \eqref{eq:fixed-point-delta}, on Fashion-MNIST (\textbf{left two}) and Kannada-MNIST (\textbf{right two}) data (class $8$ versus $9$), for $p=784$, $n=1000$, $\lambda = 10^{-7} $ and $ 10^{-3}$. The {\bf black} dashed line represents the interpolation threshold $2 N =n$. }
\label{fig:more-data-trend-delta-N/n}
\end{figure}

In Figure~\ref{fig:other-MNIST-double-descent}, we report the empirical and theoretical test errors as a function of the ratio $N/n$, on a training test of size $n=500$ ($250$ images from class $8$ and $250$ images from class $9$), by varying feature dimension $N$. An exceedingly small regularization $\lambda = 10^{-7}$ is applied to mimic the ``ridgeless'' limiting behavior as $\lambda \to 0$.
On both data sets, the corresponding double descent curve is observed where the test errors goes down and up, with a singular peak around $2N=n$, and then goes down monotonically as $N$ continues to increase when $2N > n$. 

%
%
%
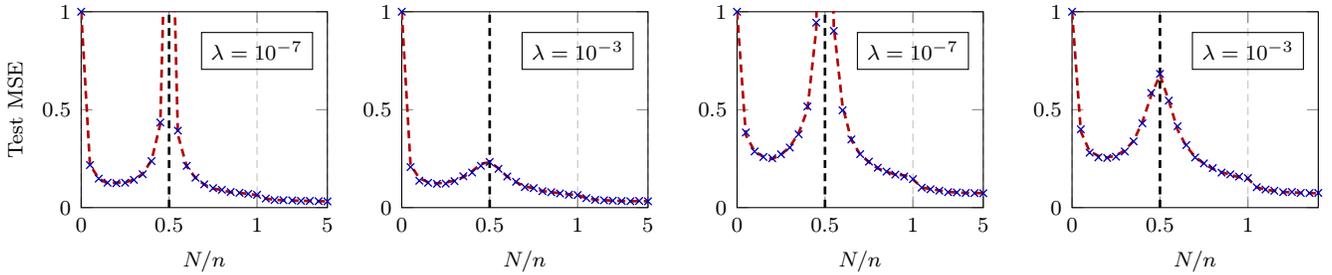
\begin{figure}[t] 
\vskip 0.1in
\begin{center}
\begin{minipage}[b]{0.24\columnwidth}%
  \begin{tikzpicture}[font=\scriptsize]
    \pgfplotsset{every major grid/.style={style=densely dashed}}
    \begin{axis}[
      width=1.15\linewidth,
      xmin=0,
      xmax=5,
      ymin=0,
      ymax=1,
      symbolic x coords={0,0.05,0.10,0.15,0.20,0.25,0.30,0.35,0.40,0.45,0.5,0.55,0.60,0.65,0.70,0.75,0.80,0.85,0.90,0.95,1,1.50,2.00,2.50,3.00,3.50,4.00,4.50,5},
      xtick={0,0.5,1,5},
      ytick={0,0.5,1},
      grid=major,
      ymajorgrids=false,
      scaled ticks=true,
      xlabel={ $N/n$ },
      ylabel={ Test MSE },
      legend style = {at={(0.98,0.98)}, anchor=north east, font=\scriptsize}
      ]
      \addplot[RED,densely dashed,line width=1pt] coordinates{
      (0,1.000000)(0.05,0.217603)(0.10,0.149248)(0.15,0.122806)(0.20,0.127601)(0.25,0.128675)(0.30,0.142880)(0.35,0.172554)(0.40,0.241256)(0.45,0.427845)(0.5,2.199542)(0.55,0.371986)(0.60,0.209942)(0.65,0.152701)(0.70,0.117498)(0.75,0.097239)(0.80,0.089800)(0.85,0.078308)(0.90,0.075203)(0.95,0.070072)(1,0.063157)(1.50,0.046528)(2.00,0.039214)(2.50,0.037958)(3.00,0.036706)(3.50,0.034165)(4.00,0.034731)(4.50,0.032753)(5,0.031741)
      };
      \addplot+[only marks,mark=x,BLUE,line width=0.5pt] coordinates{
      (0,1.000000)(0.05,0.217281)(0.10,0.148038)(0.15,0.127103)(0.20,0.124381)(0.25,0.127174)(0.30,0.141819)(0.35,0.170461)(0.40,0.237573)(0.45,0.434660)(0.5,2.164024)(0.55,0.392980)(0.60,0.213172)(0.65,0.152657)(0.70,0.118650)(0.75,0.099378)(0.80,0.090656)(0.85,0.079823)(0.90,0.075361)(0.95,0.070222)(1,0.065348)(1.50,0.046228)(2.00,0.038707)(2.50,0.038056)(3.00,0.036935)(3.50,0.034370)(4.00,0.034357)(4.50,0.032878)(5,0.031662)
      };
      \addplot[densely dashed,black,line width=1pt] coordinates{(0.5,0)(0.5,1)};
      \node[draw] at (axis cs:1,.8) { $\lambda = 10^{-7}$ };
    \end{axis}
    \end{tikzpicture}
\end{minipage}
\hfill{}
\begin{minipage}[b]{0.24\columnwidth}%
  \begin{tikzpicture}[font=\scriptsize]
    \pgfplotsset{every major grid/.style={style=densely dashed}}
    \begin{axis}[
      width=1.15\linewidth,
      xmin=0,
      xmax=5,
      ymin=0,
      ymax=1,
      symbolic x coords={0,0.05,0.10,0.15,0.20,0.25,0.30,0.35,0.40,0.45,0.5,0.55,0.60,0.65,0.70,0.75,0.80,0.85,0.90,0.95,1,1.50,2.00,2.50,3.00,3.50,4.00,4.50,5},
      xtick={0,0.5,1,5},
      ytick={0,0.5,1},
      grid=major,
      ymajorgrids=false,
      scaled ticks=true,
      xlabel={ $N/n$ },
      ylabel= \empty,
      legend style = {at={(0.98,0.98)}, anchor=north east, font=\scriptsize}
      ]
      \addplot[RED,densely dashed,line width=1pt] coordinates{
      (0,1)(0.05,0.218039)(0.10,0.150056)(0.15,0.129378)(0.20,0.124239)(0.25,0.126771)(0.30,0.139167)(0.35,0.155512)(0.40,0.185706)(0.45,0.223436)(0.5,0.231444)(0.55,0.201166)(0.60,0.157755)(0.65,0.129970)(0.70,0.108630)(0.75,0.098728)(0.80,0.084827)(0.85,0.078923)(0.90,0.073592)(0.95,0.067441)(1,0.064250)(1.50,0.048409)(2.00,0.040271)(2.50,0.037945)(3.00,0.035327)(3.50,0.033754)(4.00,0.032035)(4.50,0.033959)(5,0.032263)
      };
      \addplot+[only marks,mark=x,BLUE,line width=0.5pt] coordinates{
      (0,1.000000)(0.05,0.207101)(0.10,0.137250)(0.15,0.126631)(0.20,0.121395)(0.25,0.122707)(0.30,0.135157)(0.35,0.160125)(0.40,0.179191)(0.45,0.213404)(0.5,0.232725)(0.55,0.197972)(0.60,0.158832)(0.65,0.131745)(0.70,0.104732)(0.75,0.099360)(0.80,0.084006)(0.85,0.078470)(0.90,0.072402)(0.95,0.065702)(1,0.064221)(1.50,0.048472)(2.00,0.039663)(2.50,0.037411)(3.00,0.035543)(3.50,0.033885)(4.00,0.032455)(4.50,0.034319)(5,0.032286)
      };
      \addplot[densely dashed,black,line width=1pt] coordinates{(0.5,0)(0.5,1)};
      \node[draw] at (axis cs:1,.8) { $\lambda = 10^{-3}$ };
    \end{axis}
    \end{tikzpicture}
\end{minipage}
\hfill{}
\begin{minipage}[b]{0.24\columnwidth}%
  \begin{tikzpicture}[font=\scriptsize]
    \pgfplotsset{every major grid/.style={style=densely dashed}}
    \begin{axis}[
      width=1.15\linewidth,
      xmin=0,
      xmax=5,
      ymin=0,
      ymax=1,
      symbolic x coords={0,0.05,0.10,0.15,0.20,0.25,0.30,0.35,0.40,0.45,0.5,0.55,0.60,0.65,0.70,0.75,0.80,0.85,0.90,0.95,1,1.50,2.00,2.50,3.00,3.50,4.00,4.50,5},
      xtick={0,0.5,1,5},
      ytick={0,0.5,1},
      grid=major,
      ymajorgrids=false,
      scaled ticks=true,
      xlabel={ $N/n$ },
      ylabel= \empty,
      legend style = {at={(0.98,0.98)}, anchor=north east, font=\scriptsize}
      ]
      \addplot[RED,densely dashed,line width=1pt] coordinates{
      (0,1.000000)(0.05,0.370481)(0.10,0.286793)(0.15,0.259282)(0.20,0.248530)(0.25,0.269872)(0.30,0.310367)(0.35,0.379618)(0.40,0.498560)(0.45,0.905816)(0.5,3.393841)(0.55,0.925892)(0.60,0.517476)(0.65,0.345039)(0.70,0.271783)(0.75,0.231674)(0.80,0.200088)(0.85,0.183833)(0.90,0.168277)(0.95,0.164896)(1,0.144632)(1.50,0.102954)(2.00,0.094597)(2.50,0.087825)(3.00,0.080429)(3.50,0.077097)(4.00,0.074911)(4.50,0.074270)(5,0.074209)
      };
      \addplot+[only marks,mark=x,BLUE,line width=0.5pt] coordinates{
      (0,1.000000)(0.05,0.383564)(0.10,0.286862)(0.15,0.258553)(0.20,0.252386)(0.25,0.272132)(0.30,0.306766)(0.35,0.374419)(0.40,0.516400)(0.45,0.944312)(0.5,3.533186)(0.55,0.902479)(0.60,0.496921)(0.65,0.348203)(0.70,0.273115)(0.75,0.235013)(0.80,0.203632)(0.85,0.183893)(0.90,0.168697)(0.95,0.158505)(1,0.145600)(1.50,0.101900)(2.00,0.093688)(2.50,0.086416)(3.00,0.078804)(3.50,0.077360)(4.00,0.074840)(4.50,0.073999)(5,0.073354)
      };
      \addplot[densely dashed,black,line width=1pt] coordinates{(0.5,0)(0.5,1)};
      \node[draw] at (axis cs:1,.8) { $\lambda = 10^{-7}$ };
    \end{axis}
    \end{tikzpicture}
\end{minipage}
\hfill{}
\begin{minipage}[b]{0.24\columnwidth}%
  \begin{tikzpicture}[font=\scriptsize]
    \pgfplotsset{every major grid/.style={style=densely dashed}}
    \begin{axis}[
      width=1.15\linewidth,
      xmin=0,
      xmax=5,
      ymin=0,
      ymax=1,
      symbolic x coords={0,0.05,0.10,0.15,0.20,0.25,0.30,0.35,0.40,0.45,0.5,0.55,0.60,0.65,0.70,0.75,0.80,0.85,0.90,0.95,1,1.50,2.00,2.50,3.00,3.50,4.00,4.50,5},
      ytick={0,0.5,1},
      grid=major,
      ymajorgrids=false,
      scaled ticks=true,
      xlabel={ $N/n$ },
      ylabel= \empty,
      legend style = {at={(0.98,0.98)}, anchor=north east, font=\scriptsize}
      ]
      \addplot[RED,densely dashed,line width=1pt] coordinates{
      (0,1)(0.05,0.387754)(0.10,0.287474)(0.15,0.264232)(0.20,0.255390)(0.25,0.265762)(0.30,0.292626)(0.35,0.340849)(0.40,0.437341)(0.45,0.572883)(0.5,0.676896)(0.55,0.546679)(0.60,0.397763)(0.65,0.312984)(0.70,0.259498)(0.75,0.224182)(0.80,0.197764)(0.85,0.181164)(0.90,0.167348)(0.95,0.155417)(1,0.147565)(1.50,0.104746)(2.00,0.090986)(2.50,0.085803)(3.00,0.079009)(3.50,0.078853)(4.00,0.074963)(4.50,0.074004)(5,0.073569)
      };
      \addplot+[only marks,mark=x,BLUE,line width=0.5pt] coordinates{
      (0,1.000000)(0.05,0.399915)(0.10,0.280613)(0.15,0.256773)(0.20,0.253655)(0.25,0.260143)(0.30,0.286681)(0.35,0.338333)(0.40,0.430583)(0.45,0.585703)(0.5,0.683006)(0.55,0.546213)(0.60,0.415693)(0.65,0.319283)(0.70,0.256089)(0.75,0.225740)(0.80,0.201905)(0.85,0.178742)(0.90,0.169685)(0.95,0.158518)(1,0.150866)(1.50,0.104220)(2.00,0.092545)(2.50,0.086008)(3.00,0.078772)(3.50,0.079838)(4.00,0.074935)(4.50,0.074457)(5,0.074902)
      };
      \addplot[densely dashed,black,line width=1pt] coordinates{(0.5,0)(0.5,1)};
      \node[draw] at (axis cs:1,.8) { $\lambda = 10^{-3}$ };
    \end{axis}
    \end{tikzpicture}
\end{minipage}
\end{center}
\caption{ Empirical (\textbf{\BLUE blue} crosses) and theoretical (\textbf{\RED red} dashed lines) test error of RFF regression, as a function of the ratio $N/n$, on Fashion-MNIST (\textbf{left two}) and Kannada-MNIST (\textbf{right two}) data (class $8$ versus $9$), for $p=784$, $n=500$, $\lambda = 10^{-7} $ and $ 10^{-3}$. The {\bf black} dashed line represents the interpolation threshold $2 N =n$. }
\label{fig:other-MNIST-double-descent}
\end{figure}

\section{Conclusion}
\label{sec:conclusion}

We have established a precise description of the resolvent of RFF Gram matrices, and provided asymptotic training and test performance guarantees for RFF ridge regression, in the limit $n,p,N \rightarrow \infty$ at the same pace.
We have also discussed the under- and over-parameterized regimes, where the resolvent behaves dramatically differently. 
These observations involve only mild regularity assumptions on the data, yielding phase transition behavior and corresponding double descent test error curves for RFF regression that closely match experiments on real-world data. 
From a technical perspective, our analysis extends to arbitrary combinations of (Lipschitz) nonlinearities, such as the more involved homogeneous kernel maps \cite{vedaldi2012efficient}. This opens the door for future studies of more elaborate random feature structures and models.
Extended to a (technically more involved) multi-layer setting in the more realistic large $n,p,N$ regime, as in \cite{fan2020spectra}, our analysis may shed new light on the theoretical understanding of modern deep neural nets, beyond the large-$N$ alone neural tangent kernel limit.

\paragraph{Acknowledgments.} 
We would like to acknowledge the UC Berkeley CLTC, ARO, IARPA, NSF, and ONR for providing partial support of this work.
Our conclusions do not necessarily reflect the position or the policy of our sponsors, and no official endorsement should be~inferred.
Couillet's work is partially supported by MIAI at University Grenoble-Alpes (ANR-19-P3IA-0003).

\bibliography{liao}
\bibliographystyle{plain}


\appendix

\section{Proof of Theorem~\ref{theo:asy-behavior-E[Q]}}
\label{sec:proof-of-theo-E[Q]}

Our objective is to prove, under Assumption~\ref{ass:high-dim}, the asymptotic equivalence between the expectation (with respect to $\W$, omitted from now on) $\EE[\Q]$ and
\[
  \bar \Q \equiv \left( \frac{N}n \left(\frac{\K_{\cos} }{1+\delta_{\cos}} +  \frac{\K_{\sin} }{1+\delta_{\sin}} \right) + \lambda \I_n\right)^{-1}
\]
for $\K_{\cos} \equiv \K_{\cos}(\X,\X), \K_{\sin} \equiv \K_{\sin}(\X,\X) \in \RR^{n \times n}$ defined in \eqref{eq:def-K}, with $(\delta_{\cos}, \delta_{\cos})$ the unique positive solution to
\[
    \delta_{\cos} = \frac1n \tr (\K_{\cos} \bar \Q), \quad \delta_{\sin} = \frac1n \tr (\K_{\sin} \bar \Q).
\]
The existence and uniqueness of the above fixed-point equation is standard in random matrix literature and can be reached for instance with the standard interference function framework \cite{yates1995framework}.

The asymptotic equivalence should be announced in the sense that $\| \EE[\Q] - \bar \Q \| \to 0$ as $n,p,N \to \infty$ at the same pace. We shall proceed by introducing an intermediary resolvent $\tilde \Q$ (see definition in \eqref{eq:def-hat-bar-Q}) and show subsequently that
\[
  \| \EE[\Q] - \tilde \Q \| \to 0, \quad \| \tilde \Q - \bar \Q \| \to 0.
\]

In the sequel, we use $o(1)$ and $ o_{\| \cdot \|}(1)$ for scalars or matrices of (almost surely if being random) vanishing absolute values or operator norms as $n,p \to \infty$.

\medskip

We start by introducing the following lemma.
\begin{Lemma}[Expectation of $\sigma_1(\x_i^\T \w) \sigma_2(\w^\T \x_j)$]\label{lem:expectation}
For $\w \sim \NN(\zo, \I_p)$ and $\x_i, \x_j \in \RR^p$ we have (per Definition in \eqref{eq:def-K})
\begin{align*}
    \EE_\w[\cos(\x_i^\T \w) \cos(\w^\T \x_j)] &= e^{-\frac12 (\| \x_i \|^2 + \| \x_j \|^2) } \cosh(\x_i^\T \x_j) \equiv [\K_{\cos}(\X,\X)]_{ij}  \equiv [\K_{\cos}]_{ij} \\
    \EE_\w[\sin(\x_i^\T \w) \sin(\w^\T \x_j)] &= e^{-\frac12 (\| \x_i \|^2 + \| \x_j \|^2) } \sinh(\x_i^\T \x_j)  \equiv [\K_{\sin} (\X,\X)]_{ij}  \equiv [\K_{\sin}]_{ij} \\
    \EE_\w[\cos(\x_i^\T \w) \sin(\w^\T \x_j)] &= 0   .
\end{align*}
\end{Lemma}
\begin{proof}[Proof of Lemma~\ref{lem:expectation}]
The proof follows the integration tricks in \cite{williams1997computing,louart2018random}. Note in particular that the third equality holds in the case of $(\cos,\sin)$ nonlinearity but in general not true for arbitrary Lipschitz $(\sigma_1, \sigma_2)$.
\end{proof}

Let us focus on the resolvent $\Q \equiv \left( \frac1n \bSigma_\X^\T \bSigma_\X + \lambda \I_n \right)^{-1}$ of $\frac1n \bSigma_\X^\T \bSigma_\X \in \RR^{n \times n}$, for random Fourier feature matrix $\bSigma_\X \equiv \begin{bmatrix} \cos(\W \X) \\ \sin(\W\X) \end{bmatrix}$ that can be rewritten as
\begin{equation}\label{eq:bSigma-vec}
    \bSigma_\X^\T = [\cos(\X^\T \w_1), \ldots, \cos(\X^\T \w_N), \sin(\X^\T \w_1), \ldots, \sin(\X^\T \w_N)]
\end{equation}
for $\w_i$ the $i$-th row of $\W \in \RR^{N \times p}$ with $\w_i \sim \mathcal N (\zo, \I_p), i = 1, \ldots, N$, that is at the core of our analysis. Note from \eqref{eq:bSigma-vec} that we have 
\[
    \bSigma_\X^\T \bSigma_\X = \sum_{i=1}^N \left( \cos(\X^\T \w_i) \cos(\w_i^\T \X) + \sin(\X^\T \w_i) \sin(\w_i^\T \X) \right) = \sum_{i=1}^N \U_i \U_i^\T
\]
with $\U_i = \begin{bmatrix} \cos(\X^\T \w_i) & \sin(\X^\T \w_i) \end{bmatrix} \in \RR^{n \times 2}$.

Letting
\begin{equation}\label{eq:def-hat-bar-Q}
  \tilde \Q \equiv \left( \frac{N}n \frac{ \K_{\cos} }{ 1 + \alpha_{\cos} } + \frac{N}n \frac{ \K_{\sin} }{ 1 + \alpha_{\sin} } + \lambda \I_n \right)^{-1}
\end{equation}
with
\begin{equation}\label{eq:def-alpha}
    \alpha_{\cos} = \frac1n \tr (\K_{\cos} \EE[\Q]), \quad \alpha_{\sin} = \frac1n \tr (\K_{\sin} \EE[\Q])
\end{equation}
we have, with the resolvent identity ($\A^{-1} - \B^{-1} = \A^{-1} (\B - \A) \B^{-1}$ for invertible $\A,\B$) that 
\begin{align*}
&\EE[\Q] - \tilde \Q = \EE \left[ \Q \left( \frac{N}n \frac{ \K_{\cos} }{ 1 + \alpha_{\cos} } + \frac{N}n \frac{ \K_{\sin} }{ 1 + \alpha_{\sin} } - \frac1n \bSigma_\X^\T \bSigma_\X \right) \right] \tilde \Q\\
&= \EE[\Q] \frac{N}n \left( \frac{ \K_{\cos} }{ 1 + \alpha_{\cos} } + \frac{ \K_{\sin} }{ 1 + \alpha_{\sin} } \right) \tilde \Q - \frac{N}n \frac1N \sum_{i=1}^N \EE[\Q \U_i \U_i^\T] \tilde \Q \\ 
&= \EE[\Q] \frac{N}n \left( \frac{ \K_{\cos} }{ 1 + \alpha_{\cos} } + \frac{ \K_{\sin} }{ 1 + \alpha_{\sin} } \right) \tilde \Q - \frac{N}n \frac1N \sum_{i=1}^N \EE[\Q_{-i} \U_i (\I_2 + \frac1n \U_i^\T \Q_{-i} \U_i)^{-1} \U_i^\T] \tilde \Q ,
\end{align*}
for $\Q_{-i} \equiv \left( \frac1n \bSigma_\X^\T \bSigma_\X - \frac1n \U_i \U_i + \lambda \I_n \right)^{-1}$ that is \textbf{independent} of $\U_i$ (and thus $\w_i$), where we applied the following Woodbury identity.
\begin{Lemma}[Woodbury]\label{lem:woodbury}
For $\A, \A + \U \U^\T \in \RR^{p \times p}$ both invertible and $\U \in \RR^{p \times n}$, we have
\[
    (\A + \U \U^\T)^{-1} = \A^{-1} -  \A^{-1} \U (\I_n + \U^\T \A^{-1} \U)^{-1} \U^\T \A^{-1}
\]
so that in particular $(\A + \U \U^\T)^{-1} \U = \A^{-1} \U (\I_n + \U^\T \A^{-1} \U)^{-1}$.
\end{Lemma}

Consider now the two-by-two matrix
\[
    \I_2 + \frac1n \U_i^\T \Q_{-i} \U_i = \begin{bmatrix} 1 + \frac1n \cos(\w_i^\T \X) \Q_{-i} \cos(\X^\T \w_i) & \frac1n \cos(\w_i^\T \X) \Q_{-i} \sin(\X^\T \w_i) \\ \frac1n \sin(\w_i^\T \X) \Q_{-i} \cos(\X^\T \w_i) & 1 + \frac1n \sin(\w_i^\T \X) \Q_{-i} \sin(\X^\T \w_i) \end{bmatrix}
\]
which, according to the following lemma, is expected to be close to $\begin{bmatrix} 1 + \alpha_{\cos} & 0 \\ 0 & 1 + \alpha_{\sin} \end{bmatrix}$ as defined in \eqref{eq:def-alpha}.

\begin{Lemma}[Concentration of quadratic forms]\label{lem:trace-lemma}
Under Assumption~\ref{ass:high-dim}, for $\sigma_1(\cdot), \sigma_2(\cdot)$ two real $1$-Lipschitz functions, $\w \sim \NN(\zo, \I_p)$ and $\A \in \RR^{n \times n}$ independent of $\w$ with $\| \A \| \le 1$, then
\[
    \mathbb P \left( \left| \frac1n \sigma_a(\w^\T \X) \A \sigma_b(\X^\T \w) - \frac1n \tr ( \A \EE_\w[ \sigma_b(\X^\T \w) \sigma_a(\w^\T \X) ] ) \right| > t \right) \le C e^{-c n \min (t, t^2)}
\]
for $a,b \in \{ 1,2 \}$ and some universal constants $C,c > 0$. 
\end{Lemma}
\begin{proof}[Proof of Lemma~\ref{lem:trace-lemma}]
Lemma~\ref{lem:trace-lemma} can be easily extended from \cite[Lemma~1]{louart2018random}, where one observes the proof actually holds when different types of nonlinear Lipschitz functions $\sigma_1(\cdot), \sigma_2(\cdot)$ (and in particular $\cos$ and $\sin$) are~considered.
\end{proof}

For $\W_{-i} \in \RR^{(N-1) \times p}$ the random matrix $\W \in \RR^{N \times p}$ with its $i$-th row $\w_i$ removed, Lemma~\ref{lem:trace-lemma}, together with the Lipschitz nature of the map $\W_{-i} \mapsto \frac1n \sigma_a(\w_i^\T \X) \Q_{-i} \sigma_b(\X^\T \w_i)$ for $\Q_{-i} = (\frac1n \cos(\W_{-i} \X)^\T \cos(\W_{-i} \X) + \frac1n \sin(\W_{-i} \X)^\T \sin(\W_{-i} \X) + \lambda \I_n)^{-1}$, leads to the following concentration result
\begin{equation}\label{eq:concentration-D}
  \mathbb P \left( \left| \frac1n \sigma_a(\w_i^\T \X) \Q_{-i} \sigma_b(\X^\T \w_i) - \frac1n \tr \left( \EE[\Q_{-i}] \EE[ \sigma_b(\X^\T \w_i) \sigma_a(\w_i^\T \X) ] \right) \right| > t \right) \le C' e^{-c'n \max(t^2,t)}
\end{equation}
the proof of which follows the same line of argument of \cite[Lemma~4]{louart2018random} and is omitted here.

As a consequence, we continue to write, with again the resolvent identity, that
\begin{align*}
    & (\I_2 + \frac1n \U_i^\T \Q_{-i} \U_i)^{-1} - \begin{bmatrix} 1+\alpha_{\cos} & 0 \\ 0 & 1+\alpha_{\sin} \end{bmatrix}^{-1} \\ 
    &= \begin{bmatrix} 1 + \frac1n \cos(\w_i^\T \X) \Q_{-i} \cos(\X^\T \w_i) & \frac1n \cos(\w_i^\T \X) \Q_{-i} \sin(\X^\T \w_i) \\ \frac1n \sin(\w_i^\T \X) \Q_{-i} \cos(\X^\T \w_i) & 1 + \frac1n \sin(\w_i^\T \X) \Q_{-i} \sin(\X^\T \w_i) \end{bmatrix}^{-1} - \begin{bmatrix} 1+\alpha_{\cos} & 0 \\ 0 & 1+\alpha_{\sin} \end{bmatrix}^{-1} \\
    &=(\I_2 + \frac1n \U_i^\T \Q_{-i} \U_i)^{-1} \begin{bmatrix} \alpha_{\cos} - \frac1n \cos(\w_i^\T \X) \Q_{-i} \cos(\X^\T \w_i) & -\frac1n \cos(\w_i^\T \X) \Q_{-i} \sin(\X^\T \w_i) \\ -\frac1n \sin(\w_i^\T \X) \Q_{-i} \cos(\X^\T \w_i) & \alpha_{\sin} - \frac1n \sin(\w_i^\T \X) \Q_{-i} \sin(\X^\T \w_i) \end{bmatrix} \\ 
    & \times \begin{bmatrix} \frac1{1+\alpha_{\cos}} & 0 \\ 0 & \frac1{1+\alpha_{\sin}} \end{bmatrix} \equiv (\I_2 + \frac1n \U_i^\T \Q_{-i} \U_i)^{-1} \D_i \begin{bmatrix} \frac1{1+\alpha_{\cos}} & 0 \\ 0 & \frac1{1+\alpha_{\sin}} \end{bmatrix}  ,
\end{align*}
where we note from \eqref{eq:concentration-D} (and $\| \Q_{-i} \| \le \lambda^{-1}$) that the matrix $\EE[\D_i] = o_{\| \cdot \|}(1)$ (in fact of spectral norm of order $O(n^{-\frac12})$). So that
\begin{align*}
  &\EE[\Q] - \tilde \Q = \EE[\Q] \frac{N}n \left( \frac{ \K_{\cos} }{ 1 + \alpha_{\cos} } + \frac{ \K_{\sin} }{ 1 + \alpha_{\sin} } \right) \tilde \Q - \frac{N}n \frac1N \sum_{i=1}^N \EE[\Q_{-i} \U_i (\I_2 + \frac1n \U_i^\T \Q_{-i} \U_i)^{-1} \U_i^\T] \tilde \Q \\
  &= \EE[\Q] \frac{N}n \left( \frac{ \K_{\cos} }{ 1 + \alpha_{\cos} } + \frac{ \K_{\sin} }{ 1 + \alpha_{\sin} } \right) \tilde \Q - \frac{N}n \frac1N \sum_{i=1}^N \EE[\Q_{-i} \U_i \begin{bmatrix} \frac1{1+\alpha_{\cos}} & 0 \\ 0 & \frac1{1+\alpha_{\sin}} \end{bmatrix} \U_i^\T] \tilde \Q \\ 
  &- \frac{N}n \frac1N \sum_{i=1}^N \EE[\Q_{-i} \U_i (\I_2 + \frac1n \U_i^\T \Q_{-i} \U_i)^{-1} \D_i \begin{bmatrix} \frac1{1+\alpha_{\cos}} & 0 \\ 0 & \frac1{1+\alpha_{\sin}} \end{bmatrix} \U_i^\T] \tilde \Q \\ 
  &= (\EE[\Q] - \frac1N \sum_{i=1}^N \EE[\Q_{-i}]) \frac{N}n \left( \frac{ \K_{\cos} }{ 1 + \alpha_{\cos} } + \frac{ \K_{\sin} }{ 1 + \alpha_{\sin} } \right) \tilde \Q -\frac{N}n \frac1N \sum_{i=1}^N \EE[ \Q \U_i \D_i \begin{bmatrix} \frac1{1+\alpha_{\cos}} & 0 \\ 0 & \frac1{1+\alpha_{\sin}} \end{bmatrix} \U_i^\T] \tilde \Q  ,
\end{align*}
where we used $\EE_{\w_i}[\U_i \U_i^\T] = \K_{\cos}+ \K_{\sin}$ by Lemma~\ref{lem:expectation} and then Lemma~\ref{lem:woodbury} in reverse for the last equality. Moreover, since
\[
  \EE[\Q] - \frac1N \sum_{i=1}^N \EE[\Q_{-i}] = \frac1N \sum_{i=1}^N \EE[\Q - \Q_{-i}] = - \frac1n \frac1N \sum_{i=1}^N \EE[\Q \U_i (\I_2 + \frac1n \U_i^\T \Q_{-i} \U_i)^{-1} \U_i^\T \Q]
\]
so that with the fact $\frac1{\sqrt n} \| \Q \bSigma_\X^\T \| \le \| \sqrt{ \Q \frac1n \bSigma_\X^\T \bSigma_\X \Q }\|  \le \lambda^{-\frac12} $ we have for the first term 
$$
\| \EE[\Q] - \frac1N \sum_{i=1}^N \EE[\Q_{-i}] \| = O(n^{-1})  .
$$
It thus remains to treat the second term, which, with the relation $\A \B^\T + \B \A^\T \preceq \A\A^\T + \B\B^\T$ (in the sense of symmetric matrices), and the same line of arguments as above, can be shown to have vanishing spectral norm (of order $O(n^{-\frac12})$) as $n,p,N \to \infty$.

We thus have $\| \EE[\Q] - \tilde \Q \|  = O(n^{-\frac12}) $, which concludes the first part of the proof of Theorem~\ref{theo:asy-behavior-E[Q]}.

We shall show next that $\| \tilde \Q - \bar \Q \| \to 0$ as $n,p,N \to \infty$. First note from previous derivation that $\alpha_\sigma - \frac1n \tr \K_\sigma \tilde \Q  = O(n^{-\frac12})$ for $\sigma = \cos, \sin$. To compare $\tilde \Q$ and $\bar \Q$, it follows again from the resolvent identity that
\[
  \tilde \Q - \bar \Q =  \tilde \Q \left( \frac{N}n \frac{\K_{\cos} (\alpha_{\cos} - \delta_{\cos})}{ (1+\delta_{\cos}) (1+\alpha_{\cos}) } + \frac{N}n \frac{\K_{\sin} (\alpha_{\sin} - \delta_{\sin})}{ (1+\delta_{\sin}) (1+\alpha_{\sin}) } \right) \bar \Q
\]
so that the control of $\| \tilde \Q - \bar \Q \| $ boils down to the control of $\max\{|\alpha_{\cos} - \delta_{\cos}|, |\alpha_{\sin} - \delta_{\sin}|\}$. To this end, it suffices to write
\[
  \alpha_{\cos} - \delta_{\cos} = \frac1n \tr \K_{\cos} ( \EE[\Q] - \bar \Q ) = \frac1n \tr \K_{\cos} ( \tilde \Q - \bar \Q ) + O(n^{-\frac12})
\]
where we used $|\tr (\A \B)| \le \| \A \| \tr(\B)$ for nonnegative definite $\B$, together with the fact that $\frac1n \tr \K_\sigma$ is (uniformly) bounded under Assumption~\ref{ass:high-dim}, for $\sigma = \cos, \sin$.

As a consequence, we have
\[
  |\alpha_{\cos} - \delta_{\cos}| \le |\alpha_{\cos} - \delta_{\cos}| \frac{N}n \frac{ \frac1n \tr (\K_{\cos} \tilde \Q \K_{\cos} \bar \Q) }{ (1+\delta_{\cos}) (1+\alpha_{\cos}) } + o(1).
\]
It thus remains to show 
\[
  \frac{N}n \frac{ \frac1n \tr (\K_{\cos} \tilde \Q \K_{\cos} \bar \Q) }{ (1+\delta_{\cos}) (1+\alpha_{\cos}) } < 1
\]
or alternatively, by the Cauchy–Schwarz inequality, to show
\[
  \frac{N}n \frac{ \frac1n \tr (\K_{\cos} \tilde \Q \K_{\cos} \bar \Q) }{ (1+\delta_{\cos}) (1+\alpha_{\cos}) } \le \sqrt{ \frac{N}n \frac{ \frac1n \tr (\K_{\cos} \bar \Q \K_{\cos} \bar \Q) }{ (1+\delta_{\cos})^2 } \cdot \frac{N}n \frac{ \frac1n \tr (\K_{\cos} \tilde \Q \K_{\cos} \tilde \Q) }{ (1+\alpha_{\cos})^2}   } <  1.
\]
To treat the first right-hand side term (the second can be done similarly), it unfolds from $|\tr (\A \B)| \le \| \A \| \cdot \tr(\B)$ for nonnegative definite $\B$ that
\[
  \frac{N}n \frac{ \frac1n \tr (\K_{\cos} \bar \Q \K_{\cos} \bar \Q) }{ (1+\delta_{\cos})^2 } \le \left\| \frac{N}n \frac{\K_{\cos} \bar \Q}{1+\delta_{\cos}}  \right\|  \frac{\frac1n \tr (\K_{\cos} \bar \Q) }{1+\delta_{\cos}} = \left\| \frac{N}n \frac{\K_{\cos} \bar \Q}{1+\delta_{\cos}}  \right\|  \frac{ \gamma_{\cos} }{1+\delta_{\cos}} \le \frac{ \gamma_{\cos} }{1+\delta_{\cos}} < 1
\]
where we used the fact that $\frac{N}n \frac{\K_{\cos} \bar \Q}{1+\delta_{\cos}} = \I_n - \frac{N}n \frac{\K_{\sin} \bar \Q}{1+\delta_{\sin}} - \lambda \bar \Q$. This concludes the proof of Theorem~\ref{theo:asy-behavior-E[Q]}.
\QED

\section{Proof of Theorem~\ref{theo:asy-training-MSE}}
\label{sec:proof-theo-training-MSE}

To prove Theorem~\ref{theo:asy-training-MSE}, it indeed suffices to prove the following lemma.
\begin{Lemma}[Asymptotic behavior of \texorpdfstring{$\EE[\Q\A\Q]$}{E[QAQ]}]\label{lem:asy-behavior-E[QAQ]}
Under Assumption~\ref{ass:high-dim}, for $\Q$ defined in \eqref{eq:def-Q} and symmetric nonnegative definite $\A \in \RR^{n \times n}$ of bounded spectral norm, we have
\[
    \left\| \EE[\Q \A \Q] - \left( \bar \Q \A \bar \Q + \frac{N}n \begin{bmatrix} \frac{ \frac1n \tr (\bar \Q \A \bar \Q \K_{\cos}) }{ (1+\delta_{\cos})^2 } & \frac{ \frac1n \tr (\bar \Q \A \bar \Q \K_{\sin}) }{ (1+\delta_{\sin})^2 } \end{bmatrix} \bOmega \begin{bmatrix} \bar \Q \K_{\cos} \bar \Q \\ \bar \Q \K_{\sin} \bar \Q \end{bmatrix} \right) \right\| \to 0 
\]
almost surely as $n \to \infty$, with $\bOmega^{-1} \equiv \I_2 - \frac{N}n \begin{bmatrix} \frac{ \frac1n \tr (\bar \Q \K_{\cos} \bar \Q \K_{\cos}) }{ (1+\delta_{\cos})^2 } & \frac{ \frac1n \tr (\bar \Q \K_{\cos} \bar \Q \K_{\sin}) }{ (1+\delta_{\sin})^2 } \\ \frac{ \frac1n \tr (\bar \Q \K_{\cos} \bar \Q \K_{\sin}) }{ (1+\delta_{\cos})^2 } & \frac{ \frac1n \tr (\bar \Q \K_{\sin} \bar \Q \K_{\sin}) }{ (1+\delta_{\sin})^2 } \end{bmatrix}$. In particular, we have
\[
    \left\| \EE \begin{bmatrix} \Q \K_{\cos} \Q \\ \Q \K_{\sin} \Q \end{bmatrix} - \bOmega \begin{bmatrix} \bar \Q \K_{\cos} \bar \Q \\ \bar \Q \K_{\sin} \bar \Q \end{bmatrix} \right\| \to 0.
\]
\end{Lemma}
\begin{proof}[Proof of Lemma~\ref{lem:asy-behavior-E[QAQ]}]
The proof of Lemma~\ref{lem:asy-behavior-E[QAQ]} essentially follows the same line of arguments as that of Theorem~\ref{theo:asy-behavior-E[Q]}. Writing
\begin{align*}
\EE[\Q \A \Q] &=\EE[\bar \Q \A \Q] + \EE[(\Q - \bar \Q) \A \Q]\\
&\simeq \bar \Q \A \bar \Q + \EE\left[ \Q \left( \frac{N}n \frac{\K_{\cos}}{1+\delta_{\cos}} + \frac{N}n \frac{\K_{\sin}}{1+\delta_{\sin}} - \frac1n \bSigma_\X^\T \bSigma_\X \right) \bar \Q \A \Q \right]\\
&=\bar \Q \A \bar \Q + \frac{N}n \EE[\Q \bPhi \bar \Q \A \Q] - \frac1n \sum_{i=1}^N \EE[\Q \U_i \U_i^\T \bar \Q \A \Q]
\end{align*}
where we note $\simeq$ by ignoring matrices with vanishing spectral norm (i.e., $o_{\| \cdot \|}(1)$) in the $n,,p,N \to \infty$ limit and recall the shortcut $\bPhi \equiv \frac{\K_{\cos}}{1+\delta_{\cos}} + \frac{\K_{\sin}}{1+\delta_{\sin}} $. Developing rightmost term with Lemma~\ref{lem:woodbury} as
\begin{align*}
&\EE[\Q \U_i \U_i^\T \bar \Q \A \Q] = \EE \left[ \Q_{-i} \U_i (\I_2 + \frac1n \U_i^\T \Q_{-i} \U_i)^{-1} \U_i^\T \bar \Q \A \Q \right] \\
&= \EE \left[ \Q_{-i} \U_i (\I_2 + \frac1n \U_i^\T \Q_{-i} \U_i)^{-1} \U_i^\T \bar \Q \A \Q_{-i} \right] \\ 
&- \frac1n \EE \left[ \Q_{-i} \U_i (\I_2 + \frac1n \U_i^\T \Q_{-i} \U_i)^{-1} \U_i^\T \bar \Q \A \Q_{-i} \U_i (\I_2 + \frac1n \U_i^\T \Q_{-i} \U_i)^{-1} \U_i^\T \Q_{-i} \right] \\ 
& \simeq \EE[\Q_{-i} \bPhi \bar \Q \A \Q_{-i}] \\
&- \EE \left[ \Q_{-i} \U_i \begin{bmatrix} \frac1{1 + \delta_{\cos}} & 0 \\ 0 & \frac1{1 + \delta_{\sin}} \end{bmatrix} \begin{bmatrix} \frac1n \tr (\bar \Q \A \bar \Q \K_{\cos}) & 0 \\ 0 & \frac1n \tr (\bar \Q \A \bar \Q \K_{\sin}) \end{bmatrix} \begin{bmatrix} \frac1{1 + \delta_{\cos}} & 0 \\ 0 & \frac1{1 + \delta_{\sin}} \end{bmatrix} \U_i^\T \Q_{-i} \right]
\end{align*}
so that
\begin{align}
    \EE[\Q \A \Q] & \simeq \bar \Q \A \bar \Q + \frac{N}n \EE \left[ \Q \left( \frac{ \frac1n \tr (\bar \Q \A \bar \Q \K_{\cos}) }{ (1+\delta_{\cos})^2 } \K_{\cos} + \frac{ \frac1n \tr (\bar \Q \A \bar \Q \K_{\sin}) }{ (1+\delta_{\sin})^2 } \K_{\sin} \right) \Q \right] \nonumber \\
    & = \bar \Q \A \bar \Q + \frac{N}n \begin{bmatrix} \frac{ \frac1n \tr (\bar \Q \A \bar \Q \K_{\cos}) }{ (1+\delta_{\cos})^2 } & \frac{ \frac1n \tr (\bar \Q \A \bar \Q \K_{\sin}) }{ (1+\delta_{\sin})^2 } \end{bmatrix} \EE \begin{bmatrix} \Q \K_{\cos} \Q \\ \Q \K_{\sin} \Q \end{bmatrix} \label{eq:E[QAQ]}
\end{align}
by taking $\A = \K_{\cos}$ or $\K_{\sin}$, we result in
\begin{align*}
    \EE[\Q \K_{\cos} \Q] \simeq \frac{c}{ac - bd} \bar \Q \K_{\cos} \bar \Q + \frac{b}{ac - bd} \bar \Q \K_{\sin} \bar \Q \\
    \EE[\Q \K_{\sin} \Q] \simeq \frac{a}{ac - bd} \bar \Q \K_{\sin} \bar \Q + \frac{d}{ac - bd} \bar \Q \K_{\cos} \bar \Q
\end{align*}
with $a = 1 - \frac{N}n \frac{ \frac1n \tr (\bar \Q \K_{\cos} \bar \Q \K_{\cos}) }{ (1+\delta_{\cos})^2 } $, $b = \frac{N}n \frac{ \frac1n \tr (\bar \Q \K_{\cos} \bar \Q \K_{\sin}) }{ (1+\delta_{\sin})^2 } $, $c = 1 - \frac{N}n \frac{ \frac1n \tr (\bar \Q \K_{\sin} \bar \Q \K_{\sin}) }{ (1+\delta_{\sin})^2 } $ and $d = \frac{N}n \frac{ \frac1n \tr (\bar \Q \K_{\sin} \bar \Q \K_{\cos}) }{ (1+\delta_{\cos})^2 } $ such that $(1+\delta_{\sin})^2 b = (1+\delta_{\cos})^2 d$.

\[
    \EE \begin{bmatrix} \Q \K_{\cos} \Q \\ \Q \K_{\sin} \Q \end{bmatrix} \simeq \begin{bmatrix} a & -b \\ -d & c \end{bmatrix}^{-1} \begin{bmatrix} \bar \Q \K_{\cos} \bar \Q \\ \bar \Q \K_{\sin} \bar \Q \end{bmatrix} \equiv \bOmega \begin{bmatrix} \bar \Q \K_{\cos} \bar \Q \\ \bar \Q \K_{\sin} \bar \Q \end{bmatrix}
\]
for $\bOmega \equiv \begin{bmatrix} a & -b \\ -d & c \end{bmatrix}^{-1}$. Plugging back into \eqref{eq:E[QAQ]} we conclude the proof of Lemma~\ref{lem:asy-behavior-E[QAQ]}.
\end{proof}

Theorem~\ref{theo:asy-training-MSE} can be achieved by considering the concentration of (the bilinear form) $\frac1n \y^\T \Q^2 \y$ around its expectation $\frac1n \y^\T \EE[\Q^2] \y$ (with for instance Lemma~3 in \cite{louart2018random}), together with Lemma~\ref{lem:asy-behavior-E[QAQ]}. This concludes the proof of Theorem~\ref{theo:asy-training-MSE}. \QED
\section{Proof of Theorem~\ref{theo:asy-test-MSE}}
\label{sec:proof-theo-test-MSE}

Recall the definition of $E_{\test} = \frac1{\hat n} \| \hat \y - \bSigma_{\hat \X}^\T \bbeta \|^2 $ from \eqref{eq:def-MSE} with $\bSigma_{\hat \X} = \begin{bmatrix} \cos(\W \hat \X) \\ \sin(\W \hat \X) \end{bmatrix} \in \RR^{2N \times \hat n}$ on a test set $(\hat \X, \hat \y)$ of size $\hat n$, and first focus on the case $2N > n$ where $\bbeta = \frac1n \bSigma_\X \Q \y$ as per \eqref{eq:def-beta}. By \eqref{eq:bSigma-vec}, we have
\[
    E_{\test} = \frac1{\hat n} \left\| \hat \y - \frac1n \bSigma_{\hat \X}^\T \bSigma_\X \Q \y \right\|^2 = \frac1{\hat n} \left\| \hat \y - \frac1n \sum_{i=1}^N \hat \U_i \U_i^\T \Q \y \right\|^2
\]
where, similar to the notation $\U_i = \begin{bmatrix} \cos(\X^\T \w_i) & \sin(\X^\T \w_i) \end{bmatrix} \in \RR^{n \times 2}$ as in the proof of Theorem~\ref{theo:asy-behavior-E[Q]}, we denote
\[
    \hat \U_i \equiv \begin{bmatrix} \cos(\hat \X^\T \w_i) & \sin(\hat \X^\T \w_i) \end{bmatrix} \in \RR^{\hat n \times 2}.
\]

As a consequence, we further get
\begin{align*}
    &\EE[E_{\test}] = \frac1{\hat n} \| \hat \y \|^2 - \frac2{n \hat n} \sum_{i=1}^N \hat \y^\T \EE[\hat \U_i \U_i^\T \Q] \y + \frac1{n^2 \hat n} \sum_{i,j=1}^N \y^\T \EE[\Q \U_i \hat \U_i^\T \hat \U_j \U_j^\T \Q] \y \\ 
    &= \frac1{\hat n} \| \hat \y \|^2 - \frac2{n \hat n} \sum_{i=1}^N \hat \y^\T \EE \left[\hat \U_i (\I_2 + \frac1n \U_i^\T \Q_{-i} \U_i)^{-1} \U_i^\T \Q_{-i} \right] \y + \frac1{n^2 \hat n} \sum_{i,j=1}^N \y^\T \EE[\Q \U_i \hat \U_i^\T \hat \U_j \U_j^\T \Q] \y \\ 
    &\simeq \frac1{\hat n} \| \hat \y \|^2 - \frac2{n \hat n} \sum_{i=1}^N \hat \y^\T \EE \left[\hat \U_i \begin{bmatrix} \frac1{1+\delta_{\cos}} & 0 \\ 0 & \frac1{1+\delta_{\sin}} \end{bmatrix} \U_i^\T \Q_{-i} \right] \y + \frac1{n^2 \hat n} \sum_{i,j=1}^N \y^\T \EE[\Q \U_i \hat \U_i^\T \hat \U_j \U_j^\T \Q] \y \\ 
    &\simeq \frac1{\hat n} \| \hat \y \|^2 - \frac2{\hat n} \hat \y^\T \left( \frac{N}n \frac{\K_{\cos}(\hat \X, \X)}{1+\delta_{\cos}} + \frac{N}n \frac{\K_{\sin}(\hat \X, \X)}{1+\delta_{\sin}} \right) \bar \Q \y + \frac1{n^2 \hat n} \sum_{i,j=1}^N \y^\T \EE[\Q \U_i \hat \U_i^\T \hat \U_j \U_j^\T \Q] \y
\end{align*}
where we similarly denote
\begin{align*}
  \K_{\cos}(\hat \X, \X) &\equiv \left\{  e^{-\frac12 (\| \hat \x_i \|^2 + \| \x_j \|^2) } \cosh(\hat \x_i^\T \x_j)  \right\}_{i,j=1}^{\hat n, n} \\ 
  \K_{\sin}(\hat \X, \X) &\equiv \left\{  e^{-\frac12 (\| \hat \x_i \|^2 + \| \x_j \|^2) } \sinh(\hat \x_i^\T \x_j) \right\}_{i,j=1}^{\hat n, n} \in \RR^{\hat n \times n}.
\end{align*}

Note that, different from the proof of Theorem~\ref{theo:asy-behavior-E[Q]}~and~\ref{theo:asy-training-MSE} where we constantly use the fact that $\| \Q \| \le \lambda^{-1}$ and
\[
    \frac1n \bSigma_\X^\T \bSigma_\X \Q = \I_n - \lambda \Q
\]
so that $\| \frac1n \bSigma_\X^\T \bSigma_\X \Q \| \le 1$, we do not have in general a simple control for $\| \frac1n \bSigma_{\hat \X}^\T \bSigma_\X \Q \|$, when arbitrary $\hat \X$ is considered. Intuitively speaking, this is due to the loss-of-control for $\| \frac1n (\bSigma_{\hat \X} - \bSigma_\X)^\T \bSigma_\X \Q \|$ when $\hat \X$ can be chosen arbitrarily with respect to $\X$. It was remarked in \cite[Remark~1]{louart2018random} that in general only a $O(\sqrt n)$ upper bound can be derived for $ \| \frac1{\sqrt n} \bSigma_\X \|$ or $ \| \frac1{\sqrt n} \bSigma_{\hat \X} \|$. Nonetheless, this problem can be resolved with the additional Assumption~\ref{ass:data-concent}.

More precisely, note that
\begin{equation}
  \| \frac1n \bSigma_{\hat \X}^\T \bSigma_\X \Q \| \le \frac1n  \| \bSigma_{\X}^\T \bSigma_\X \Q \| + \frac1n \| (\bSigma_{\hat \X} - \bSigma_{\X})^\T \bSigma_{\X} \Q \| \le 1 + \frac1{\sqrt n} \| \bSigma_{\hat \X} - \bSigma_{\X} \| \cdot \frac1{\sqrt n} \| \bSigma_\X \Q \|
\end{equation}
it remains to show that $\| \bSigma_{\X} - \bSigma_{\hat \X} \| = O(\sqrt n)$ under Assumption~\ref{ass:data-concent} to establish $\| \frac1n \bSigma_{\hat \X}^\T \bSigma_\X \Q \| = O(1)$, that is, to show that
\begin{equation}
  \| \sigma(\W \X) - \sigma(\W \hat \X) \| = O(\sqrt n)
\end{equation}
for $\sigma \in \{ \cos, \sin\}$. Note this cannot be achieved using only the Lipschitz nature of $\sigma(\cdot)$ and the fact that $\| \X - \hat \X \| \le \| \X \| + \| \hat \X \| = O(1)$ under Assumption~\ref{ass:high-dim} by writing 
\begin{equation}
  \| \sigma(\W \X) - \sigma(\W \hat \X) \| \le \| \sigma(\W \X) - \sigma(\W \hat \X) \|_F \le \| \W \|_F \cdot \| \X - \hat \X \| = O(n).
\end{equation}
where we recall that $\| \W \| = O(\sqrt n)$ and $\| \W \|_F = O(n)$. Nonetheless, from \cite[Proposition~B.1]{louart2018concentration} we have that the product $\W \X$, and thus $\sigma(\W \X)$, strongly concentrates around its expectation in the sense of \eqref{eq:def-concentration}, so that
\begin{align*}
  \| \sigma(\W \X) - \sigma(\W \hat \X) \| &\le  \| \sigma(\W \X) - \EE[\sigma(\W \X)] \| + \| \EE[\sigma(\W \X) - \sigma(\W \hat \X)] \| \\ 
   &+ \| \sigma(\W \hat \X) - \EE[\sigma(\W \hat \X)] \| = O(\sqrt n)
\end{align*}
under Assumption~\ref{ass:data-concent}. As a results, we are allowed to control $\frac1n \bSigma_{\hat \X}^\T \bSigma_\X \Q$ and similarly $\frac1n \bSigma_{\hat \X}^\T \bSigma_{\hat \X} \Q$ in the same vein as $\frac1n \bSigma_\X^\T \bSigma_\X \Q$ in the proof of Theorem~\ref{theo:asy-behavior-E[Q]}~and~\ref{theo:asy-training-MSE} in Appendix~\ref{sec:proof-of-theo-E[Q]}~and~\ref{sec:proof-theo-training-MSE}, respectively.

It thus remains to handle the last term (noted $\Z$) as follows
\begin{align*}
    \Z &\equiv \frac1{n^2 \hat n} \sum_{i,j=1}^N \y^\T \EE[\Q \U_i \hat \U_i^\T \hat \U_j \U_j^\T \Q] \y \\ 
    &= \frac1{n^2 \hat n} \sum_{i=1}^N \y^\T \EE[\Q \U_i \hat \U_i^\T \hat \U_i \U_i^\T \Q] \y + \frac1{n^2 \hat n} \sum_{i=1}^N \sum_{j \neq i} \y^\T \EE[\Q \U_i \hat \U_i^\T \hat \U_j \U_j^\T \Q] \y = \Z_1 + \Z_2
\end{align*}
where $\Z_1$ term can be treated as
\begin{align*}
    \Z_1 &\equiv \frac1{n^2 \hat n} \sum_{i=1}^N \y^\T \EE[\Q \U_i \hat \U_i^\T \hat \U_i \U_i^\T \Q] \y \\ 
    &= \frac1{n \hat n}  \sum_{i=1}^N \y^\T \EE[\Q_{-i} \U_i (\I_2 + \frac1n \U_i^\T \Q_{-i} \U_i)^{-1} \frac1n \hat \U_i^\T \hat \U_i (\I_2 + \frac1n \U_i^\T \Q_{-i} \U_i)^{-1} \U_i^\T \Q_{-i}] \y \\ 
    &\simeq \frac1{n \hat n} \sum_{i=1}^N \y^\T \EE[\Q_{-i} \U_i \begin{bmatrix} \frac1{1+\delta_{\cos}} & 0 \\ 0 & \frac1{1+\delta_{\sin}} \end{bmatrix} \begin{bmatrix} \frac1n \tr \hat{\hat\K}_{\cos} & 0 \\ 0 & \frac1n \tr \hat{\hat\K}_{\sin} \end{bmatrix} \begin{bmatrix} \frac1{1+\delta_{\cos}} & 0 \\ 0 & \frac1{1+\delta_{\sin}} \end{bmatrix} \U_i^\T \Q_{-i}] \y \\
    &\simeq \frac{N}{n} \frac1{\hat n} \y^\T \EE \left[\Q \left( \frac{ \frac1n \tr \K_{\cos}(\hat \X, \hat \X) }{(1+\delta_{\cos})^2} \K_{\cos} + \frac{ \frac1n \tr \K_{\sin}(\hat \X, \hat \X) }{(1+\delta_{\sin})^2} \K_{\sin} \right) \Q \right] \y \\
    &\simeq \frac{N}n \frac1{\hat n} \begin{bmatrix} \frac{ \frac1n \tr \K_{\cos}(\hat \X, \hat \X) }{ (1+\delta_{\cos})^2 } & \frac{ \frac1n \tr \frac1n \tr \K_{\sin}(\hat \X, \hat \X) }{ (1+\delta_{\sin})^2 } \end{bmatrix} \bOmega \begin{bmatrix} \y^\T \bar \Q \K_{\cos} \bar \Q \y \\ \y^\T \bar \Q \K_{\sin} \bar \Q \y \end{bmatrix}
\end{align*}
where we apply Lemma~\ref{lem:asy-behavior-E[QAQ]} and recall
\[
    \K_{\cos}(\hat \X, \hat \X) \equiv \left\{  e^{-\frac12 (\| \hat \x_i \|^2 + \| \hat \x_j \|^2) } \cosh(\hat \x_i^\T \hat \x_j)  \right\}_{i,j=1}^{\hat n}, \quad\K_{\sin}(\hat \X, \hat \X) \equiv \left\{  e^{-\frac12 (\| \hat \x_i \|^2 + \| \hat \x_j \|^2) } \sinh(\hat \x_i^\T \hat \x_j) \right\}_{i,j=1}^{\hat n}
\]

Moving on to $\Z_2$ and we write
\begin{align*}
    &\Z_2 \equiv \frac1{n^2 \hat n} \EE \sum_{i=1}^N \sum_{j \neq i} \y^\T \Q \U_i \hat \U_i^\T \hat \U_j \U_j^\T \Q \y \\ 
    & = \frac1{n^2 \hat n} \EE \sum_{i=1}^N \sum_{j \neq i} \y^\T \Q_{-j} \U_i \hat \U_i^\T \hat \U_j (\I_2 + \frac1n \U_j^\T \Q_{-j} \U_j)^{-1} \U_j^\T \Q_{-j} \y \\ 
    & - \frac1{n^2 \hat n} \EE \sum_{i=1}^N \sum_{j \neq i} \y^\T \Q_{-j} \U_j (\I_2 + \frac1n \U_j^\T \Q_{-j} \U_j)^{-1} \U_j^\T \Q_{-j} \U_i \hat \U_i^\T \hat \U_j (\I_2 + \frac1n \U_j^\T \Q_{-j} \U_j)^{-1} \U_j^\T \Q_{-j} \y \\ 
    &\simeq \frac1{n \hat n} \EE \sum_{i=1}^N \sum_{j \neq i} \y^\T \Q_{-j} \U_i \hat \U_i^\T \left( \frac{\K_{\cos}(\hat \X, \X)}{1+\delta_{\cos}} + \frac{\K_{\sin}(\hat \X, \X)}{1+\delta_{\sin}} \right) \Q_{-j} \y \\ 
    & - \frac1{n^2 \hat n} \EE \sum_{i=1}^N \sum_{j \neq i} \y^\T \Q_{-j} \U_j \begin{bmatrix} \frac1{1+\delta_{\cos}} & 0 \\ 0 & \frac1{1+\delta_{\sin}} \end{bmatrix} \begin{bmatrix} \frac1n \tr (\Q_{-j} \U_i \hat \U_i^\T \K_{\cos}(\hat \X, \X)) & 0 \\ 0 & \frac1n \tr (\Q_{-j} \U_i \hat \U_i^\T \K_{\sin}(\hat \X, \X)) \end{bmatrix} \\ 
    &\begin{bmatrix} \frac1{1+\delta_{\cos}} & 0 \\ 0 & \frac1{1+\delta_{\sin}} \end{bmatrix} \U_j^\T \Q_{-j} \y \equiv \Z_{21} - \Z_{22}.
\end{align*}

For the term $\Z_{21}$, note that $\Q_{-j} \simeq \Q$ and \textbf{depends} on $\U_i$ (and $\hat \U_i$), such that
\begin{align*}
    &\Z_{21} \equiv \frac1{n^2 \hat n} \EE \sum_{i=1}^N \sum_{j \neq i} \y^\T \Q_{-j} \U_i \hat \U_i^\T \left( \frac{\K_{\cos}(\hat \X, \X)}{1+\delta_{\cos}} + \frac{\K_{\sin}(\hat \X, \X)}{1+\delta_{\sin}} \right) \Q_{-j} \y \\ 
    & \simeq \frac{N}n \frac1{n \hat n} \EE \sum_{i=1}^N \y^\T \Q \U_i \hat \U_i^\T \left( \frac{\K_{\cos}(\hat \X, \X)}{1+\delta_{\cos}} + \frac{\K_{\sin}(\hat \X, \X)}{1+\delta_{\sin}} \right) \Q \y \\ 
    &= \frac{N}n \frac1{n \hat n} \EE \sum_{i=1}^N \y^\T \Q_{-i} \U_i (\I_2 + \frac1n \U_i^\T \Q_{-i} \U_i)^{-1} \hat \U_i^\T \hat \bPhi \Q_{-i} \y \\ 
    &- \frac{N}n \frac1{n \hat n} \EE \sum_{i=1}^N \y^\T \Q_{-i} \U_i (\I_2 + \frac1n \U_i^\T \Q_{-i} \U_i)^{-1} \hat \U_i^\T \hat \bPhi \Q_{-i} \U_i (\I_2 + \frac1n \U_i^\T \Q_{-i} \U_i)^{-1} \U_i^\T \Q_{-i} \y \\ 
    & \simeq \frac{N}n \frac1{n \hat n} \EE \sum_{i=1}^N \y^\T \Q_{-i} \left( \frac{\K_{\cos}(\hat \X, \X)}{1+\delta_{\cos}} + \frac{\K_{\sin}(\hat \X, \X)}{1+\delta_{\sin}} \right)^\T \hat \bPhi \Q_{-i} \y \\ 
    &- \frac{N}n \frac1{\hat n} \EE \sum_{i=1}^N \y^\T \Q_{-i} \U_i \begin{bmatrix} \frac1{1+\delta_{\cos}} & 0 \\ 0 & \frac1{1+\delta_{\sin}} \end{bmatrix} \frac1n \hat \U_i^\T \hat \bPhi \Q_{-i} \U_i \begin{bmatrix} \frac1{1+\delta_{\cos}} & 0 \\ 0 & \frac1{1+\delta_{\sin}} \end{bmatrix} \U_i^\T \Q_{-i} \y
\end{align*}
where we recall the shortcut $\bPhi \equiv \frac{\K_{\cos}}{1+\delta_{\cos}} + \frac{\K_{\sin}}{1+\delta_{\sin}} $ and similarly $\hat \bPhi \equiv \frac{\K_{\cos}(\hat \X, \X)}{1+\delta_{\cos}} + \frac{\K_{\sin}(\hat \X, \X)}{1+\delta_{\sin}} \in \RR^{\hat n \times n}$. As a consequence, we further have, with Lemma~\ref{lem:asy-behavior-E[QAQ]} that
\begin{align*}
    &\Z_{21} \simeq \left(\frac{N}n \right)^2 \frac1{\hat n} \y^\T \EE \left[ \Q \hat \bPhi^\T \hat \bPhi \Q \right] \y \\ 
    &- \frac{N}n \frac1{\hat n} \EE \sum_{i=1}^N \y^\T \Q_{-i} \U_i \begin{bmatrix} \frac1{1+\delta_{\cos}} & 0 \\ 0 & \frac1{1+\delta_{\sin}} \end{bmatrix} \begin{bmatrix} \frac1n \tr( \hat \bPhi \bar \Q \K_{\cos}(\hat \X, \X)^\T) & 0 \\ 0 & \frac1n \tr( \hat \bPhi \bar \Q \K_{\sin}(\hat \X, \X)^\T) \end{bmatrix} \\ 
    &\times \begin{bmatrix} \frac1{1+\delta_{\cos}} & 0 \\ 0 & \frac1{1+\delta_{\sin}} \end{bmatrix} \U_i^\T \Q_{-i} \y \\ 
    &\simeq \left(\frac{N}n \right)^2 \frac1{\hat n} \y^\T \EE \left[ \Q \hat \bPhi^\T \hat \bPhi \Q \right] \y  \\ 
    &- \left(\frac{N}n\right)^2 \frac1{\hat n} \EE \y^\T \Q \left( \frac1n \tr( \hat \bPhi \bar \Q \K_{\cos}(\hat \X, \X)^\T) \frac{\K_{\cos}}{(1+\delta_{\cos})^2} + \frac1n \tr( \hat \bPhi \bar \Q \K_{\sin}(\hat \X, \X)^\T) \frac{\K_{\sin}}{(1+\delta_{\sin})^2} \right) \Q \y \\ 
    &\simeq \left(\frac{N}n \right)^2 \frac1{\hat n} \y^\T \EE \left[ \Q \hat \bPhi^\T \hat \bPhi \Q \right] \y - \left(\frac{N}n\right)^2 \frac1{\hat n} \y^\T \left( \begin{bmatrix} \frac{ \frac1n \tr( \hat \bPhi \bar \Q \K_{\cos}(\hat \X, \X)^\T) }{(1+\delta_{\cos})^2} & \frac{ \frac1n \tr( \hat \bPhi \bar \Q \K_{\sin}(\hat \X, \X)^\T) }{(1+\delta_{\sin})^2} \end{bmatrix} \EE \begin{bmatrix} \Q \K_{\cos} \Q \\ \Q \K_{\sin} \Q \end{bmatrix} \right) \y \\ 
    &\simeq \left(\frac{N}n \right)^2 \frac1{\hat n} \y^\T \bar \Q \bPhi^\T \hat \bPhi \bar \Q \y \\ 
    & + \left(\frac{N}n \right)^2 \frac1{\hat n} \begin{bmatrix} \frac{ \frac1n \tr \bar \Q \frac{N}n \hat \bPhi^\T \hat \bPhi \bar \Q \K_{\cos}  - \frac1n \tr \bar \Q \hat \bPhi \K_{\cos} (\hat \X, \X)}{ (1+\delta_{\cos})^2 } & \frac{ \frac1n \tr \bar \Q \frac{N}n \hat \bPhi^\T \hat \bPhi \bar \Q \K_{\sin} - \frac1n \tr \bar \Q \hat \bPhi^\T \K_{\sin} (\hat \X, \X)}{ (1+\delta_{\sin})^2 } \end{bmatrix} \bOmega \begin{bmatrix} \y^\T \bar \Q \K_{\cos} \bar \Q \y \\ \y^\T \bar \Q \K_{\sin} \bar \Q \y \end{bmatrix}
\end{align*}


The last term $\Z_{22}$ can be similarly treated as
\[
    \Z_{22} \simeq \frac1{n^2 \hat n} \EE \sum_{i=1}^N \sum_{j \neq i} \y^\T \Q_{-j} \U_j \begin{bmatrix} \frac{ \frac1n \tr (\Q \U_i \hat \U_i^\T \K_{\cos}(\hat \X, \X)) }{(1+\delta_{\cos})^2} & 0 \\ 0 & \frac{ \frac1n \tr (\Q \U_i \hat \U_i^\T \K_{\sin}(\hat \X, \X)) }{(1+\delta_{\sin})^2} \end{bmatrix} \U_j^\T \Q_{-j} \y
\]
where by Lemma~\ref{lem:woodbury} we deduce
\begin{align*}
    &\frac1n \tr (\Q \U_i \hat \U_i^\T \K_{\cos}(\hat \X, \X)) \simeq \frac1n \tr \left( \Q_{-i} \U_i (\I_2 + \U_i^\T \Q_{-i} \U_i)^{-1} \hat \U_i^\T \K_{\cos}(\hat \X, \X) \right) \\ 
    & \simeq \frac1n \tr \left( \Q_{-i} \U_i \begin{bmatrix} \frac1{1+\delta_{\cos}} & 0 \\ 0 & \frac1{1+\delta_{\sin}} \end{bmatrix} \hat \U_i^\T \K_{\cos}(\hat \X, \X) \right) \simeq \frac1n \tr ( \bar \Q \hat \bPhi^\T \K_{\cos}(\hat \X, \X) )
\end{align*}
so that by again Lemma~\ref{lem:asy-behavior-E[QAQ]}
\begin{align*}
    &\Z_{22} \simeq \frac{N}n \frac1{n \hat n} \EE \sum_{j=1}^N \y^\T \Q_{-j} \U_j \begin{bmatrix} \frac{ \frac1n \tr ( \bar \Q \hat \bPhi^\T \K_{\cos}(\hat \X, \X) ) }{(1+\delta_{\cos})^2} & 0 \\ 0 & \frac{ \frac1n \tr ( \bar \Q \hat \bPhi^\T \K_{\sin}(\hat \X, \X) ) }{(1+\delta_{\sin})^2} \end{bmatrix} \U_j^\T \Q_{-j} \y \\ 
    & \simeq \left( \frac{N}n \right)^2 \frac1{\hat n} \y^\T \EE \left[ \Q \left( \frac{ \frac1n \tr ( \bar \Q \hat \bPhi^\T \K_{\cos}(\hat \X, \X) ) }{(1+\delta_{\cos})^2} \K_{\cos} + \frac{ \frac1n \tr ( \bar \Q \hat \bPhi^\T \K_{\sin}(\hat \X, \X) ) }{(1+\delta_{\sin})^2}  \K_{\sin} \right) \Q \right] \y \\ 
    & \simeq \left( \frac{N}n \right)^2 \frac1{\hat n} \y^\T \left( \bar \Q \bXi \bar \Q + \frac{N}n \begin{bmatrix} \frac{ \frac1n \tr (\bar \Q \bXi \bar \Q \K_{\cos}) }{ (1+\delta_{\cos})^2 } & \frac{ \frac1n \tr (\bar \Q \bXi \bar \Q \K_{\sin}) }{ (1+\delta_{\sin})^2 } \end{bmatrix} \bOmega \begin{bmatrix} \bar \Q \K_{\cos} \bar \Q \\ \bar \Q \K_{\sin} \bar \Q \end{bmatrix} \right) \y \\
    &\simeq \left( \frac{N}n \right)^2 \frac1{\hat n} \begin{bmatrix} \frac{ \frac1n \tr ( \bar \Q \hat \bPhi^\T \K_{\cos}(\hat \X, \X) ) }{(1+\delta_{\cos})^2} & \frac{ \frac1n \tr ( \bar \Q \hat \bPhi^\T \K_{\sin}(\hat \X, \X) ) }{(1+\delta_{\sin})^2} \end{bmatrix} \bOmega \begin{bmatrix} \y^\T \bar \Q \K_{\cos} \bar \Q \y \\ \y^\T \bar \Q \K_{\sin} \bar \Q \y \end{bmatrix}.
\end{align*}


Assembling the estimates for $\Z_1$, $\Z_{21}$ and $\Z_{22}$, we get
\begin{align*}
    &\EE[E_{\test}] \simeq \frac1{\hat n} \| \hat \y \|^2 - \frac2{\hat n} \hat \y^\T \frac{N}n \hat \bPhi \bar \Q \y + \frac1{\hat n} \y^\T \left( \frac{N^2}{n^2} \bar \Q \hat \bPhi^\T \hat \bPhi \bar \Q \right) \y + \left( \frac{N}n \right)^2 \frac1{n \hat n} \times  \\ 
    & \begin{bmatrix} \frac{ \frac{n}N \tr \K_{\cos} (\hat \X, \hat \X) + \frac{N}n \tr \bar \Q \hat \bPhi^\T \hat \bPhi \bar \Q \K_{\cos}  - 2 \tr \bar \Q \hat \bPhi^\T \K_{\cos} (\hat \X, \X)}{ (1+\delta_{\cos})^2 } & \frac{ \frac{n}N \tr \K_{\sin} (\hat \X, \hat \X) +\frac{N}n \tr \bar \Q \hat \bPhi^\T \hat \bPhi \bar \Q \K_{\sin} - 2 \tr \bar \Q \hat \bPhi^\T \K_{\sin} (\hat \X, \X)}{ (1+\delta_{\sin})^2 } \end{bmatrix} \\ 
    & \times \bOmega \begin{bmatrix} \y^\T \bar \Q \K_{\cos} \bar \Q \y \\ \y^\T \bar \Q \K_{\sin} \bar \Q \y \end{bmatrix}
\end{align*}
which, up to further simplifications, concludes the proof of Theorem~\ref{theo:asy-test-MSE}.

\section{Several Useful Lemmas}
\label{sec:detail-section-double-descent}

\begin{Lemma}[Some useful properties of $\bOmega$]\label{lem:property-of-Delta}
For any $\lambda > 0$ and $\bOmega$ defined in \eqref{eq:def-Omega}, we have
\begin{enumerate}
  \item all entries of $\bOmega$ are positive;
  \item for $2N=n$, $\det(\bOmega^{-1})$, as well as the entries of $\bOmega$, scales like $\lambda$ as $\lambda \to 0$;
\end{enumerate}

\end{Lemma}

\begin{proof}
Developing the inverse we obtain
\begin{align*}
  \bOmega &= \begin{bmatrix} 1 - \frac{N}n \frac{\frac1n \tr (\bar \Q \K_{\cos} \bar \Q \K_{\cos})}{(1+\delta_{\cos})^2} & -\frac{N}n \frac{\frac1n \tr (\bar \Q \K_{\cos} \bar \Q \K_{\sin})}{(1+\delta_{\sin})^2} \\ -\frac{N}n \frac{\frac1n \tr (\bar \Q \K_{\cos} \bar \Q \K_{\sin})}{(1+\delta_{\cos})^2} & 1 - \frac{N}n \frac{\frac1n \tr (\bar \Q \K_{\sin} \bar \Q \K_{\sin})}{(1+\delta_{\sin})^2} \end{bmatrix}^{-1}
\end{align*}
we have $[\bOmega^{-1}]_{11} = \frac1{1+\delta_{\cos}} + \frac{\lambda}n \tr \bar \Q \frac{\K_{\cos}}{1+\delta_{\cos}} \bar \Q + \frac{N}n \frac1n \tr \bar \Q \frac{\K_{\cos}}{1+\delta_{\cos}} \bar \Q \frac{\K_{\sin}}{1+\delta_{\sin}} > 0$, $[\bOmega^{-1}]_{12} <0 $, and similarly $[\bOmega^{-1}]_{21}< 0 $, $[\bOmega^{-1}]_{22} >0$. Furthermore, the determinant writes
\begin{align*}
  &\det(\bOmega^{-1}) = \left( 1 - \frac1n \tr \bar \Q \frac{\K_{\cos}}{1+\delta_{\cos}} + \frac{\lambda}n \tr \bar \Q \frac{\K_{\cos}}{1+\delta_{\cos}} \bar \Q \right) \left( 1 - \frac1n \tr \bar \Q \frac{\K_{\sin}}{1+\delta_{\sin}} + \frac{\lambda}n \tr \bar \Q \frac{\K_{\sin}}{1+\delta_{\sin}} \bar \Q \right) \\
  &+ \left( 1 - \frac1n \tr \bar \Q \frac{\K_{\cos}}{1+\delta_{\cos}} + 1 - \frac1n \tr \bar \Q \frac{\K_{\sin}}{1+\delta_{\sin}} + \frac{\lambda}n \tr \bar \Q \left( \frac{\K_{\cos}}{1+\delta_{\cos}} + \frac{\K_{\sin}}{1+\delta_{\sin}} \right) \bar \Q \right) \\ 
  & \times \frac{N}n \frac1n \tr \bar \Q \frac{\K_{\cos}}{1+\delta_{\cos}} \bar \Q \frac{\K_{\sin}}{1+\delta_{\sin}}
\end{align*}
where we constantly use the fact that $\bar \Q \frac{N}n \left( \frac{\K_{\cos}}{1+\delta_{\cos}} + \frac{\K_{\sin}}{1+\delta_{\sin}} \right) = \I_n - \lambda \bar \Q$. Note that 
\begin{align*}
  &1 - \frac1n \tr \bar \Q \frac{\K_{\cos}}{1+\delta_{\cos}} =  \frac1{1+\delta_{\cos}} >0, \quad 1 - \frac1n \tr \bar \Q \frac{\K_{\sin}}{1+\delta_{\sin}} =  \frac1{1+\delta_{\sin}} >0 \\
  &\frac1{1+\delta_{\cos}} + \frac1{1+\delta_{\sin}} = \underline{2 - \frac{n}N} + \frac{\lambda}N \tr \bar \Q >0
\end{align*}
so that 1) $\det(\bOmega^{-1}) > 0$ and 2) for $2N=n$, $\det(\bOmega^{-1})$ scales like $\lambda$ as $\lambda \to 0$.
\end{proof}
%
%
%
%
%
\begin{Lemma}[Derivatives with respect to $N$]\label{lem:delta-derivative-N}
Let Assumption~\ref{ass:high-dim} holds, for any $\lambda > 0$ and
\[
  \begin{cases}
  \delta_{\cos} = \frac1n \tr (\K_{\cos} \bar \Q) = \frac1n \tr \K_{\cos} \left( \frac{N}n \left(\frac{\K_{\cos} }{1+\delta_{\cos}} + \frac{\K_{\sin} }{1+\delta_{\sin}} \right) + \lambda \I_n \right)^{-1} \\ 
  \delta_{\sin} = \frac1n \tr (\K_{\sin} \bar \Q) = \frac1n \tr \K_{\sin} \left( \frac{N}n \left(\frac{\K_{\cos} }{1+\delta_{\cos}} + \frac{\K_{\sin} }{1+\delta_{\sin}} \right) + \lambda \I_n \right)^{-1}
  \end{cases}
\]
defined in Theorem~\ref{theo:asy-behavior-E[Q]}, we have that $(\delta_{\cos},\delta_{\sin})$ and $\| \bar \Q \|$ are all decreasing functions of $N$. Note in particular that the same conclusion holds for $2N > n$ as $\lambda \to 0$. 
\end{Lemma}
\begin{proof}
We write
\begin{equation}
  \begin{bmatrix} \frac{\partial \delta_{\cos}}{\partial N} \\ \frac{\partial \delta_{\sin}}{\partial N} \end{bmatrix} = - \frac1n {\bOmega} \begin{bmatrix} \frac1n \tr \left( \bar \Q \bPhi \bar \Q \K_{\cos} \right) \\ \frac1n \tr \left( \bar \Q \bPhi \bar \Q \K_{\sin} \right) \end{bmatrix} = - \frac{n}N \frac1n {\bOmega} \begin{bmatrix} \delta_{\cos} - \frac{\lambda}n \tr ( \bar \Q \K_{\cos} \bar \Q) \\ \delta_{\sin} - \frac{\lambda}n \tr ( \bar \Q \K_{\sin} \bar \Q) \end{bmatrix}
\end{equation}
for $\bOmega$ defined in \eqref{eq:def-Omega} and $\bPhi = \frac{\K_{\cos}}{1+\delta_{\cos}} + \frac{\K_{\sin}}{1+\delta_{\sin}}$, which, together with Lemma~\ref{lem:property-of-Delta}, allows us to conclude that $\frac{\partial \delta_{\cos}}{\partial N}, \frac{\partial \delta_{\sin}}{\partial N}< 0$. Further note that 
\begin{align*}
  \frac{\partial \bar \Q}{\partial N} &= - \frac1n \bar \Q \left( \bPhi - \frac{\K_{\cos}}{(1+\delta_{\cos})^2} N \frac{\partial \delta_{\cos}}{\partial N} - \frac{\K_{\sin}}{(1+\delta_{\sin})^2} N \frac{\partial \delta_{\sin}}{\partial N} \right) \bar \Q 
\end{align*}
which concludes the proof.
\end{proof}
\begin{Lemma}[Derivative with respect to $\lambda$]\label{lem:delta-derivative-lambda}
For any $\lambda > 0$, $(\delta_{\cos}, \delta_{\sin})$ and $\| \bar \Q \|$ defined in Theorem~\ref{theo:asy-behavior-E[Q]} decrease as $\lambda$ grows large.
\end{Lemma}
\begin{proof}
Taking the derivative of $(\delta_{\cos}, \delta_{\sin})$ with respect to $\lambda >0$, we have explicitly
\begin{equation}
  \begin{bmatrix} \frac{\partial \delta_{\cos}}{\partial \lambda} \\ \frac{\partial \delta_{\sin}}{\partial \lambda} \end{bmatrix} = - \bOmega \begin{bmatrix} \frac1n \tr (\bar \Q \K_{\cos} \bar \Q) \\ \frac1n \tr (\bar \Q \K_{\sin} \bar \Q) \end{bmatrix}
\end{equation}
which, together with the fact that all entries of $\bOmega$ are positive (Lemma~\ref{lem:property-of-Delta}), allows us to conclude that $\frac{\partial \delta_{\cos}}{\partial \lambda}, \frac{\partial \delta_{\sin}}{\partial \lambda}< 0$. Further considering 
\[
  \frac{\partial \bar \Q}{\partial \lambda} = \bar \Q \left( \frac{N}n \frac{\K_{\cos}}{(1+\delta_{\cos})^2} \frac{\partial \delta_{\cos}}{\partial \lambda} + \frac{N}n \frac{\K_{\sin}}{(1+\delta_{\sin})^2} \frac{\partial \delta_{\sin}}{\partial \lambda} - \I_n \right) \bar \Q
\]
and thus the conclusion for $\bar \Q$. 

\end{proof}

\end{document}